\documentclass[preprint, 12pt, authoyear]{article}
\usepackage{jmlr2e} 

\usepackage[english]{babel}
\usepackage[latin1]{inputenc}
\usepackage{amsmath,amssymb,latexsym,amsfonts}
\usepackage{graphics}
\usepackage{mathtools}
\usepackage{algorithm}
\usepackage{algorithmic}
\usepackage{bbm}
\usepackage[font=footnotesize,labelfont=footnotesize]{caption}
\usepackage{color}
\usepackage{float}
\usepackage{graphics}
\usepackage{graphicx} \usepackage{verbatim}
\usepackage[margin=1in]{geometry}
\usepackage{mathtools}
\usepackage{tikz}
\usepackage{tikz-cd}
\usetikzlibrary{positioning}
\usetikzlibrary{matrix,arrows,decorations.pathmorphing}
\usepackage[font=scriptsize,labelfont=scriptsize]{subcaption}
\usepackage{tabulary}
\usepackage{booktabs}
\usepackage{pifont}
\usepackage[normalem]{ulem}
\usepackage{multirow}
\usepackage{adjustbox}
\usepackage{enumitem}

\DeclarePairedDelimiter\floor{\lfloor}{\rfloor}

\DeclareMathOperator*{\argmax}{arg\,max}

\newtheorem{thm}{Theorem}[section]
\newtheorem{cor}[thm]{Corollary}
\newtheorem{prop}[thm]{Proposition}
\newtheorem{lem}[thm]{Lemma}

\newtheorem{defn}[thm]{Definition}

\newtheorem{rems}[thm]{Remark}

\newtheorem{assumption}{Assumption}

\newcommand{\Prob}{\mathbb{P}}
\newcommand{\Lsym}{{L_{\text{SYM}}}}

\DeclareMathOperator*{\argmin}{arg\,min}
\DeclareMathOperator*{\Bin}{\text{Bin}}

\DeclareMathOperator*{\Ncut}{\text{Ncut}}

\DeclareMathOperator*{\Unif}{Unif}
\DeclareMathOperator*{\X}{\mathcal{X}}

\newcommand{\vol}{\text{vol}}
\newcommand{\numclust}{K} 
\newcommand{\kNoise}{{k_{\text{nse}}}} 
\newcommand{\kEuc }{{k_{\text{Euc}}}}
\newcommand{\kLLPD}{{k_{\ell\ell}}} 

\newcommand{\Deg}{{D}}  
\newcommand{\vold}{\mathcal{H}^{d}}
\newcommand{\volD}{\mathcal{H}^{D}}
\newcommand{\epsilonincluster}{\epsilon_{\mathrm{in}}}
\newcommand{\epsilonnoiseknn}{\epsilon_{\mathrm{nse}}}
\newcommand{\epsilonbetweencluster}{\epsilon_{\mathrm{btw}}}
\newcommand{\epsilonsep}{\epsilon_{\mathrm{sep}}}
\newcommand{\diam}{\text{diam}}
\newcommand{\thetathres}{\theta}
\newcommand{\nmin}{n_{\text{min}}}

\makeatletter
\newcommand{\distas}[1]{\mathbin{\overset{#1}{\kern\z@\sim}}}%
\newsavebox{\mybox}\newsavebox{\mysim}
\newcommand{\distras}[1]{%
  \savebox{\mybox}{\hbox{\kern3pt$\scriptstyle#1$\kern3pt}}%
  \savebox{\mysim}{\hbox{$\sim$}}%
  \mathbin{\overset{#1}{\kern\z@\resizebox{\wd\mybox}{\ht\mysim}{$\sim$}}}%
}
\makeatother

\makeatletter
\let\c@equation\c@thm
\makeatother
\numberwithin{equation}{section}

\makeatletter
\def\ps@pprintTitle{%
 \let\@oddhead\@empty
 \let\@evenhead\@empty
 \def\@oddfoot{}%
 \let\@evenfoot\@oddfoot}
\makeatother

\begin{document}

\title{Path-Based Spectral Clustering: Guarantees, Robustness to Outliers, and Fast Algorithms}

\author{\name Anna Little \email littl119@msu.edu   \\
       \addr Department of Computational Mathematics, Science, and Engineering\\
       Michigan State University,
       East Lansing, MI 48824, USA
       \AND
       \name Mauro Maggioni \email mauromaggionijhu@icloud.com \\
       \addr        Department of Applied Mathematics and Statistics, Department of Mathematics \\
       Johns Hopkins University,
       Baltimore, MD 21218, USA
       \AND 
         \name James M. Murphy \email jm.murphy@tufts.edu \\
        \addr  Department of Mathematics\\
       Tufts University, 
       Medford, MA 02139, USA\\
       }

\editor{}

\maketitle

\begin{abstract}

We consider the problem of clustering with the longest-leg path distance (LLPD) metric, which is informative for elongated and irregularly shaped clusters. We prove finite-sample guarantees on the performance of clustering with respect to this metric when random samples are drawn from multiple intrinsically low-dimensional clusters in high-dimensional space, in the presence of a large number of high-dimensional outliers.  By combining these results with spectral clustering with respect to LLPD, we provide conditions under which the Laplacian eigengap statistic correctly determines the number of clusters for a large class of data sets,  and prove guarantees on the labeling accuracy of the proposed algorithm.  Our methods are quite general and provide performance guarantees for spectral clustering with any ultrametric.  We also introduce an efficient, easy to implement approximation algorithm for the LLPD based on a multiscale analysis of adjacency graphs, which allows for the runtime of LLPD spectral clustering to be quasilinear in the number of data points.

\end{abstract}



\section{Introduction}

Clustering is a fundamental unsupervised problem in machine learning, seeking to detect group structures in data without any references or labeled training data.  Determining clusters can become harder as the dimension of the data increases:  one of the manifestations of the curse of dimension is that points drawn from high-dimensional distributions are far from their nearest neighbors, which can make noise and outliers challenging to address \citep{Hughes1968, Gyorfi2006, Bellman2015}.  However, many clustering problems for real data involve data that exhibit low dimensional structure, which can be exploited to circumvent the curse of dimensionality.  Various assumptions are imposed on the data to model low-dimensional structure, including requiring that the clusters be drawn from affine subspaces \citep{Parsons2004, Chen2009foundations, Chen2009spectral, Vidal2011, Zhang2012, Elhamifar2013, Wang2014, Soltanolkotabi2014} or more generally from low-dimensional mixture models \citep{McLachlan1988, Arias2011, AriasChen2011, Arias2017}.  

When the shape of clusters is unknown or deviates from both linear structures \citep{Vidal2011, Soltanolkotabi2012} or well-separated approximately spherical structures (for which $\numclust$-means performs well \citep{Mixon2017}), {\em{spectral clustering}} \citep{Ng2002, VonLuxburg2007} is a very popular approach, often robust with respect to the geometry of the clusters and of noise and outliers \citep{Arias2011,AriasChen2011}. 
Spectral clustering requires an initial distance or similarity measure, as it operates on a graph constructed between near neighbors measured and weighted based on such distance.
In this article, we propose to analyze low-dimensional clusters when spectral clustering is based on the {\em{longest-leg path distance}} (LLPD) metric, in which the distance between points $x,y$ is the minimum over all paths between $x,y$ of the longest edge in the path.  Distances in this metric exhibit stark phase transitions between within-cluster distances and between-cluster distances.  We are interested in performance guarantees with this metric which will explain this phase transition.  We prove theoretical guarantees on the performance of LLPD as a discriminatory metric, under the assumption that data is drawn randomly from distributions supported near low-dimensional sets, together with a possibly very large number of outliers sampled from a distribution in the high dimensional ambient space.  Moreover, we show that LLPD spectral clustering correctly determines the number of clusters and achieves high classification accuracy for data drawn from certain non-parametric mixture models.  The existing state-of-the-art for spectral clustering struggles in the highly noisy setting, in the case when clusters are highly elongated---which leads to large within-cluster variance for traditional distance metrics---and also in the case when clusters have disparate volumes.  In contrast, our method can tolerate a large amount of noise, even in its natural non-parametric setting, and it is essentially invariant to geometry of the clusters.

In order to efficiently analyze large datasets, a fast algorithm for computing LLPD is required.  Fast nearest neighbor searches have been developed for Euclidean distance on intrinsically low-dimensional sets (and other doubling spaces) using cover trees \citep{Beygelzimer2006}, among other popular algorithms (e.g. $k$-d trees \citep{Bentley1975}), and have been successfully employed in fast clustering algorithms.  These algorithms compute the $O(1)$ nearest neighbors for \emph{all points} in $O(n\log(n))$, where $n$ is the number of points, and are hence crucial to the scalability of many machine learning algorithms.   LLPD seems to require the computation of a minimizer over a large set of paths. We introduce here an algorithm for LLPD, efficient and easy to implement, with the same quasilinear computational complexity as the algorithms above: this makes LLPD nearest neighbor searches scalable to large data sets.  We moreover present a fast eigensolver for the (dense) LLPD graph Laplacian that allows for the computation of the approximate eigenvectors of this operator in essentially linear time. 


\subsection{Summary of Results}

The major contributions of the present work are threefold.  

First, we analyze the {\em{finite sample behavior of LLPD}} for points drawn according to a flexible probabilistic data model, with points drawn from low dimensional structures contaminated by a large number of high dimensional outliers. We derive bounds for maximal within-cluster LLPD and minimal between-cluster LLPD that hold with high probability, and also derive a lower bound for the minimal LLPD to a point's $\kNoise$ nearest neighbor in the LLPD metric. These results rely on a combination of techniques from manifold learning and percolation theory, and may be of independent interest.  

Second, we deploy these finite sample results to prove that, under our data model, the {\em{eigengap statistic for LLPD-based Laplacians correctly determines the number of clusters}}.  While the eigengap heuristic is often used in practice, existing theoretical analyses of spectral clustering fail to provide a rich class of data for which this estimate is provably accurate.  Our results regarding the eigengap are quite general and can be applied to give state-of-the-art performance guarantees for spectral clustering with any ultrametric, not just the LLPD.  Moreover, we prove that \emph{the LLPD-based spectral embedding learned by our method is clustered correctly by $\numclust$-means with high probability}, with misclassification rate improving over the existing state-of-the-art for Euclidean spectral clustering.  

Finally, we present a {\em{fast and easy to implement approximation algorithm for LLPD}}, based on a multiscale decomposition of adjacency graphs.  Let $\kLLPD$ be the number of LLPD nearest neighbors sought.  Our approach generates approximate $\kLLPD$-nearest neighbors in the LLPD at a cost of $O(n(\kEuc C_{\text{NN}}+m(\kEuc \vee\log(n))+\kLLPD)),$ where $n$ is the number of data points, $\kEuc $ is the number of nearest neighbors used to construct an initial adjacency graph on the data, $C_{\text{NN}}$ is the cost of a Euclidean nearest neighbor query, $m$ is related to the approximation scheme, and $\vee$ denotes the maximum.  Under the realistic assumption $\kEuc ,\kLLPD,m\ll\log(n)$, this algorithm is $O(n\log^2(n))$ for data with low intrinsic dimension.  If $\kEuc ,\kLLPD,m=O(1)$ with respect to $n$, this reduces to $O(n\log(n))$.  We quantify the resulting approximation error, which can be uniformly bounded independent of the data.  We moreover \emph{develop a fast eigensolver to compute the $\numclust$ principal eigenfunctions of the dense approximate LLPD Laplacian in $O(n(\kEuc C_{\text{NN}}+m(\kEuc \vee\log(n)\vee \numclust^{2})))$ time}.  If $\kEuc ,\numclust, m=O(1)$ with respect to $n$, this reduces to $O(n\log(n))$.  This allows for the fast computation of the eigenvectors without resorting to constructing a sparse Laplacian.  \emph{The proposed method is demonstrated on a variety of synthetic and real datasets, with performance consistently with our theoretical results}. 


\noindent{\bf{Article outline}}.
In Section \ref{sec:Background}, we present an overview of clustering methods, with an emphasis on those most closely related to the one we propose.  A summary of our data model and main results, together with motivating examples, are in Section \ref{sec:MajorContributions}.  In Section \ref{sec:FiniteSampleAnalysis}, we analyze the LLPD for non-parametric mixture models.  In Section \ref{sec:SpectralClusteringAnalysis}, performance guarantees for spectral clustering with LLPD are derived, including guarantees on when the eigengap is informative and on the accuracy of clustering the spectral embedding obtained from the LLPD graph Laplacian.  Section \ref{sec:Algorithm} proposes an efficient approximation algorithm for LLPD yielding faster nearest neighbor searches and computation of the eigenvectors of the LLPD Laplacian.  Numerical experiments on representative datasets appear in Section \ref{sec:NumericalExperiments}.  We conclude and discuss new research directions in Section \ref{sec:Conclusions}.

\subsection{Notation}

In Table \ref{tab:Notation}, we introduce notation we will use throughout the article.

\begin{table}[!htb]
\begin{center}
\begin{tabular}{r c p{10cm} }
\toprule
$X=\{x_{i}\}_{i=1}^{n}\subset\mathbb{R}^{D}$ & Data points to cluster\\
$d$ & Intrinsic dimension of cluster sets\\
$\numclust$ &  Number of clusters\\
$\{X_{l}\}_{l=1}^{\numclust}$ & Discrete data clusters \\
$\tilde{X}$ & Discrete noise data\\
$X_{N}$ & Denoised data; $X_{N} \subseteq X$\\
$N$ & Number of points remaining after denoising\\
$\nmin$ & Smallest number of points in a cluster\\
$\kEuc $  & Number of nearest neighbors in construction of initial NN-graph\\
$\kLLPD$  & Number of nearest neighbors for LLPD \\
$\kNoise$ & Number of nearest neighbors for LLPD denoising\\
$C_{\text{NN}}$ & Complexity of computing a Euclidean NN\\
$W$ & Weight matrix \\
$\Lsym$ & Symmetric normalized Laplacian\\
$\sigma$ & Scaling parameter in construction of weight matrix\\
$\{(\phi_{i},\lambda_{i})\}_{i=1}^{n}$ & Eigenvectors and eigenvalues of an $n\times n$ $\Lsym$\\
$\epsilonincluster$ & Maximum within cluster LLPD; see (\ref{e:epsilons})\\
$\epsilonnoiseknn$ & Minimum LLPD of noise points to $\kNoise$ nearest neighbor; see (\ref{e:epsilons})\\
$\epsilonbetweencluster$ & Minimum between cluster LLPD; see (\ref{e:epsilons})\\
$\epsilonsep$ & Minimum between cluster LLPD after denoising; see (\ref{equ:assump3})\\
$\delta$ & Minimum Euclidean distance between clusters; see Definition \ref{defn:LDLN}\\
$\thetathres$ & Denoising parameter; see Definition \ref{defn:DenoisedData}\\
$\zeta_n, \zeta_{\thetathres}$ & LDLN data cluster balance parameters; see (\ref{eqn:zetas})  \\
$\zeta_{N}$ & Empirical cluster balance parameter after denoising; see Assumption \ref{assumption:ultrametric}\\
$\rho$ & Arbitrary metric\\
$\rho_{\ell \ell}$ & LLPD metric; see Definition \ref{def:LLPD}\\
$\mathcal{H}^{d}$ & $d$-dimensional Hausdorff measure\\
$B_{\epsilon}(x)$ & $D$-dimensional ball of radius $\epsilon$ centered at $x$\\
$B_1$ & Unit ball, with dimension clear from context\\
$a \vee b$, $a \wedge b$ & Maximum, minimum of $a$ and $b$\\
$a \lesssim b$, $a \gtrsim b$ & $a \leq C b$ , $a \geq C b$ for some absolute constant $C>0$ \\
\bottomrule
\end{tabular}
\end{center}
\caption{Notation used throughout the article.\label{tab:Notation}}
\end{table}


\section{Background}\label{sec:Background}
\subsection{Background on Clustering}

The process of determining groupings within data and assigning labels to data points according to these groupings without supervision is called \emph{clustering} \citep{Hastie2009}.  It is a fundamental problem in machine learning, with many approaches known to perform well in certain circumstances, but not in others.  In order to provide performance guarantees, analytic, geometric, or statistical assumptions are placed on the data.  
Perhaps the most popular clustering scheme is $\numclust$-means \citep{Steinhaus1957, Friedman2001,Hastie2009}, together with its variants \citep{Ostrovsky2006, Arthur2007, Park2009}, which are used in conjunction with feature extraction methods.  This approach partitions the data into a user-specified number $\numclust$ groups, where the partition is chosen to minimize within-cluster dissimilarity: $C^{*}=\argmin_{C=\{C_{k}\}_{k=1}^{\numclust}} \sum_{k=1}^{\numclust}\sum_{x\in C_{k}}\|x-\bar{x}_{k}\|_{2}^{2}.$  Here, $\{C_k\}_{k=1}^\numclust$ is a partition of the points, $C_{k}$ is the set of points in the $k^{th}$ cluster and $\bar{x}_{k}$ denotes the mean of the $k^{th}$ cluster.  Unfortunately, the $\numclust$-means algorithm and its refinements perform poorly for datasets that are not the union of well-separated, spherical clusters, and are very sensitive to outliers.  In general, density-based methods such as density-based spatial clustering of applications with noise (DBSCAN) and variants \citep{Ester1996, Xu1998} or spectral methods  \citep{Shi2000, Ng2002} are required to handle irregularly shaped clusters.


\subsection{Hierarchical Clustering}
Hierarchical clustering algorithms build a family of clusters at distinct hierarchical levels.  Their results are readily presented as a \emph{dendrogram} (see Figure \ref{fig:Dendrogram}).  Hierarchical clustering algorithms can be \emph{agglomerative}, where  individual points start as their own clusters and are iteratively merged, or \emph{divisive}, where the full dataset is iteratively split until some stopping criterion is reached. It is often challenging to infer a global partition of the data from hierarchical algorithms, as it is unclear where to cut the dendrogram.  

\begin{figure}[!htb]
\centering
\begin{subfigure}{.32\textwidth}
\includegraphics[width=\textwidth]{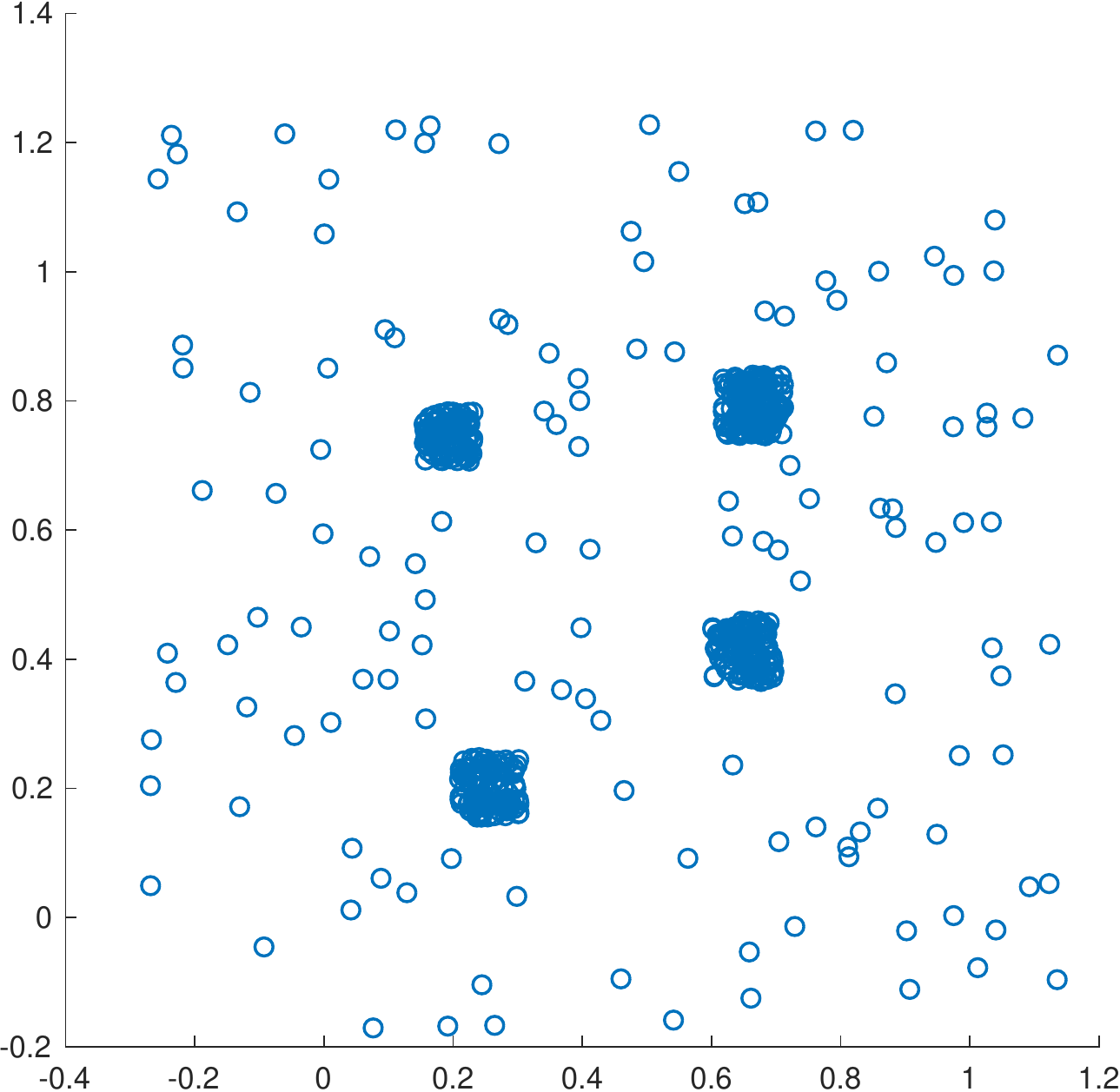}
\subcaption{Data to cluster}
\end{subfigure}
\qquad
\begin{subfigure}{.32\textwidth}
\includegraphics[width=\textwidth]{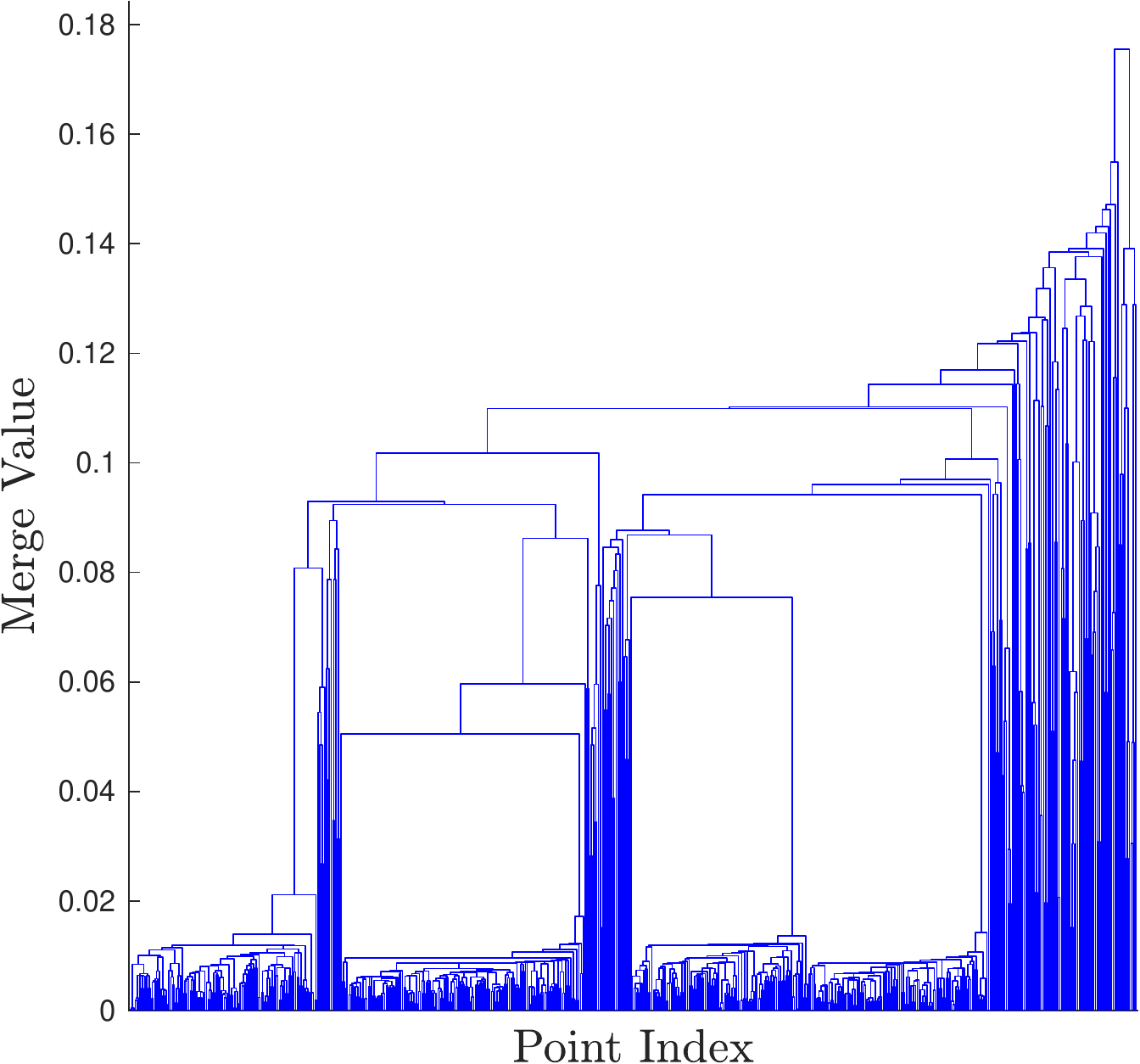}
\subcaption{Corresponding single linkage dendrogram}
\end{subfigure}
\caption{\label{fig:Dendrogram}Four two-dimensional clusters together with noise \citep{Zelnik2004} appear in (a).  In (b) is the corresponding single-linkage dendrogram.  Each point begins as its own cluster, and at each level of the dendrogram, the two nearest clusters are merged.  It is hard to distinguish between the noise and cluster points from the single linkage dendrogram, as it is not obvious where the four clusters are.
}
\end{figure}

For agglomerative methods, it must be determined which clusters ought to be merged at a given iteration.  This is done by a \emph{cluster dissimilarity metric} $\rho_{c}$.  For two clusters $C_{i},C_{j}$, $\rho_{c}(C_{i},C_{j})$ small means the clusters are candidates for merger.  Let $\rho_{X}$ be a metric defined on all the data points in $X$.  Standard $\rho_{c}$, and the corresponding clustering methods, include:
\begin{itemize}

\item[$\cdot$] $\rho_{SL}(C_{i},C_{j})=\min_{x_{i}\in C_{i}, x_{j}\in C_{j}}\rho_{X}(x_{i},x_{j})$: \emph{single linkage clustering}.

\item[$\cdot$] $\rho_{CL}(C_{i},C_{j})= \max_{x_{i}\in C_{i}, x_{j}\in C_{j}}\rho_{X}(x_{i},x_{j})$: \emph{complete linkage clustering}.

\item[$\cdot$] $\rho_{GA}(C_{i},C_{j})= \frac{1}{|C_{i}\|C_{j}|}\sum_{x_{i}\in C_{i}}\sum_{x_{j}\in C_{j}}\rho_{X}(x_{i},x_{j})$: \emph{group average clustering}.
\end{itemize}
In Section \ref{sec:Algorithm} we make theoretical and practical connections between the proposed method and single linkage clustering.  


\subsection{Spectral Clustering}\label{subsec:SpectralClustering}

\emph{Spectral clustering} methods \citep{Shi2000, Meila2001_Learning, Ng2002, VonLuxburg2007} use a spectral decomposition of an \emph{adjacency} or \emph{Laplacian matrix} to define an embedding of the data, and then cluster the embedded data using a standard algorithm, commonly $\numclust$-means.  The basic idea is to construct a weighted graph on the data that represents local relationships.  The graph has low edge weights for points far apart from each other and high edge weights for points close together.  This graph is then partitioned into clusters so that there are large edge weights within each cluster, and small edge weights between each cluster. Spectral clustering in fact relaxes an NP-hard graph partition problem \citep{Chung1997, Shi2000}. 

We now introduce notation related to spectral clustering that will be used throughout this work. Let $f_{\sigma}:\mathbb{R}\rightarrow[0,1]$ denote a kernel function with scale parameter $\sigma$.  Given a metric $\rho:\mathbb{R}^{D}\times\mathbb{R}^{D}\rightarrow [0,\infty)$ and some discrete set $X=\{x_{i}\}_{i=1}^{n}\subset\mathbb{R}^{D}$, let $W_{ij} = f_{\sigma}(\rho(x_{i},x_{j}))$ be the corresponding weight matrix.  Let $d_i = \sum_{j=1}^nW_{ij}$ denote the degree of point $x_i$, and define the diagonal degree matrix $\Deg_{ii} = d_i, \Deg_{ij}=0$ for $i\ne j$.  The graph Laplacian is then defined by $L=\Deg-W,$ which is often normalized to obtain the symmetric Laplacian $\Lsym=I-\Deg^{-\frac{1}{2}}W\Deg^{-\frac{1}{2}}$ or random walk Laplacian $L_{\text{RW}}=I-\Deg^{-1}W.$  Using the eigenvectors of $L$ to define an embedding leads to \textit{unnormalized} spectral clustering, whereas using the eigenvectors of $\Lsym$ or $L_{\text{RW}}$ leads to \textit{normalized} spectral clustering. While both normalized and unnormalized spectral clustering minimize between-cluster similarity, only normalized spectral clustering maximizes within-cluster similarity, and is thus preferred in practice \citep{VonLuxburg2007}.
 
In this article we consider spectral clustering with $\Lsym$ and construct the spectral embedding defined according to the popular algorithm of \citet{Ng2002}. When appropriate, we will use $\Lsym(X,\rho,f_{\sigma})$ to denote the matrix $\Lsym$ computed on the data set $X$ using metric $\rho$ and kernel $f_{\sigma}$. We denote the eigenvalues of $\Lsym$ (which are identical to those of $L_{\text{RW}}$) by $\lambda_1 \leq \ldots \leq \lambda_n$, and the corresponding eigenvectors by $\phi_1, \ldots, \phi_n$. To cluster the data into $\numclust$ groups according to \citet{Ng2002}, one first forms an $n\times \numclust$ matrix $\Phi$ whose columns are given by $\{\phi_i\}_{i=1}^\numclust$; these $\numclust$ eigenvectors are called the $\numclust$ \emph{principal eigenvectors}. The rows of $\Phi$ are then normalized to obtain the matrix $V$, that is $V_{ij} = {\Phi_{ij}}/({{\sum_j \Phi^2_{ij}}})^{1/2}$. Let $\{\mathbf{v}_i\}_{i=1}^n \in \mathbb{R}^{\numclust}$ denote the rows of $V$. Note that if we let $g:\mathbb{R}^D \rightarrow \mathbb{R}^{\numclust}$ denote the spectral embedding, $\mathbf{v}_i = g(x_i)$. Finally, $\numclust$-means is applied to cluster the $\{\mathbf{v}_i\}_{i=1}^n$ into $\numclust$ groups, which defines a partition of our data points $\{x_i\}_{i=1}^n$. One can use $L_{\text{RW}}$ similarly \citep{Shi2000}.  

Choosing $\numclust$ is an important aspect of spectral clustering, and various spectral-based mechanisms have been proposed in the literature \citep{azran2006spectral,azran2006new,Zelnik2004,sanguinetti2005automatic}.
The eigenvalues of $\Lsym$ have often been used to heuristically estimate the number of clusters as the largest empirical eigengap $\hat{\numclust}=\argmax_{i}\lambda_{i+1}-\lambda_{i},$ although there are many data sets for which this heuristic is known to fail \citep{VonLuxburg2007}; this estimate is called the \emph{eigengap} statistic. We remark that sometimes in the literature it is required that not only should $\lambda_{\hat{\numclust}+1}-\lambda_{\hat{\numclust}}$ be maximal, but also that $\lambda_{i}$ should be close to 0 for $i\le \hat{\numclust}$; we shall not make this additional assumption on $\lambda_{i}, i\le \hat{K}$, though we find in practice it is usually satisfied when the eigengap is accurate.  

A description of the spectral clustering algorithm of \cite{Ng2002} in the case that $\numclust$ is not known a priori appears in Algorithm \ref{alg:SC}; the algorithm can be modified in the obvious way if $\numclust$ is known and does not need to be estimated, or when using a sparse Laplacian, for example when $W$ is defined by a sparse nearest neighbors graph. 

\begin{algorithm}[tb]
	\caption{\label{alg:SC}Spectral Clustering with metric $\rho$}
	\textbf{Input:} $\{x_{i}\}_{i=1}^{n}$ (Data) , $\sigma>0$ (Scaling parameter)\\
	\textbf{Output:} $Y$ (Labels)\
	\begin{algorithmic}[1]
		\STATE Compute the weight matrix $W\in\mathbb{R}^{n\times n}$ with $W_{ij}=\exp(-\rho(x_{i},x_{j})^{2}/\sigma^{2})$.
		\STATE Compute the diagonal degree matrix $D\in\mathbb{R}^{n\times n}$ with $D_{ii}=\sum_{j=1}^{n}W_{ij}$.
		\STATE Form the symmetric normalized Laplacian $\Lsym=I-D^{-\frac{1}{2}}WD^{-\frac{1}{2}}$.
		\STATE Compute the eigendecomposition  $\{(\phi_{k},\lambda_{k})\}_{k=1}^{n}$, sorted so that $0=\lambda_{1}\le\lambda_{2}\le\dots\le\lambda_{n}$.
		\STATE Estimate the number of clusters $\numclust$ as $\hat{\numclust}=\argmax_{k}\lambda_{k+1}-\lambda_{k}$.
		\STATE For $1\leq i\leq n$, let $\mathbf{v}_i = (\phi_{1}(x_{i}),\phi_{2}(x_{i}),\dots,\phi_{\hat{\numclust}}(x_{i}))/||(\phi_{1}(x_{i}),\phi_{2}(x_{i}),\dots,\phi_{\hat{\numclust}}(x_{i}))||_2$ define the (row normalized) spectral embedding.
		\STATE Compute labels $Y$ by running $\numclust$-means on the data $\{\mathbf{v}_i \}_{i=1}^{n}$ using $\hat{\numclust}$ as the number of clusters. 
			\end{algorithmic}
\end{algorithm}

In addition to determining $\numclust$, performance guarantees for $\numclust$-means (or other clustering methods) on the spectral embedding is a topic of active research \citep{Schiebinger2015,Arias2017}.  However, spectral clustering typically has poor performance in the presence of noise and highly elongated clusters.


\subsection{Background on LLPD}

Many clustering and machine learning algorithms make use of Euclidean distances to compare points.  While universal and popular, this distance is data-independent, not adapted to the geometry of the data.  Many data-dependent metrics have been developed, for example diffusion distances \citep{Coifman2005, Coifman2006}, which are induced by diffusion processes on a dataset, and path-based distances \citep{Fischer2003,Chang2008}.  We shall consider a path-based distance for undirected graphs.

\begin{defn}
\label{def:LLPD}
For $X=\{x_{i}\}_{i=1}^{n}\subset\mathbb{R}^{D}$, let $G$ be the complete graph on $X$ with edges weighted by Euclidean distance between points.  For $x_{i}, x_{j}\in X$, let $\mathcal{P}(x_{i},x_{j})$ denote the set of all paths connecting $x_{i},x_{j}$ in $G$.  The \emph{longest-leg path distance (LLPD)} is:
\begin{align*}\rho_{\ell \ell}(x_{i},x_{j})=\min_{\{y_{l}\}_{l=1}^{L}\in\mathcal{P}(x_{i},x_{j})} \max_{l=1,2,\dots,L-1}\|y_{l+1}-y_{l}\|_{2}.\end{align*}
\end{defn}

In this article we use LLPD with respect to the Euclidean distance, but our results very easily generalize to other base distances. Our goal is to analyze the effects of transforming an original metric through the \emph{min-max} distance along paths in the definition of LLPD above.  We note that the LLPD is an \emph{ultrametric}, i.e. \begin{align}\label{defn:ultrametric} \forall x,y,z\in X \quad \rho_{\ell\ell}(x,y)\le \max\{\rho_{\ell\ell}(x,z),\rho_{\ell\ell}(y,z)\}\,.\end{align} 
This property is central to the proofs of Sections \ref{sec:FiniteSampleAnalysis} and \ref{sec:SpectralClusteringAnalysis}.  Figure \ref{fig:MoonDemo} illustrates how LLPD successfully differentiates elongated clusters, whereas Euclidean distance does not.

\begin{figure}[!htb]
\centering
\begin{subfigure}{.38\textwidth}
\includegraphics[width=\textwidth]{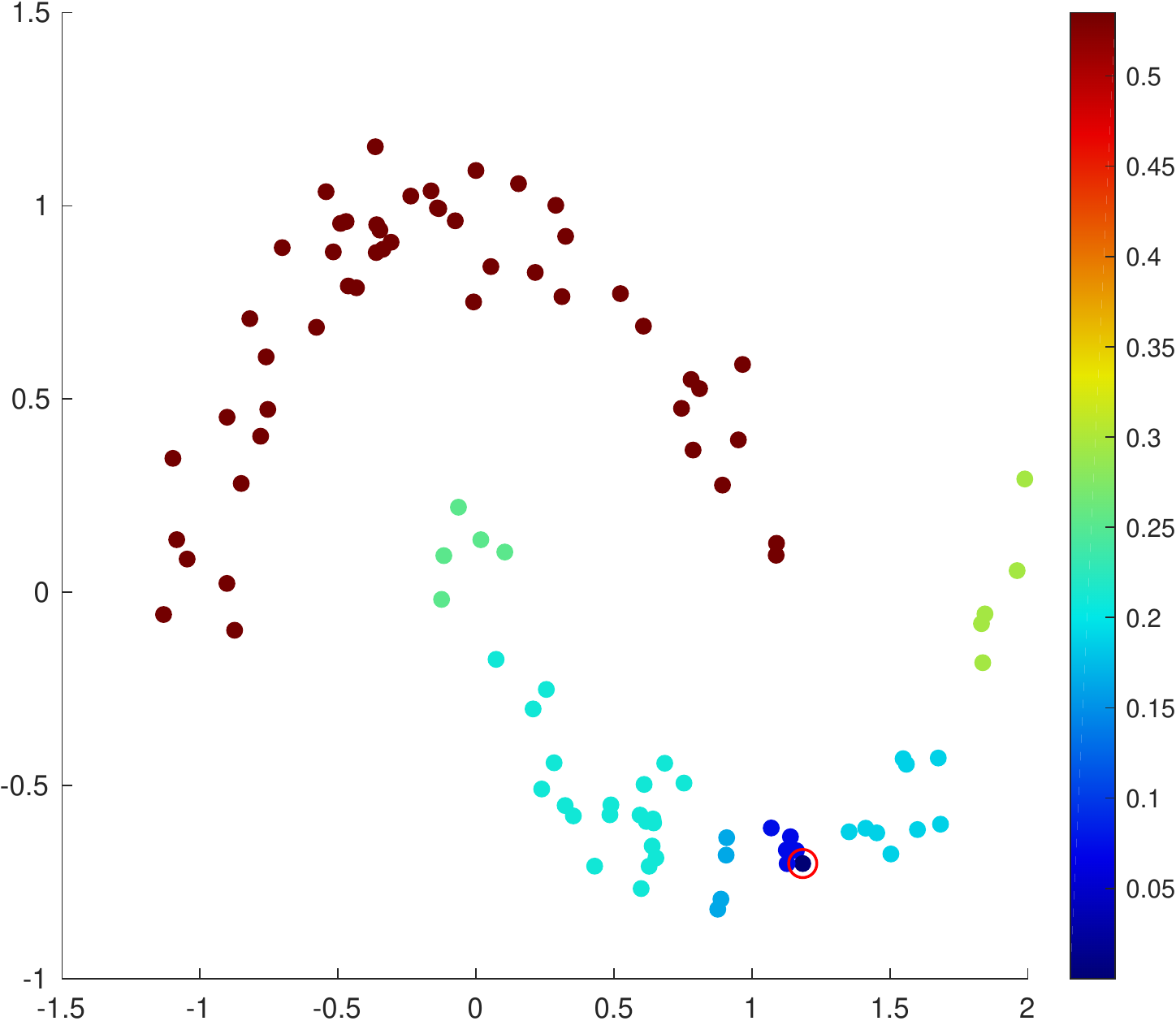}
\subcaption{LLPD from the marked point.}
\end{subfigure}
\qquad
\begin{subfigure}{.38\textwidth}
\includegraphics[width=\textwidth]{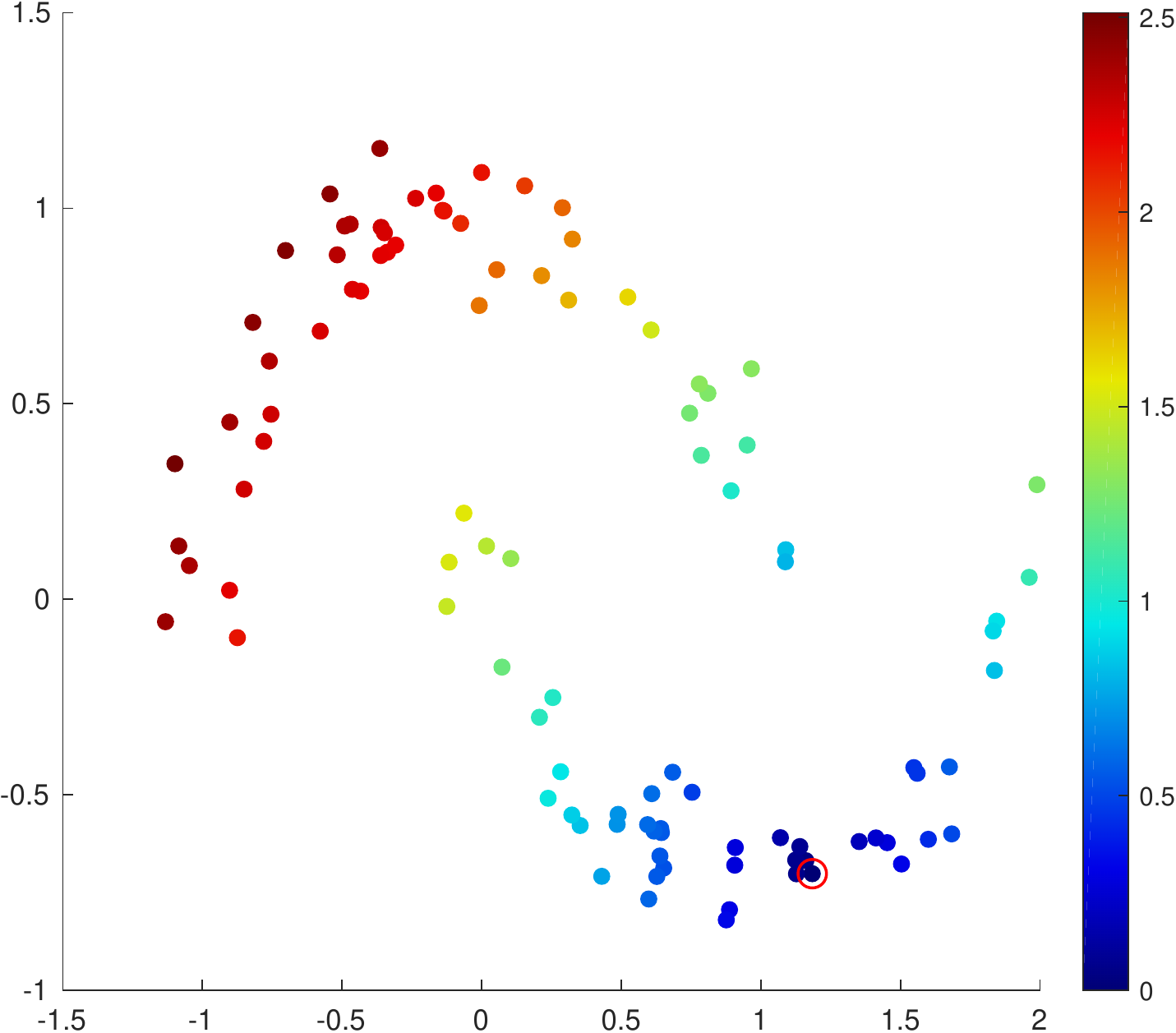}
\subcaption{Euclidean distances from the marked point.}
\end{subfigure}
\caption{\label{fig:MoonDemo}In this example, LLPD is compared with Euclidean distance.  The distance from the red circled source point is shown in each subfigure.  Notice that the LLPD has a phase transition that separates the clusters clearly, and that all distances within-cluster are comparable.
}
\end{figure}


\subsubsection{Probabilistic Analysis of LLPD}

Existing theoretical analysis of LLPD is based on studying the uniform distribution on certain geometric sets.  The degree and connectivity properties of near-neighbor graphs defined on points sampled uniformly from $[0,1]^{d}$ and their connections with percolation have been studied extensively \citep{Appel1997_max, Appel1997_min, Appel2002, Penrose1997, Penrose1999}.  Related results in the case of points drawn from low-dimensional structures were studied by \citet{Arias2011}.  These results motivate some of the ideas in this article; detailed references are given below when appropriate.  


\subsubsection{Spectral Clustering with LLPD}

Spectral clustering with LLPD has been shown to enjoy good empirical performance \citep{Fischer2001, Fischer2003, Fischer2004} and is made more robust by incorporating outlier removal \citep{Chang2008}.  The method and its variants generally perform well for non-convex and highly elongated clusters, even in the presence of noise.  However, no theoretical guarantees seem to be available.  Moreover, numerical implementation of LLPD spectral clustering appears underdeveloped, and existing methods have been evaluated mainly on small, low-dimensional datasets.  This article derives theoretical guarantees on performance of LLPD spectral clustering which confirms empirical insights, and also provides a fast implementation of the method suitable for large datasets.  

\subsubsection{Computing LLPD}
The problem of computing this distance is referred to by many names in the literature, including the maximum capacity path problem, the widest path problem, and the bottleneck edge query problem \citep{Pollack1960, Hu1961, Camerini1978, Gabow1988}.  A naive computation of LLPD distances is expensive, since the search space $\mathcal{P}(x,y)$ is potentially very large.  However, for a fixed pair of points $x,y$ connected in a graph $G=G(V,E)$, $\rho_{\ell \ell}(x,y)$ can be computed in $O(|E|)$ \citep{Punnen1991}.  There has also been significant work on the related problem of finding bottleneck spanning trees.  For a fixed root vertex $s\in V$, the \emph{minimal bottleneck spanning tree} rooted at $s$ is the spanning tree whose maximal edge length is minimal.  The bottleneck spanning tree can be computed in $O(\min\{n\log(n)+|E|, |E|\log(n)\})$ \citep{Camerini1978, Gabow1988}.

Computing all LLPDs for all points is the \emph{all points path distance (APPD)} problem.  Naively applying the bottleneck spanning tree construction to each point gives an APPD runtime of $O(\min\{n^{2}\log(n)+n|E|, n|E|\log(n)\})$.  However the APPD distance matrix can be computed in $O(n^{2})$, for example with a modified SLINK algorithm \citep{Sibson1973}, or with Cartesian trees \citep{Alon1987, Demaine2009, Demaine2014}.  We propose to approximate LLPD and implement LLPD spectral clustering with an algorithm near-linear in $n$, which enables the analysis of very large datasets (see Section \ref{sec:Algorithm}).


\section{Major Contributions}\label{sec:MajorContributions}

In this section we present a simplified version of our main theoretical result.  More general versions of these results, with detailed constants, will follow in Sections \ref{sec:FiniteSampleAnalysis} and \ref{sec:SpectralClusteringAnalysis}. We first discuss a motivating example and outline our data model and assumptions, which will be referred to throughout the article.


\subsection{Motivating Examples}
\begin{figure}[!htb]
	\begin{subfigure}[t]{.32\textwidth}
		\captionsetup{width=.95\linewidth}
		\includegraphics[width=\textwidth]{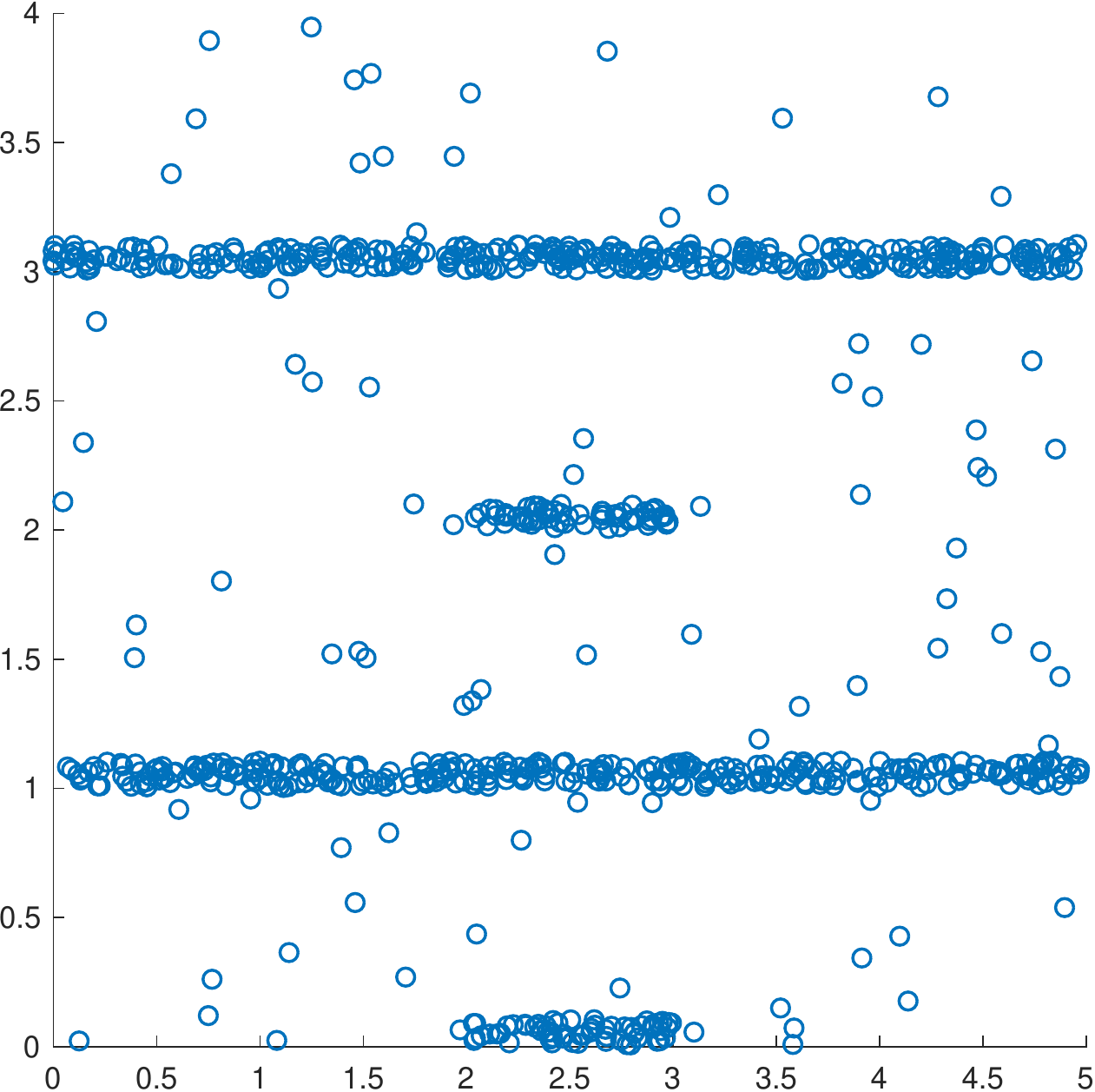}
		\subcaption{Original dataset.}
	\end{subfigure}
	\begin{subfigure}[t]{.32\textwidth}
		\captionsetup{width=.95\linewidth}
		\includegraphics[width=\textwidth]{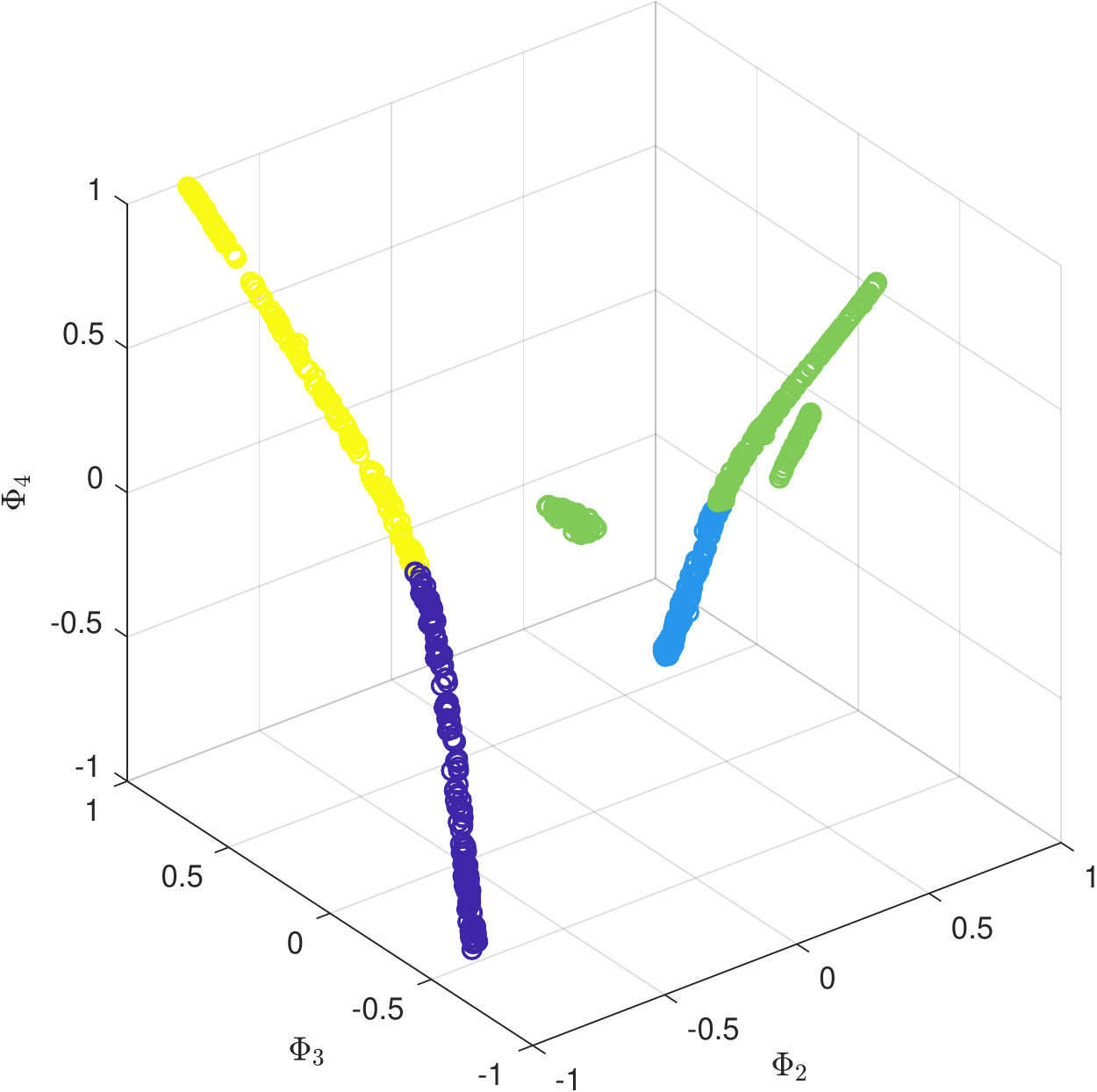}
		\subcaption{3 dimensional spectral embedding with Euclidean distances, labeled with $\numclust$-means.  The data has been denoised based on thresholding with Euclidean distances.}
	\end{subfigure}
	\begin{subfigure}[t]{.32\textwidth}
		\captionsetup{width=.95\linewidth}
		\includegraphics[width=\textwidth]{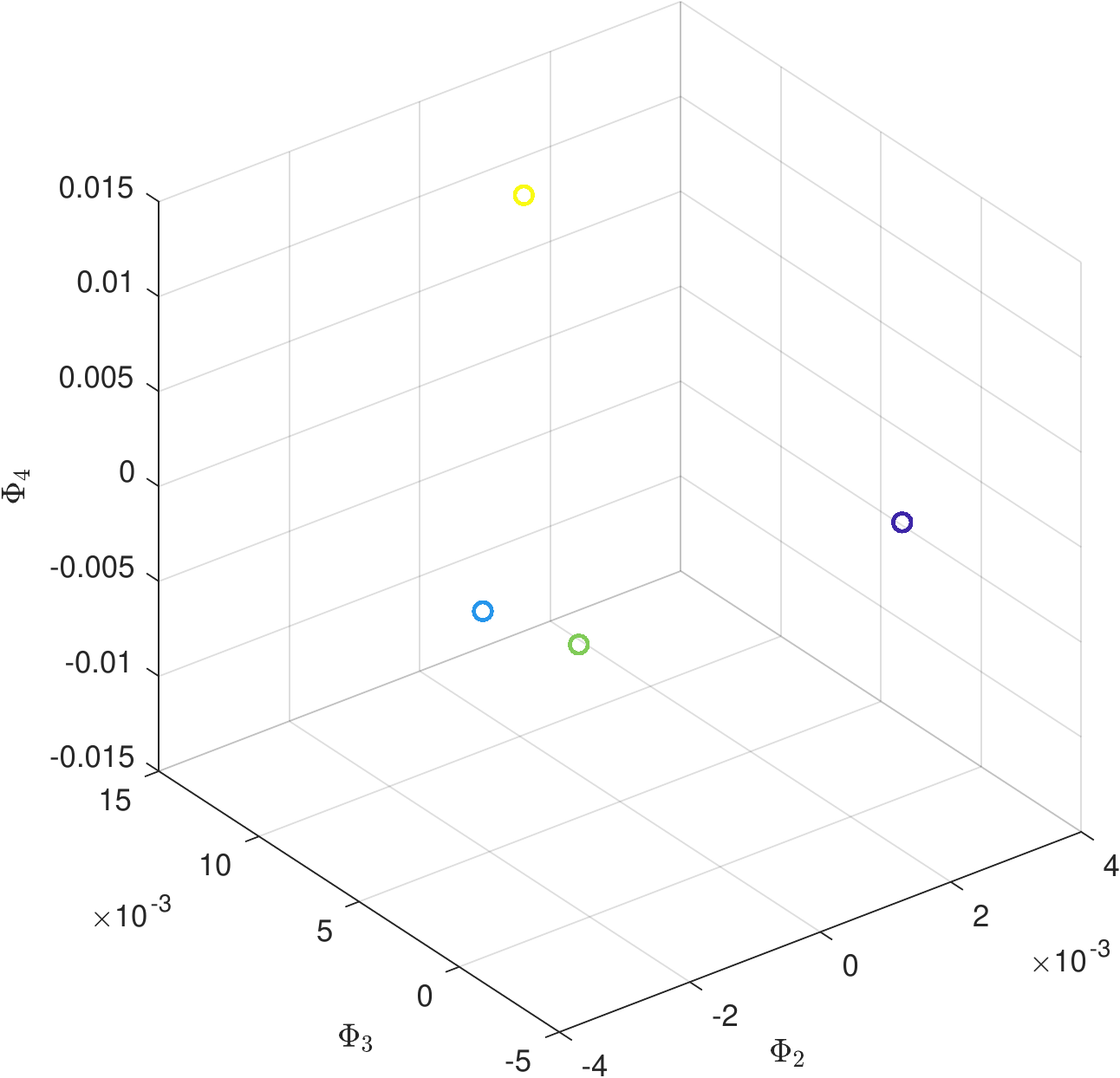}
		\subcaption{3 dimensional spectral embedding with LLPD, labeled with $\numclust$-means.   The data has been denoised based on thresholding with LLPD.}
	\end{subfigure}
	\begin{subfigure}[t]{.32\textwidth}
		\captionsetup{width=.95\linewidth}
		\includegraphics[width=\textwidth]{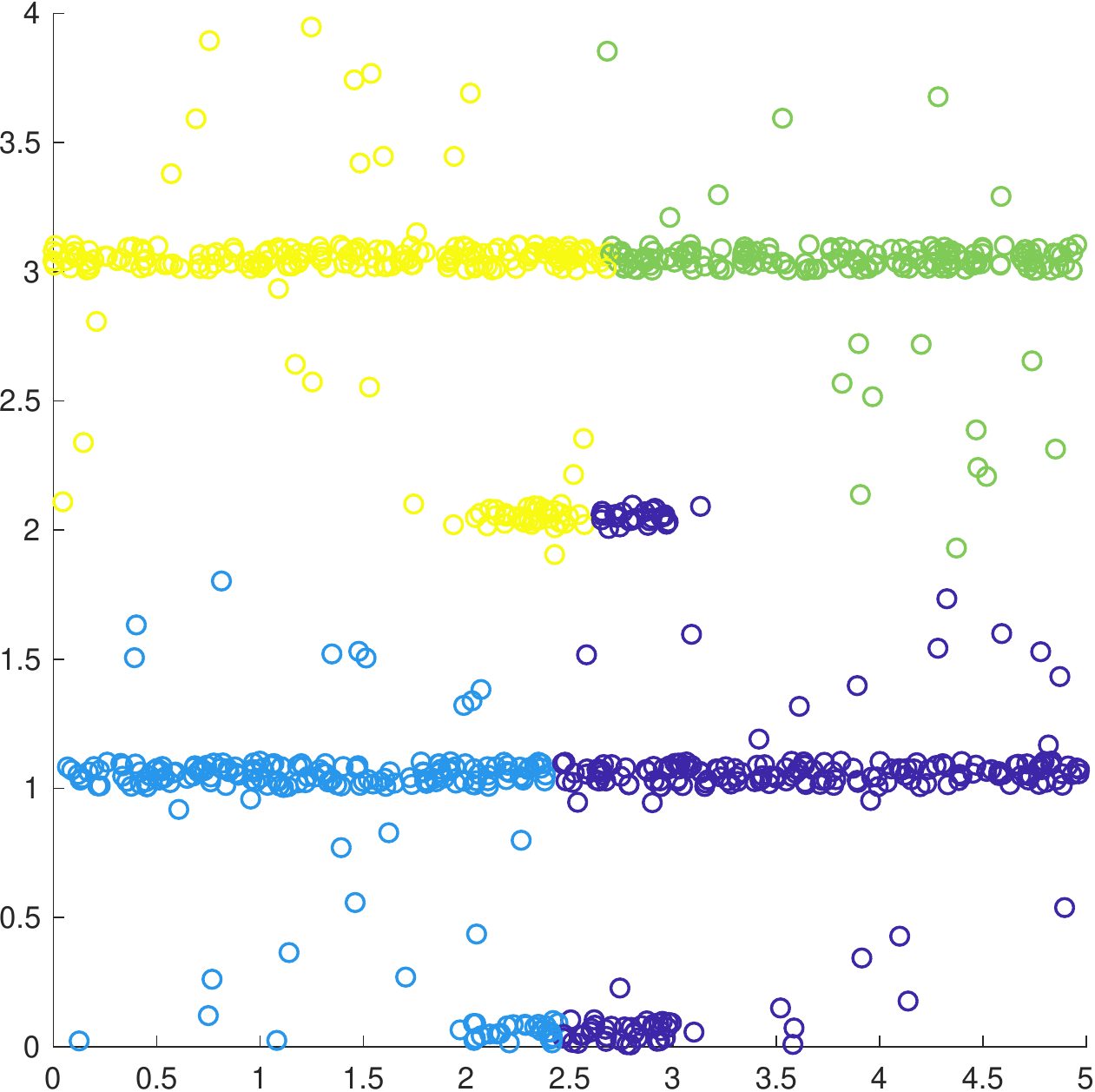}
		\subcaption{$\numclust$-means labels.}
	\end{subfigure}
	\begin{subfigure}[t]{.32\textwidth}
		\captionsetup{width=.95\linewidth}
		\includegraphics[width=\textwidth]{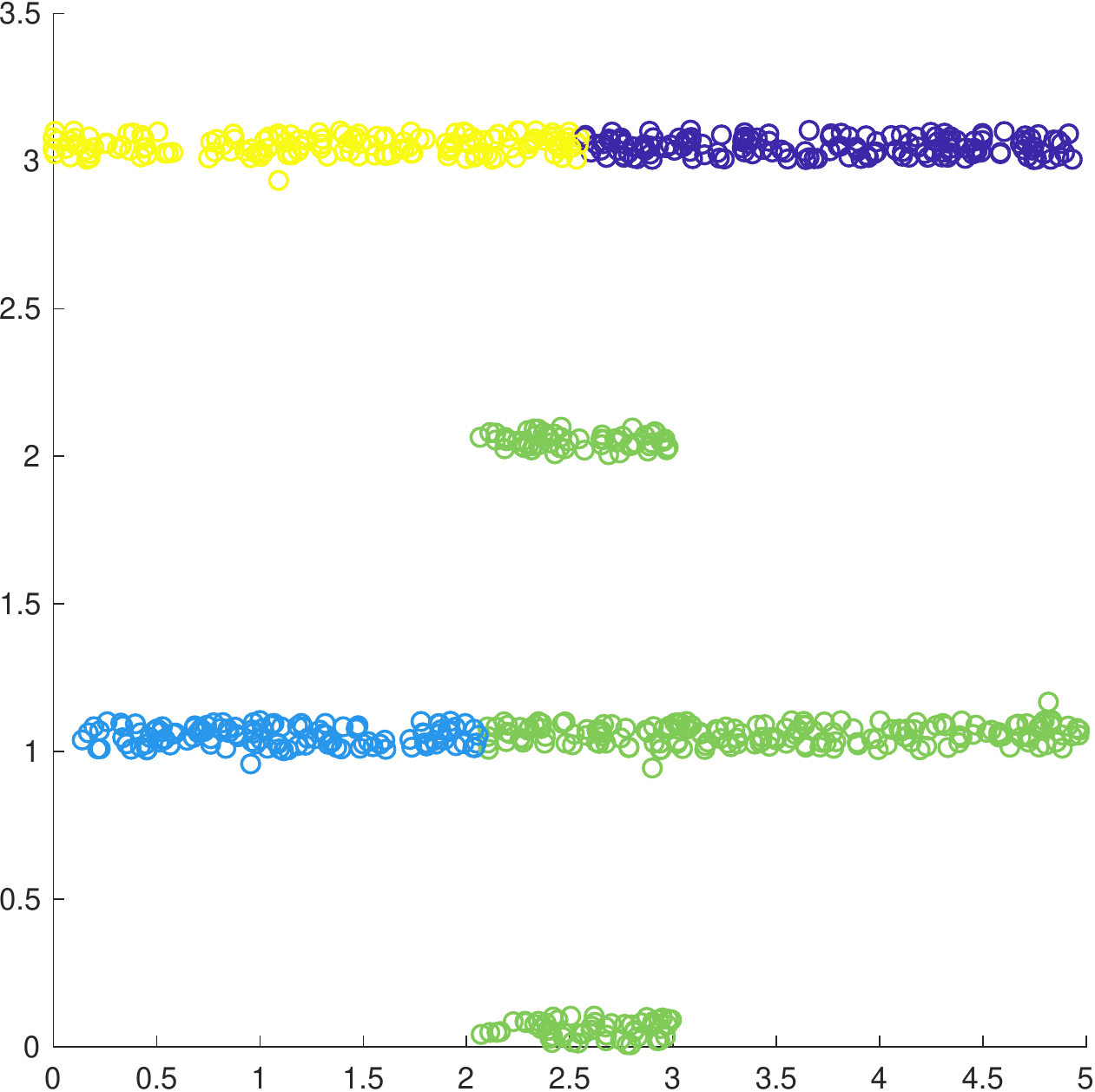}
		\subcaption{Spectral clustering results with Euclidean distances.}
	\end{subfigure}
	\begin{subfigure}[t]{.32\textwidth}
		\captionsetup{width=.95\linewidth}
		\includegraphics[width=\textwidth]{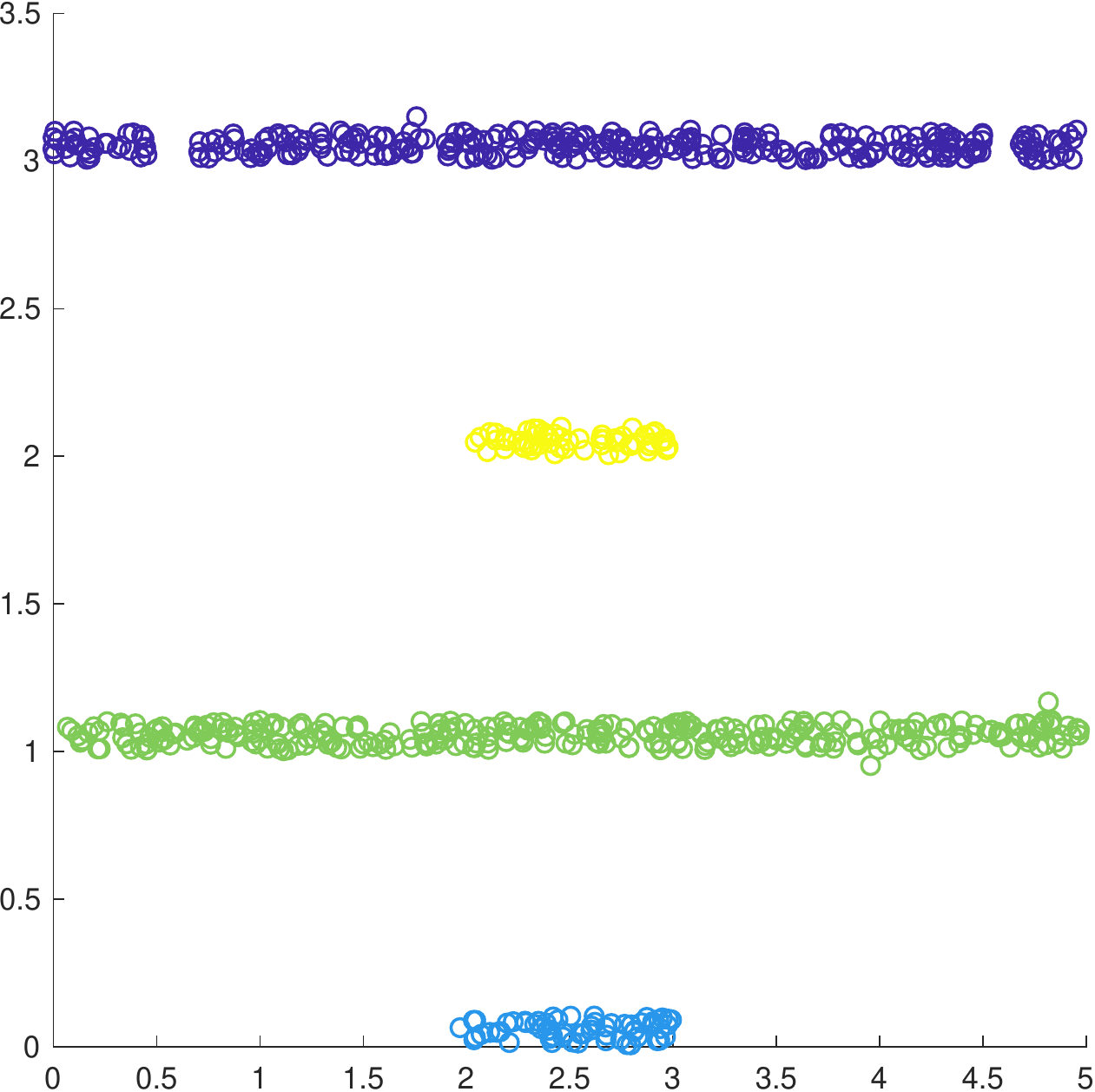}
		\subcaption{Spectral clustering results with LLPD.}
	\end{subfigure}
	\caption{The dataset consists of four elongated clusters in $\mathbb{R}^{2}$, together with ambient noise.  The labels given by $\numclust$-means are quite inaccurate, as are those given by regular spectral clustering.  The labels given by LLPD spectral clustering are perfect.  Note that $\Phi_{i}$ denotes the $i^{th}$ principal eigenvector of $\Lsym$.  For both variants of spectral clustering, the $\numclust$-means algorithm was run in the 4 dimensional embedding space given by the first 4 principal eigenvectors of $\Lsym$.}
	\label{fig:IntroductoryExample}		
\end{figure}

In this subsection we illustrate in which regimes LLPD spectral clustering advances the state-of-art for clustering. As will be explicitly described in Subsection \ref{subsec:assumptions}, we model clusters as connected, high-density regions, and we model noise as a low-density region separating the clusters. Our method easily handles highly elongated and irregularly shaped clusters, where traditional $\numclust$-means and even spectral clustering fail.  For example, consider the four elongated clusters in $\mathbb{R}^{2}$ illustrated in Figure \ref{fig:IntroductoryExample}. Both $\numclust$-means and Euclidean spectral clustering split one or more of the most elongated clusters, whereas the LLPD spectral embedding perfectly separates them.  Moreover, the eigenvalues of the LLPD Laplacian correctly infer there are 4 clusters, unlike the Euclidean Laplacian.

There are naturally situations where LLPD spectral clustering will not perform well, such as for certain types of structured noise. For example, consider the dumbbell shown in Figure \ref{fig:Dumbbell}.  When there is a high-density bridge connecting the dumbbell, LLPD will not be able to distinguish the two balls. However, it is worth noting that this property is precisely what allows for robust performance with elongated clusters, and that if the bridge has a lower density than the clusters, LLPD spectral clustering performs very well.

\begin{figure}[!htb]
	\centering
	\begin{subfigure}{.49\textwidth}
		\captionsetup{width=.95\linewidth}
		\includegraphics[width=\textwidth]{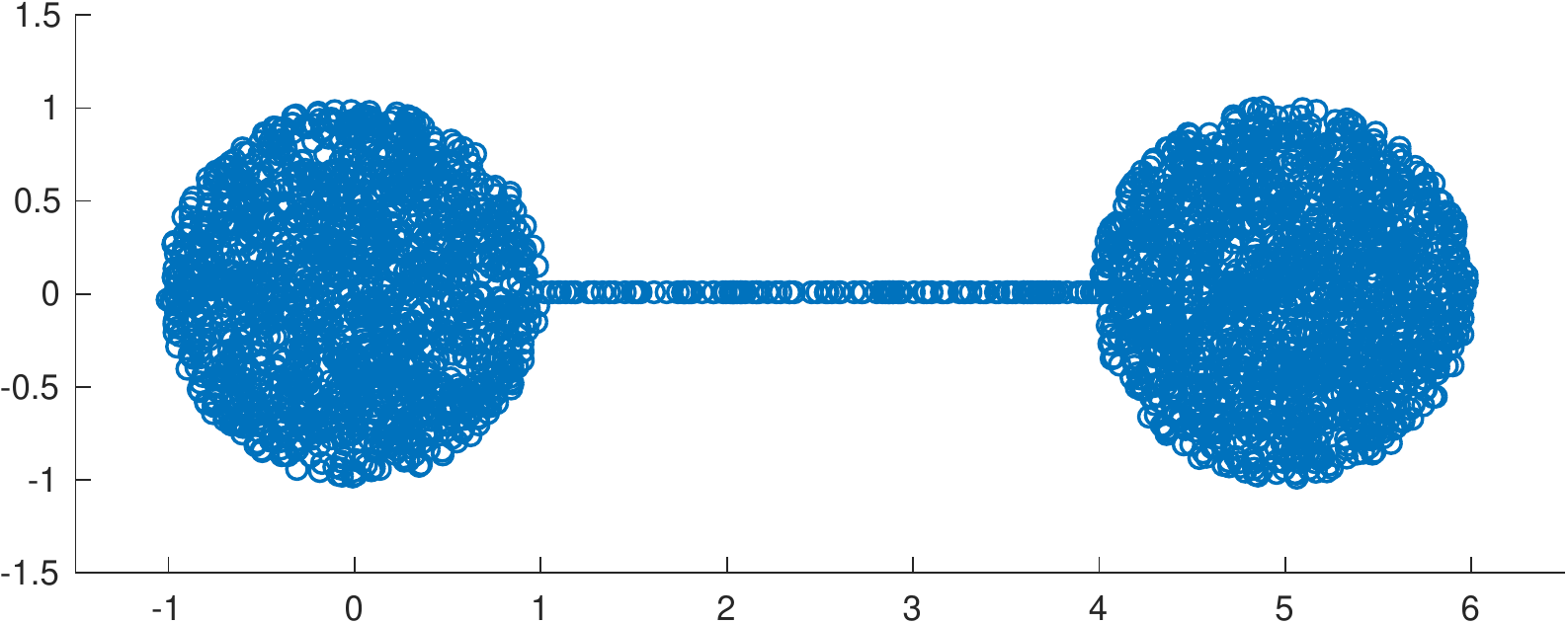}
		\subcaption{Two clusters connected by a bridge of roughly the same empirical density.}
	\end{subfigure}
	\begin{subfigure}{.49\textwidth}
		\captionsetup{width=.95\linewidth}
		\includegraphics[width=\textwidth]{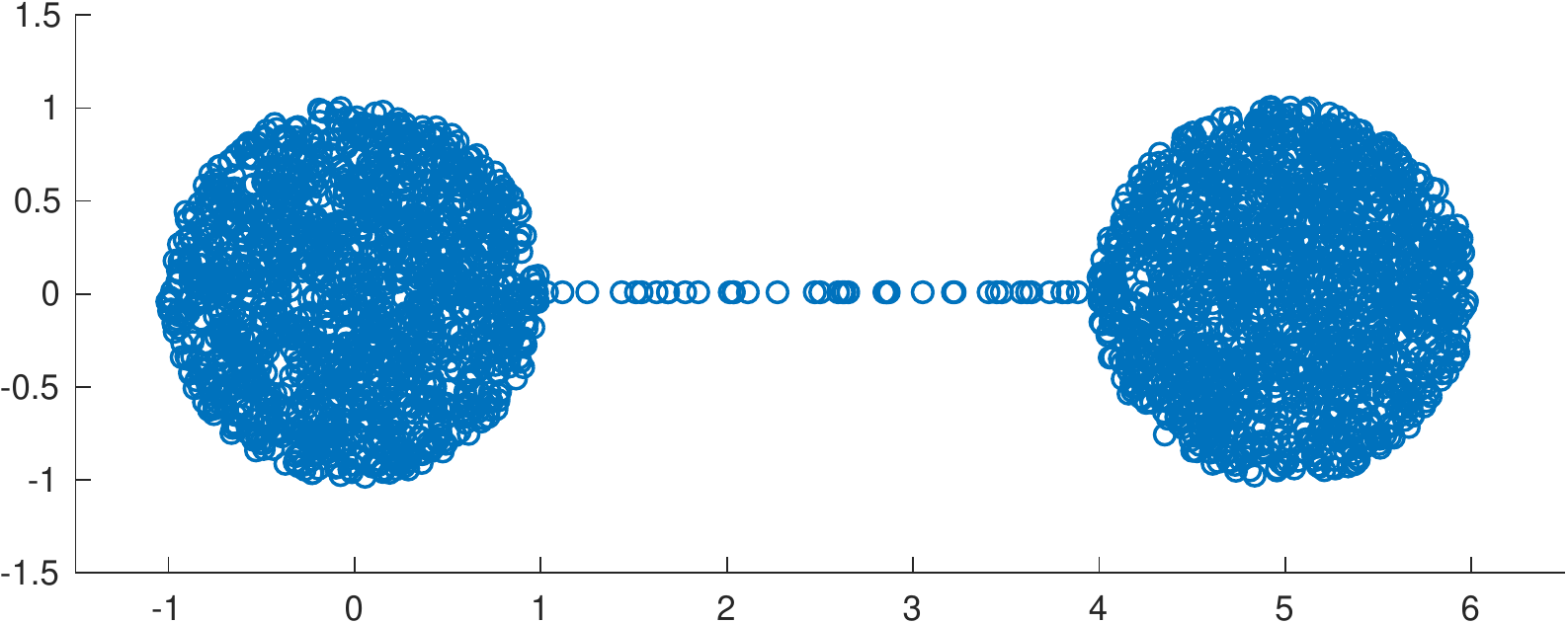}
		\subcaption{Two clusters connected by a bridge of lower empirical density.}
	\end{subfigure}
	\caption{In (a), two spherical clusters are connected with a bridge of approximately the same density; LLPD spectral clustering fails to distinguish between theses two clusters.  Despite the fact that the bridge consists of a very small number of points relative to the entire dataset, it is quite adversarial for the purposes of LLPD separation.  This is a limitation of the proposed method: it is robust to large amounts of diffuse noise, but not to a potentially small amount of concentrated, adversarial noise.  Conversely, if the bridge is of lower density, as in (b), then the proposed method will succeed.}
	\label{fig:Dumbbell}		
\end{figure} 


\subsection{Low Dimensional Large Noise (LDLN) Data Model and Assumptions}
\label{subsec:assumptions}

We first define the \emph{low dimensional, large noise (LDLN)} data model, and then establish notation and assumptions for the LLPD metric and denoising procedure on data drawn from this model.

We consider a collection of $\numclust$ disjoint, connected, approximately $d$-dimensional sets $\X_1, \ldots, \X_{\numclust}$ embedded in a measurable, $D$-dimensional ambient set $\X \subset \mathbb{R}^{D}$.  We recall the definition of \emph{$d$-dimensional Hausdorff measure} as follows \citep{Benedetto2010_Integration}.  For $A\subset\mathbb{R}^{D}$, let $\diam(A)=\sup_{x,y\in A}\|x-y\|_{2}$.  Fix $\delta>0$ and for any $A\subset\mathbb{R}^{D}$, let 
$$\vold_{\delta}(A)=\inf\left\{\sum_{i=1}^{\infty}\diam(U_{i})^{d}\ | \ A\subset \bigcup_{i=1}^{\infty}U_{i}, \ \diam(U_{i})<\delta\right\}.$$  
The $d$-dimensional Hausdorff measure of $A$ is $\vold(A)=\lim_{\delta\rightarrow 0^{+}}\vold_{\delta}(A)$.  Note that $\volD(A)$ is simply a rescaling of the Lebesgue measure in $\mathbb{R}^{D}$.  

\begin{defn}\label{defn:AdmissibleLowDimensionalSpace}A set $S\subset\mathbb{R}^{D}$ is an element of $\mathcal{S}_{d}(\kappa,\epsilon_0)$ for some $\kappa\ge 1$ and $\epsilon_{0}>0$ if it has finite $d$-dimensional Hausdorff measure, is connected, and:
$$\forall x\in S,\quad \forall \epsilon\in (0,\epsilon_{0}),\quad \ \kappa^{-1}\epsilon^{d}\le \frac{\vold(S \cap B_{\epsilon}(x))}{\vold(B_1)}\le \kappa \epsilon^{d}\,.$$\end{defn}

Note that $\mathcal{S}_{d}(\kappa,\epsilon_0)$ includes $d$-dimensional smooth compact manifolds (which have finite positive reach \citep{Federer1959_Curvature}). With some abuse of notation, we denote by $\Unif(S)$ the probability measure $\vold/\vold(S)$.  For a set $A$ and $\tau\ge0$, we define
\[ B(A,\tau) := \{ x \in \mathbb{R}^D : \exists y \in A \text{ with } \|x-y\|_2 \leq \tau\}.\]
Clearly $B(A,0)=A.$  

\begin{defn}[LDLN model]
\label{defn:LDLN}
The \emph{Low-Dimensional Large Noise (LDLN) model} consists of a $D$-dimensional ambient set $\X\subset\mathbb{R}^D$ and $\numclust$ cluster regions $\X_1,\dots,\X_\numclust\subset\X$ and noise set $\tilde{\X}\subset\mathbb{R}^D$ such that:
\begin{itemize}
\item[(i)] $0<\volD(\X)<\infty$;
\item[(ii)] $\X_{l}=B(S_{l},\tau)$ for $S_{l}\in \mathcal{S}_{d}(\kappa,\epsilon_{0})$, $l=1,\dots,\numclust$, $\tau\ge 0$ fixed;
\item[(iii)] $\tilde{\X} = \X\setminus ({\X}_1 \cup \ldots \cup {\X}_{\numclust})$;
\item[(iv)] the minimal Euclidean distance $\delta$ between two cluster regions satisfies
\[ \delta := \min_{l\ne s}\text{dist}({\X}_l,{\X}_s)=\min_{l\ne s}\min_{x\in {\X}_l, y \in {\X}_s}\|x-y\|_2>0. \] 
\end{itemize}
\end{defn}
Condition (i) says that the ambient set $\X$ is nontrivial and has bounded $D$-dimensional volume; condition (ii) says that the cluster regions behave like tubes of radius $\tau$ around well-behaved $d$-dimensional sets; condition (iii) defines the noise as consisting of the high-dimensional ambient region minus any cluster region; condition (iv) states that the cluster regions are well-separated.

\begin{defn}[LDLN data]
Given a LDLN model, LDLN data consists of sets $X_{l}$, each consisting of $n_l$ i.i.d. draws from $\Unif(\X_l),$ for $1\leq l\leq \numclust$, and $\tilde{X}$ consisting of $\tilde{n}$ i.i.d. draws from $\Unif(\tilde{\X})$.  We let $X=X_{1}\cup\dots\cup X_{\numclust}\cup\tilde{X}$, $n:=n_1+\ldots+n_{\numclust}+\tilde{n}, \nmin := \min_{1\leq l \leq \numclust} n_l$.  
\end{defn}

\begin{rems}
Although our model assumes sampling from a uniform distribution on the cluster regions, our results easily extend to any probability measure $\mu_l$ on $\X_l$ such that there exist constants $0<C_{1}\le C_{2}< \infty$ so that $C_1{\vold(S)}/{\vold(\X_l)} \leq \mu_l(S) \leq C_2{\vold(S)}/{\vold(\X_l)}$ for any measurable subset $S \subset \X_l$, and the same generalization holds for sampling from the noise set $\tilde{\X}$. The constants in our results change but nothing else; thus for ease of exposition we assume uniform sampling.   
\end{rems}

\begin{rems}
We could also consider a fully probabilistic model with the data consisting of $n$ i.i.d samples from a mixture model $\sum_{l=1}^{\numclust} q_l\Unif(\X_l)+\tilde q\Unif(\tilde \X)$, with suitable mixture weights $q_1,\dots,q_{\numclust},\tilde q$ summing to $1$. Then with high probability we would have $n_i$ (now a random variable) close to $q_i n$ and $\tilde n$ close to $\tilde q n$, falling back to the above case. We will use the model above in order to keep the notation simple.
\end{rems}

We define two cluster balance parameters for the LDLN data model:
\begin{align}\label{eqn:zetas}
\zeta_n &:= \frac{\sum_{l=1}^{\numclust}n_l}{n_{\min}}\quad,\quad \zeta_{\thetathres} := \frac{\sum_{l=1}^\numclust, p_{l,\thetathres}}{p_{\min,\thetathres}}
\end{align}
where $p_{l,\thetathres} := \volD(B(\X_l,\thetathres)\setminus\X_l)/\volD(\tilde{\X})$, $p_{\min,\thetathres} := \min_{1 \leq l \leq \numclust}p_{l,\thetathres}$, and $\thetathres$ is related to the denoising procedure (see Definition \ref{defn:DenoisedData}). The parameter $\zeta_n$ measures the balance of cluster sample size and the parameter $\zeta_{\thetathres}$ depends on the balance in surface area of the cluster sets $\X_l$. When all $n_i$ are equal and the cluster sets have the same geometry, $\zeta_n = \zeta_{\thetathres} = \numclust$.

Let $\rho_{\ell\ell}$ refer to LLPD in the full set $X$. For $A \subset X$, let $\rho_{\ell\ell}^A$ refer to LLPD when paths are restricted to being contained in the set $A$. For $x \in X$, let $\beta_{\kNoise}(x, A)$ denote the LLPD from $x$ to its $\kNoise^{\!\!\!\!\!\text{th}}$ LLPD-nearest neighbor when paths are restricted to the set $A$:

\[ \beta_{\kNoise}(x, A) := \min_{B\subset A\setminus\{x\}, |B| = \kNoise} \max_{y \in B} \rho_{\ell\ell}^{A}(x,y). \]

Let $\epsilonincluster$ be the maximal within-cluster LLPD, $\epsilonnoiseknn$ the minimal distance of noise points to their $\kNoise^{\!\!\!\!\!\text{th}}$ LLPD-nearest neighbor in the absence of cluster points, and $\epsilonbetweencluster$ the minimal between-cluster LLPD:
\begin{align}
\epsilonincluster :=  \max_{1\le l \le \numclust}\max_{x \ne y \in X_{l}} \rho_{\ell\ell}(x, y),\quad
\epsilonnoiseknn :=  \min_{x \in \tilde{X}} \beta_{\kNoise}(x, \tilde{X}), \quad 
\epsilonbetweencluster := \min_{l \ne l'}\min_{x\in \X_l, y\in \X_{l'}}\rho_{\ell\ell}(x,y)\,.
\label{e:epsilons}
\end{align}

\begin{defn}[Denoised LDLN data]\label{defn:DenoisedData}
We preprocess LDLN data (denoising) by removing any points that have a large LLPD to their $\kNoise^{\!\!\!\!\!\text{th}}$ LLPD-nearest neighbor, i.e. by removing all points $x\in X$ which satisfy $\beta_{\kNoise}(x, X) > \thetathres$ for some thresholding parameter $\thetathres$. We let $N \leq n$ denote the number of points which survive thresholding, and $X_N \subset X$ be the corresponding subset of points.
\end{defn}


\subsection{Overview of Main Results}

This article investigates geometric conditions implying $\epsilonincluster \ll \epsilonnoiseknn$ with high probability.  In this context higher density sets are separated by lower density regions; the points in these lower density regions will be referred to as noise and outliers interchangeably.  In this regime, noise points are identified and removed with high probability, leading to well-separated clusters that are internally coherent in the sense of having uniformly small within-cluster distances.  The proposed clustering method is shown to be highly robust to the choice of scale parameter in the kernel function, and to produce accurate clustering results even in the context of very large amounts of noise and highly nonlinear or elongated clusters.  Theorem \ref{thm:SimplifiedMainResult} simplifies two major results of the present article, Theorem \ref{thm:SpectralGapDataModel} and
Corollary \ref{cor:sigmarange}, which establish conditions guaranteeing two desirable properties of LLPD spectral clustering.  First, that the $\numclust^{\text{th}}$ eigengap of $\Lsym$ is the largest gap with high probability, so that the eigengap statistic correctly estimates the number of clusters.  Second, that embedding the data according to the principal eigenvectors of the LLPD Laplacian $\Lsym$ followed by a simple clustering algorithm correctly labels all points.  Throughout the theoretical portions of this article, we will define the accuracy of a clustering algorithm as follows. Let $\{y_{i}\}_{i=1}^{n}$ be ground truth labels taking values in $[\numclust] = \{1,\ldots,\numclust\}$, and  let $\{\hat{y}_{i}\}_{i=1}^{n} \in [\numclust]$ be the labels learned from running a clustering algorithm. Following \cite{abbe2018community}, we define the agreement function between $y$ and $\hat{y}$ as
	\begin{align}\label{eqn:agreement}
	A(y,\hat{y}) &= \max_{\pi\in \Pi^\numclust} \frac{1}{n}\sum_{i=1}^n \mathbbm{1}(\pi(\hat{y}_i)=y_i),
	\end{align}
	where the maximum is taken over all permutations of the label set $\Pi^\numclust$, and the \textit{accuracy} of a clustering algorithm as the value of the resulting agreement function. The agreement function can be computed numerically using the Hungarian algorithm \citep{Munkres1957_Algorithms}. If ground truth labels are only available on a data subset (as for LDLN data where noise points are unlabeled), then the accuracy is computed by restricting to the labeled data points.  In Section \ref{sec:NumericalExperiments}, additional notions of accuracy will be introduced for the empirical evaluation of LLPD spectral clustering.  
	
\begin{thm}
	\label{thm:SimplifiedMainResult}
	Under the LDLN data model and assumptions, suppose that the cardinality $\tilde n$ of the noise set and the tube radius $\tau$ are such that
 	$$\tilde{n}\leq \left(\frac{C_2}{C_1}\right)^{\frac{\kNoise D}{\kNoise+1}} \nmin^{\frac{D}{d+1}\left(\frac{\kNoise}{\kNoise+1}\right)}\quad,\quad\tau < \frac{C_1}{8} \nmin^{-(d+1)} \wedge \frac{\epsilon_0}{5}\,.$$ 
	Let $f_{\sigma}(x)= e^{-x^2/\sigma^2}$ be the Gaussian kernel and assume $\kNoise=O(1)$.  If $\nmin$ is large enough and $\thetathres, \sigma$ satisfy 
	\begin{align}
	C_1 \nmin^{-\frac{1}{d+1}} &\leq \thetathres \leq C_2\tilde{n}^{-\left(\frac{\kNoise+1}{\kNoise}\right)\frac{1}{D}} \label{equ:theta_range}\\
	C_3(\zeta_n+\zeta_\thetathres)\thetathres  &\leq \sigma \leq C_4 \delta (\log (\zeta_n+\zeta_\thetathres))^{-1/2} \label{equ:sigma_range}
	\end{align}
	then with high probability the denoised LDLN data $X_N$ satisfies:
	\begin{itemize}
			\item[(i)] the largest gap in the eigenvalues of $L_{\text{SYM}}(X_N, \rho^{X_N}_{\ell\ell}, f_{\sigma})$ is $\lambda_{\numclust+1}-\lambda_\numclust$. 
		\item[(ii)] spectral clustering with LLPD with $\numclust$ principal eigenvectors achieves perfect accuracy on  $X_N$. 
	\end{itemize}
	The constants $\{C_i\}_{i=1}^4$ depend on the geometric quantities $\numclust, d, D, \kappa, \tau, \{\vold(S_l)\}_{l=1}^{\numclust},\volD(\tilde{\X})$, but do not depend on $n_1, \ldots, n_{\numclust}, \tilde{n}, \thetathres, \sigma$.  
\end{thm}

Section \ref{sec:FiniteSampleAnalysis} verifies that with high probability a point's distance to its $\kNoise^{\!\!\!\!\!\text{th}}$ nearest neighbor (in LLPD) scales like $\nmin^{-(d+1)}$ for cluster points and $\tilde{n}^{-\left(\frac{\kNoise+1}{\kNoise}\right)\frac{1}{D}}$ for noise points; thus when the denoising parameter $\thetathres$ satisfies (\ref{equ:theta_range}), we successfully distinguish the cluster points from the noise points, and this range is large when the number of noise points $\tilde{n}$ is small relative to $\nmin^{\frac{D}{d+1}\left(\frac{\kNoise}{\kNoise+1}\right)}$. Thus, Theorem \ref{thm:SimplifiedMainResult} illustrates that when clusters are (intrinsically) low-dimensional, a number of noise points exponentially (in $D/d$) larger than $\nmin$ may be tolerated. If the data is denoised at an appropriate threshold level, the maximal eigengap heuristic correctly identifies the number of clusters and spectral clustering achieves high accuracy for any kernel scale $\sigma$ satisfying (\ref{equ:sigma_range}). This range for $\sigma$ is large whenever the cluster separation $\delta$ is large relative to the denoising parameter $\thetathres$.  We note that the case when $\kNoise$ is not $O(1)$ is discussed in Section \ref{subsubsec:ParameterSelection}.

In the noiseless case ($\tilde{n}=0$) when clusters are approximately balanced ($\zeta_n,\zeta_\thetathres = O(1)$),
Theorem \ref{thm:SimplifiedMainResult} can be further simplified as stated in the following corollary.  Note that no denoising is necessary in this case; one simply needs the kernel scale $\sigma$ to be not small relative to the maximal within cluster distance (which is upper bounded by $\nmin^{-(d+1)}$) and not large relative to the distance between clusters $\delta$. 

\begin{cor}[Noiseless, Balanced Case]
	\label{cor:SimplifiedMainResultCorollary}
	Under the LDLN data model and assumptions, further assume the cardinality of the noise set $\tilde{n}=0$ and the tube radius $\tau$ satisfies $\tau < \frac{C_1}{8} \nmin^{-(d+1)} \wedge \frac{\epsilon_0}{5}$. 
	Let $f_{\sigma}(x)= e^{-x^2/\sigma^2}$ be the Gaussian kernel and assume $\kNoise, \numclust, \zeta_n=O(1)$.  
	If $\nmin$ is large enough and $\sigma$ satisfies $$C_1 \nmin^{-\frac{1}{d+1}} \leq \sigma \leq C_4 \delta$$ for constants $C_1, C_4$ not depending on $n_1, \ldots, n_{\numclust}, \sigma$,  then with high probability the LDLN data $X$ satisfies:
	\begin{itemize}
		\item[(i)] the largest gap in the eigenvalues of $L_{\text{SYM}}(X, \rho^{X}_{\ell\ell}, f_{\sigma})$ is $\lambda_{\numclust+1}-\lambda_\numclust$. 
		\item[(ii)] spectral clustering with LLPD with $\numclust$ principal eigenvectors achieves perfect accuracy on  $X$. 
	\end{itemize}
\end{cor}

\begin{rems}
If one extends the LDLN model to allow the $S_l$ sets to have different dimensions $d_l$ and $\X_l$ to have different tube widths $\tau_l$, that is, $S_l \in \mathcal{S}_{d_l}(\kappa,\epsilon_0)$ and $\X_l = B(S_l, \tau_l)$,
Theorem \ref{thm:SimplifiedMainResult} still holds with $\max_l \tau_l$ replacing $\tau$ and $\max_l n_l^{-{1}/{(d_l+1)}}$ replacing $\nmin^{-{1}/{(d+1)}}$.  Alternatively, $\sigma$ can be set in a manner that adapts to local density \citep{Zelnik2004}.  
\end{rems}

\begin{rems}The constants in Theorem \ref{thm:SimplifiedMainResult} and Corollary \ref{cor:SimplifiedMainResultCorollary} have the following dimensional dependencies.  
	
	\begin{enumerate}
		
		\item $C_1 \lesssim\min_l({\kappa\vold(S_l)}/{\vold(B_1)})^{\frac{1}{d}}$ for $\tau=0$.  Letting $\text{rad}(\mathcal{M})$ denote the geodesic radius of a manifold $\mathcal{M}$, if $S_l$ is a complete Riemannian manifold with nonnegative Ricci curvature, then by the Bishop-Gromov inequality \citep{bishop2011geometry}, $({\vold(S_l)}/{\vold(B_1)})^{\frac{1}{d}} \leq \text{rad}(S_l)$; noting that $\kappa$ is at worst exponential in $d$, it follows that $C_1$ is then dimension independent for $\tau=0$.  For $\tau>0$, $C_1$ is upper bounded by an exponential in $D/d$.
		
		\item  $C_2\lesssim({\volD(\tilde{\X})}/{\volD(B_1)})^{\frac{1}{D}}$.  Assume $\volD(\tilde{\X})\gtrsim\volD(\X)$: if $\X$ is the unit $D$-dimensional ball, then $C_2$ is dimension independent; if $\X$ is the unit cube, then $C_2$ scales like $\sqrt{D}$. This illustrates that when $\X$ is not elongated in any direction, we expect $C_2$ to scale like $\text{rad}(\X)$.

		\item $C_3, C_4$ are independent of $d$ and $D$.
		
	\end{enumerate}
	
\end{rems}


\section{Finite Sample Analysis of LLPD}\label{sec:FiniteSampleAnalysis}

In this section we derive high probability bounds for the maximal within-cluster LLPD and the minimal between-cluster LLPD, and also derive a bound for the minimal $\kNoise^{\!\!\!\!\!\text{th}}$ LLPD-nearest neighbor distance. From these results we infer a sampling regime where LLPD is able to effectively differentiate between clusters and noise.

\subsection{Upper-Bounding Within-Cluster LLPD}


For bounding the within-cluster LLPD, we seek a uniform upper bound on $\rho_{\ell\ell}$ that holds with high probability.  The following two results are essentially Lemma 1 and Theorem 1 in \citet{Arias2011} with all constants explicitly computed; the proofs are in Appendix \ref{app:PDProofs}.

\begin{lem}\label{Lemma:BoundingConstants}Let $S\in \mathcal{S}_{d}(\kappa,\epsilon_0)$, and let $\epsilon,\tau>0$ with $\epsilon<\frac{2\epsilon_{0}}{5}$.  Then $\forall x\in B(S,\tau),$
\begin{align}\label{eqn:GeometricBound} C_{1}\epsilon^{d}(\tau\wedge\epsilon)^{D-d}\le \volD(B(S,\tau)\cap B_{\epsilon}(x))/\volD(B_1)\le C_{2}\epsilon^{d}(\tau\wedge\epsilon)^{D-d},\end{align}
for constants $C_{1}=\kappa^{-2}2^{-2D-d}, C_{2}=\kappa^{2}2^{2D+2d}$ independent of $\epsilon$.
\end{lem}

\begin{thm}\label{thm:WithinCluster}
Let $S\in \mathcal{S}_{d}(\kappa,\epsilon_0)$ and let $\tau>0$, $\epsilon<\epsilon_{0}$.  Let $x_{1},\dots,x_{n}\distras{i.i.d.} \Unif(B(S,\tau))$ and $C=\kappa^{2}2^{2D+d}$.  Then
\begin{small} 
$$n\ge  \frac{C \volD(B(S,\tau))}{ \left(\frac{\epsilon}{4}\right)^{d}(\tau \wedge \frac{\epsilon}{4})^{D-d}\volD(B_{1})}\log\frac{C\volD(B(S,\tau))}{\left(\frac{\epsilon}{8}\right)^{d}(\tau \wedge \frac{\epsilon}{8})^{D-d}\volD(B_{1})t}\implies \Prob(\max_{i,j} \rho_{\ell \ell}(x_{i},x_{j})< \epsilon)\ge 1-t\,.$$ 
\end{small}
\end{thm}

When $\tau$ is sufficiently small and ignoring constants, the sampling complexity suggested in Theorem \ref{thm:WithinCluster} depends only on $d$. The following corollary uses the above result to bound $\epsilonincluster$ in the LDLN data model; the proof also is given in Appendix \ref{app:PDProofs}.

\begin{cor}
\label{cor:WithinCluster_tau_small}
Assume the LDLN data model and assumptions, and let $0<\tau<\frac{\epsilon}{8} \wedge \frac{\epsilon_0}{5}$, $\epsilon<\epsilon_{0}$, and $C=\kappa^52^{4D+5d}$. Then
\begin{equation}
\label{equ:tau_small_sampling_cond}
n_l\ge  \frac{C \vold(S_l)}{ \left(\frac{\epsilon}{4}\right)^{d}\vold(B_{1})}\log\frac{C\vold(S_l)\numclust}{\left(\frac{\epsilon}{8}\right)^{d}\vold(B_{1})t}\quad\forall l=1,\dots,\numclust\implies \Prob(\epsilonincluster < \epsilon)\ge 1-t\,.
\end{equation}

\end{cor}	

The case $\tau=0$ corresponds to cluster regions being elements of $\mathcal{S}_{d}(\kappa,\epsilon_0)$, and is proved similarly to Theorem \ref{thm:WithinCluster} (the proof is omitted):

\begin{thm}
\label{thm:WithinLowDimCluster}
Let $S\in \mathcal{S}_{d}(\kappa,\epsilon_{0})$, $\tau=0$, and let  $\epsilon\in(0,\epsilon_{0})$.  Suppose $x_{1},\dots,x_{n}\distras{i.i.d.} \Unif(S)$.  Then 
$$n\ge  \frac{\kappa\vold(S)}{ \left(\frac{\epsilon}{4}\right)^{d}\vold(B_1)}\log\frac{\kappa\vold(S)}{ \left(\frac{\epsilon}{8}\right)^{d}\vold(B_1)t}\implies \Prob(\max_{i,j} \rho_{\ell \ell}(x_{i},x_{j})< \epsilon)\ge 1-t.$$ 
\end{thm}

Thus, up to geometric constants, for $\tau=0$ the uniform bound on LLPD depends only on the intrinsic dimension, $d$, not the ambient dimension, $D$.  When $d \ll D$, this leads to a huge gain in sampling complexity, compared to sampling in the ambient dimension.  


\subsubsection{Comparison with Existing Asymptotic Estimates}

To put Theorem \ref{thm:WithinLowDimCluster} in context, we remark on known asymptotic results for LLPD in the case $S=[0,1]^{d}$ \citep{Appel1997_max, Appel1997_min, Appel2002, Penrose1997, Penrose1999}.  Note that this assumes $\tau=0$, that is, the cluster is truly intrinsically $d$-dimensional.  Let $G_{n}^{d}$ denote a random graph with $n$ vertices, with edge weights $W_{ij}=\|x_{i}-x_{j}\|_{\infty}$, where $x_{1},\dots,x_{n}\distras{i.i.d.} \Unif([0,1]^{d})$.  For $\epsilon>0$, let $G_{n}^{d}(\epsilon)$ be the thresholded version of $G_{n}^{d}$, where edges with $W_{ij}$ greater than $\epsilon$ are deleted.  Define the random variable $c_{n,d}=\inf \{\epsilon>0 : \ G_{n}^{d}(\epsilon) \text{ is connected}\}.$  It is known \citep{Penrose1999} that $\max_{i,j}\rho_{\ell\ell}(x_{i},x_{j})=c_{n,d}$ for a fixed realization of the points $\{x_{i}\}_{i=1}^{n}$.  Moreover, \citet{Appel2002} showed $c_{n,d}$ has an almost sure limit in $n$: $$\lim_{n\rightarrow \infty} (c_{n,d})^{d}\frac{n}{\log(n)}=\begin{cases}1 & d=1,\\ \frac{1}{2d} & d\ge2.\end{cases}$$  Therefore $\max_{i,j}\rho_{\ell\ell}(x_{i},x_{j})\sim ({\log(n)}/{n})^{\frac{1}{d}}$, almost surely as $n\rightarrow\infty$.  Since the $\ell^{2}$ and $\ell^{\infty}$ norms are equivalent up to a $\sqrt{d}$ factor, a similar result holds in the case of $\ell^{2}$ norm being used for edge weights.  To compare this asymptotic limit with our results,  let $\epsilon_{*} = \max_{i,j}\rho_{\ell\ell}(x_{i},x_{j})$.  By Theorem \ref{thm:WithinCluster}, $\epsilon_{*}^{-d}\log(\epsilon_{*}^{-d}) \gtrsim n$. Since $\epsilon_{*} \sim \left({\log n}/{n}\right)^{\frac{1}{d}}$, $\epsilon_{*}^{-d}\log(\epsilon_{*}^{-d})\sim({n}/{\log n}) \log({n}/{\log n}) \sim n$ as $n \rightarrow \infty$. This shows that our lower bound for $\epsilon_{*}^{-d}\log(\epsilon_{*}^{-d})$ matches the one given by the asymptotic limit and is thus sharp.


\subsection{Lower-Bounding Between-Cluster Distances and $k$NN LLPD}
\label{BetweenClusterDis}

Having shown conditions guaranteeing that all points within a cluster are close together in the LLPD, we now derive conditions guaranteeing that points in different clusters are far apart in LLPD.  Points in the noise region may generate short paths between the clusters: we will upper-bound the number of between-clusters noise points that can be tolerated.  Our approach is related to percolation theory \citep{Gilbert1961, Roberts1968, Stauffer1994} and analysis of single linkage clustering \citep{Hartigan1981}. The following theorem is in fact inspired by Lemma 2 in \citet{Hartigan1981}.

  \begin{thm}
  \label{thm:BetweenClusters}
Under the LDLN data model and assumptions, with $\epsilonbetweencluster$ as in \eqref{e:epsilons}, for $\epsilon>0$ $$\tilde{n} \leq \frac{t^{\floor{\frac{\delta}{\epsilon}}^{-1}}\volD(\tilde{\X})}{\epsilon^D \volD(B_{1})}\implies \Prob\left(\epsilonbetweencluster>\epsilon\right)\geq 1-t\,.$$
\end{thm}

\begin{proof}We say that the ordered set of points $x_{i_1},\dots,x_{i_{\kNoise}}$ forms an $\epsilon$-chain of length $\kNoise$ if $\|x_{i_j}-x_{i_{j+1}}\|_{2} \leq \epsilon$ for $1\leq j \leq \kNoise-1$.
The probability that an ordered set of $\kNoise$ points forms an $\epsilon$-chain is bounded above by $\left(\frac{\volD(B_{\epsilon})}{\volD(\tilde{\X})}\right)^{\kNoise-1}$. There are $\frac{\tilde{n}!}{(\tilde{n}-\kNoise)!}$ ordered sets of $\kNoise$ points. Letting $A_{\kNoise}$ be the event that there exist $\kNoise$ points forming an $\epsilon$-chain of length $\kNoise$, we have $$ \Prob(A_{\kNoise}) \leq \frac{\tilde{n}!}{(\tilde{n}-\kNoise)!} \left(\frac{\volD(B_{\epsilon})}{\volD(\tilde{\X})}\right)^{\kNoise-1}
\leq \tilde{n} \left(\frac{\volD(B_1)}{\volD(\tilde{\X})}\tilde{n}\epsilon^D \right)^{\kNoise-1}.$$  Note that $A_{\kNoise+1} \subset A_{\kNoise}$.  In order for there to be a path between $\X_i$ and $\X_j$ (for some $i\neq j$) with all legs bounded by $\epsilon$, there must be at least $\floor{{\delta}/{\epsilon}}-1$
points in $\tilde{\X}$ forming an $\epsilon$-chain. Thus recalling $\epsilonbetweencluster = \min_{l \ne s}\min_{x\in \X_l, y\in \X_s}\rho_{\ell\ell}(x,y)$, we have: 
\begin{small}
$$\Prob\left(\epsilonbetweencluster \leq \epsilon \right)
\le \Prob\left(\bigcup_{\kNoise =  \floor{\frac{\delta}{\epsilon}}-1}^{\infty} A_{\kNoise} \right)
= \Prob\left(A_{\floor{\frac{\delta}{\epsilon}}-1} \right)
\leq \tilde{n} \left(\frac{\volD(B_1)}{\volD(\tilde{\X})} \tilde{n}\epsilon^D \right)^{\floor{\frac{\delta}{\epsilon}}-2}
\leq t
$$ 
\end{small}as long as $\log t \ge \log \tilde{n} + (\floor{{\delta}/{\epsilon}}-2)(\log \tilde{n} + \log \epsilon^D + \log{\volD(B_1)}/{\volD(\tilde{\X})})
$. A simple calculation proves the claim.
\end{proof}
\begin{rems}
\label{rem:BetweenClusters}
The above bound is independent of the number of clusters $\numclust$, as the argument is completely based on the minimal distance that must be crossed between-clusters. 
\end{rems}

Combining Theorem \ref{thm:BetweenClusters} with Theorem \ref{thm:WithinCluster} or \ref{thm:WithinLowDimCluster} allows one to derive conditions guaranteeing the maximal within cluster LLPD is smaller than the minimal between cluster LLPD with high probability, which in turn can be used to derive performance guarantees for spectral clustering on the cluster points. Since however it is not known a priori which points are cluster points, one must robustly distinguish the clusters from the noise.  We propose removing any point whose LLPD to its $\kNoise^{\!\!\!\!\!\text{th}}$ LLPD-nearest neighbor is sufficiently large (denoised LDLN data). The following theorem guarantees that, under certain conditions, all noise points that are not close to a cluster region will be removed by this procedure. The argument is similar to that in Theorem \ref{thm:BetweenClusters}, although we replace the notion of an $\epsilon$-chain of length $\kNoise$ with that of an $\epsilon$-group of size $\kNoise$.

\begin{thm}
\label{thm:knnLLPD}
Under the LDLN data model and assumptions, with $\epsilonnoiseknn$ as in \eqref{e:epsilons}, for $\epsilon>0$
$$\tilde{n} \leq \frac{2t^{\frac{1}{{\kNoise}+1}}}{({\kNoise}+1)} \left(\frac{ \volD(\tilde{\X})}{\volD(B_1)}\right)^{\frac{{\kNoise}}{{\kNoise}+1}} \epsilon^{-D\frac{{\kNoise}}{{\kNoise}+1}}\implies \Prob\left( \epsilonnoiseknn > \epsilon \right) \geq 1 - t\,.$$
\end{thm}
\begin{proof}
Let $\{x_i\}_{i=1}^{\tilde{n}}$ denote the points in $\tilde{X}$.  Let $A_{{\kNoise},\epsilon}$ be the event that there exists an $\epsilon$-group of size ${\kNoise}$, that is, there exist ${\kNoise}$ points such that the LLPD between all pairs is at most $\epsilon$. Note that $A_{{\kNoise},\epsilon}$ can also be described as the event that there exists an ordered set of ${\kNoise}$ points $x_{\pi_1}, \ldots, x_{\pi_{\kNoise}}$ such that $x_{\pi_i} \in \bigcup_{j=1}^{i-1} B_\epsilon(x_{\pi_j})$ for all $2 \leq i \leq {\kNoise}$.  Let $C_{\pi, i}$ denote the event that $x_{\pi_i} \in \bigcup_{j=1}^{i-1} B_\epsilon(x_{\pi_j})$.  For a fixed ordered set of points associated with the ordered index set $\pi$, we have
\begin{align*}
\Prob&\left(x_{\pi_i} \in \bigcup_{j=1}^{i-1} B_\epsilon(x_{\pi_j}) \text{ for } 2 \leq i \leq {\kNoise} \right)
= \Prob(C_{\pi, 2})\Prob( C_{\pi, 3} | C_{\pi, 2})\ldots\Prob\left(C_{\pi, {\kNoise}} | \bigcap_{j=2}^{{\kNoise}-1}C_{\pi, j}\right) \\
&\leq \frac{\volD(B_{\epsilon})}{\volD(\tilde{\X})}\left(2\frac{\volD(B_{\epsilon})}{\volD(\tilde{\X})}\right)\ldots \left(({\kNoise}-1)\frac{\volD(B_{\epsilon})}{\volD(\tilde{\X})}\right)
= ({\kNoise}-1)! \left(\frac{\volD(B_{\epsilon})}{\volD(\tilde{\X})}\right)^{{\kNoise}-1}. 
\end{align*}
There are $\frac{\tilde{n}!}{(\tilde{n}-{\kNoise})!}$ ordered sets of ${\kNoise}$ points, so that
\begin{align*}
\Prob(A_{{\kNoise},\epsilon}) 
\leq \frac{\tilde{n}!}{(\tilde{n}-{\kNoise})!}({\kNoise}-1)! \left(\frac{\volD(B_{\epsilon})}{\volD(\tilde{\X})}\right)^{{\kNoise}-1}
\leq \tilde{n}({\kNoise}-1)! \left(\frac{\volD(B_{1})}{\volD(\tilde{\X})}\tilde{n}\epsilon^D\right)^{{\kNoise}-1}
\leq  t
\end{align*}
as long as $\tilde{n}({\kNoise}-1)! \left(\frac{\volD(B_{1})}{\volD(\tilde{\X})}\tilde{n}\epsilon^D\right)^{{\kNoise}-1} \leq t,$ which occurs if $\tilde{n} \leq \frac{2t^{\frac{1}{{\kNoise}}} \volD(\tilde{\X})^{\frac{{\kNoise}-1}{{\kNoise}}} }{{\kNoise} \volD(B_1)^{\frac{{\kNoise}-1}{{\kNoise}}}\epsilon^{\frac{D({\kNoise}-1)}{{\kNoise}} } }$ for $\kNoise\geq 2$.  Since $\Prob( \epsilonnoiseknn > \epsilon)=\Prob( \min_{x \in \tilde{X}} \beta_{\kNoise}(x,\tilde{X}) > \epsilon) = 1-\Prob(A_{{\kNoise}+1, \epsilon})$, the theorem holds for ${\kNoise}\geq 1$.
\end{proof}





\begin{rems}
\label{rem:kNNLLPD}
The theorem guarantees $\epsilonnoiseknn \geq \left(\frac{ 2\volD(\tilde{\X})(2t)^{\frac{1}{{\kNoise}}} }{\volD(B_1)(({\kNoise}+1)\tilde{n})^{ \frac{{\kNoise}+1}{{\kNoise}} }}\right)^{\frac{1}{D}}$ with probability at least $1-t$.  The lower bound for $\epsilonnoiseknn$ is maximized at the unique maximizer in ${\kNoise}>0$ of $f({\kNoise}) = { (2t)^{\frac{1}{{\kNoise}}} }{(({\kNoise}+1)\tilde{n})^{-\frac{{\kNoise}+1}{{\kNoise}}}},$ which occurs at the positive root ${\kNoise}_{*}$ of $\kNoise - \log({\kNoise}+1) = \log \tilde{n} - \log (2t)$. Notice that ${\kNoise}_{*}=O(\log \tilde{n})$, so we may, and will, restrict our attention to ${\kNoise} \leq {\kNoise}_{*}= O(\log \tilde{n})$.
\end{rems}


\subsection{Robust Denoising with LLPD}
Combining Corollary \ref{cor:WithinCluster_tau_small} ($\tau>0$ but small) or Theorem \ref{thm:WithinLowDimCluster} ($\tau=0$) with Theorem \ref{thm:knnLLPD} determines how many noise points can be tolerated while within-cluster LLPD remain small relative to $\kNoise^{\!\!\!\!\!\text{th}}$ nearest neighbor LLPD of noise points.  Any $C\geq 1$ in the following theorem guarantees $\epsilonincluster < \epsilonnoiseknn$; when $C\gg 1$, $\epsilonincluster \ll \epsilonnoiseknn$, and LLPD easily differentiates the clusters from the noise. The proof is given in Appendix \ref{app:PDProofs}.  A similar result for the set-up of Theorem \ref{thm:WithinCluster} is omitted for brevity.

\begin{thm}
\label{thm:CombinedResult}
Assume the LDLN data model and assumptions, and define $$\tau_*:=\max_{l=1,\dots,k}\left(\frac{ \kappa^52^{4D+5d} \vold(S_l)}{n_l\vold(B_1)}\log \left(2^d n_l\frac{2\numclust}{t}\right)\right)^{\frac{1}{d}}\,.$$
Let $0\le\tau< \frac{\tau_*}{8} \wedge \frac{\epsilon_0}{5}$ and let $\epsilonincluster, \epsilonnoiseknn$ as in \eqref{e:epsilons}.  For any $C>0$,$$\tilde{n} <  \frac{ \left(\frac{t}{2}\right)^{\frac{1}{\kNoise+1}}}{\kNoise+1} \left(\frac{\volD(\tilde{\X})}{\volD(B_1)}\right)^{\frac{\kNoise}{\kNoise+1}} \left(\frac{1}{C\tau_*}\right)^{D\frac{\kNoise}{\kNoise+1}}\!\!\!\!\!\!\!\! \implies \Prob( C\epsilonincluster < \epsilonnoiseknn ) \geq 1- t\,.$$
\end{thm}

Ignoring $\log$ terms and geometric constants, the number of noise points $\tilde{n}$ can be taken as large as $\min_l n_l^{\frac{D}{d}\left(\frac{\kNoise}{\kNoise+1}\right)}$. Hence if  $d\ll D$, an enormous amount of noise points are tolerated while $\epsilonincluster$ is still small relative to $\epsilonnoiseknn$.  This result is deployed to prove LLPD spectral clustering is robust to large amounts of noise in Theorem \ref{thm:SpectralGapDataModel} and Corollary \ref{cor:sigmarange}, and is in particular relevant to condition (\ref{equ:thetarange}), which articulates the range of denoising parameters for which LLPD spectral clustering will perform well.


\subsection{Phase Transition in LLPD}\label{subsec:PhaseTransition}

In this section, we numerically validate our denoising scheme on simple data. The data is a mixture of five uniform distributions: four from non-adjacent edges of $[0,1]\times[0,\frac{1}{2}]\times[0,\frac{1}{2}] $, and one from the interior of $[0,1]\times[0,\frac{1}{2}]\times[0,\frac{1}{2}] $.  Each distribution contributed 3000 sample points.  Figure \ref{fig:phasetrans_data} shows the data and Figure \ref{fig:sortedLLPDs} all sorted LLPDs.  The sharp phase transition is explained mathematically by Theorems \ref{thm:WithinLowDimCluster} and \ref{thm:BetweenClusters}.  Indeed, $d=1, D=3$ in this example, so Theorem \ref{thm:WithinLowDimCluster} guarantees that with high probability, the maximum within cluster LLPD, call it $\epsilonincluster$, scales as $\epsilonincluster^{-1} \log (\epsilonincluster^{-1}) \gtrsim n$ while Theorem \ref{thm:BetweenClusters} guarantees that with high probability, the minimum between cluster LLPD, call it $\epsilonbetweencluster$, scales as $\epsilonbetweencluster\gtrsim n^{-\frac{1}{3}}$.  The empirical estimates can be compared with the theoretical guarantees, which are shown on the plot.  The guarantees require a confidence level, parametrized by $t$; this parameter was chosen to be $t=.01$ for this example.  The solid red line denotes the maximum within cluster LLPD guaranteed with probability exceeding $1-t=.99$, and the dashed red line denotes the minimum between cluster LLPD, guaranteed with probability exceeding $1-t$.  It is clear from Figure \ref{fig:PhaseTransitionData} that the theoretical lower lower bound on $\epsilonbetweencluster$ is rather sharp, while the theoretical upper bound on $\epsilonincluster$ is looser.  Despite the lack of sharpness in estimating $\epsilonincluster$, the theoretical bounds  are quite sufficient to separate the within-cluster and between cluster LLPD.  When $d\ll D$, the difference between these theoretical bounds becomes much larger.

\begin{figure}[!htb]
\begin{subfigure}{.32\textwidth}
\includegraphics[width=\textwidth]{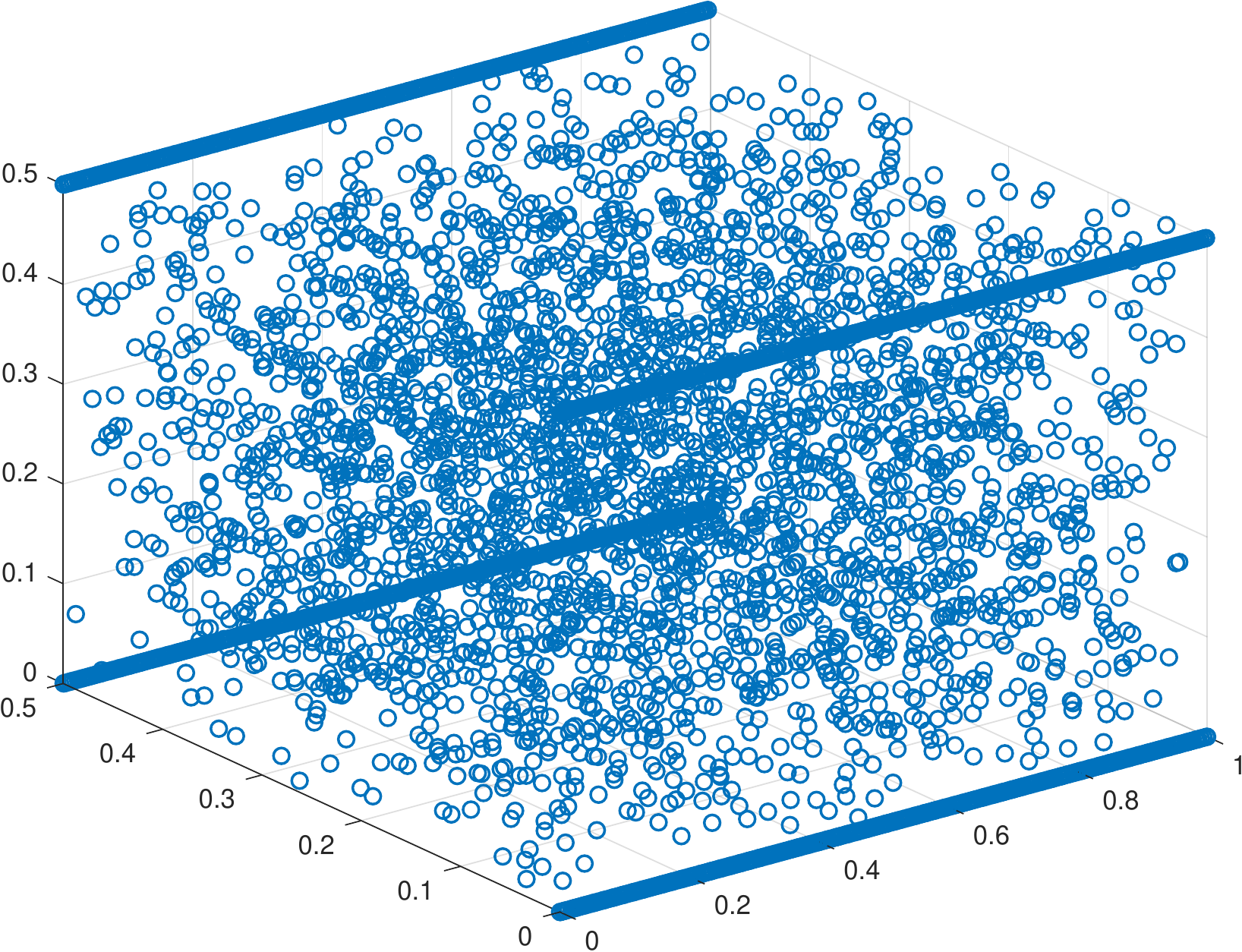}
\subcaption{Four clusters in $[0,1]\times [0,\frac{1}{2}]\times [0,\frac{1}{2}]$.}
\label{fig:phasetrans_data}
\end{subfigure}
\begin{subfigure}{.32\textwidth}
\includegraphics[width=\textwidth]{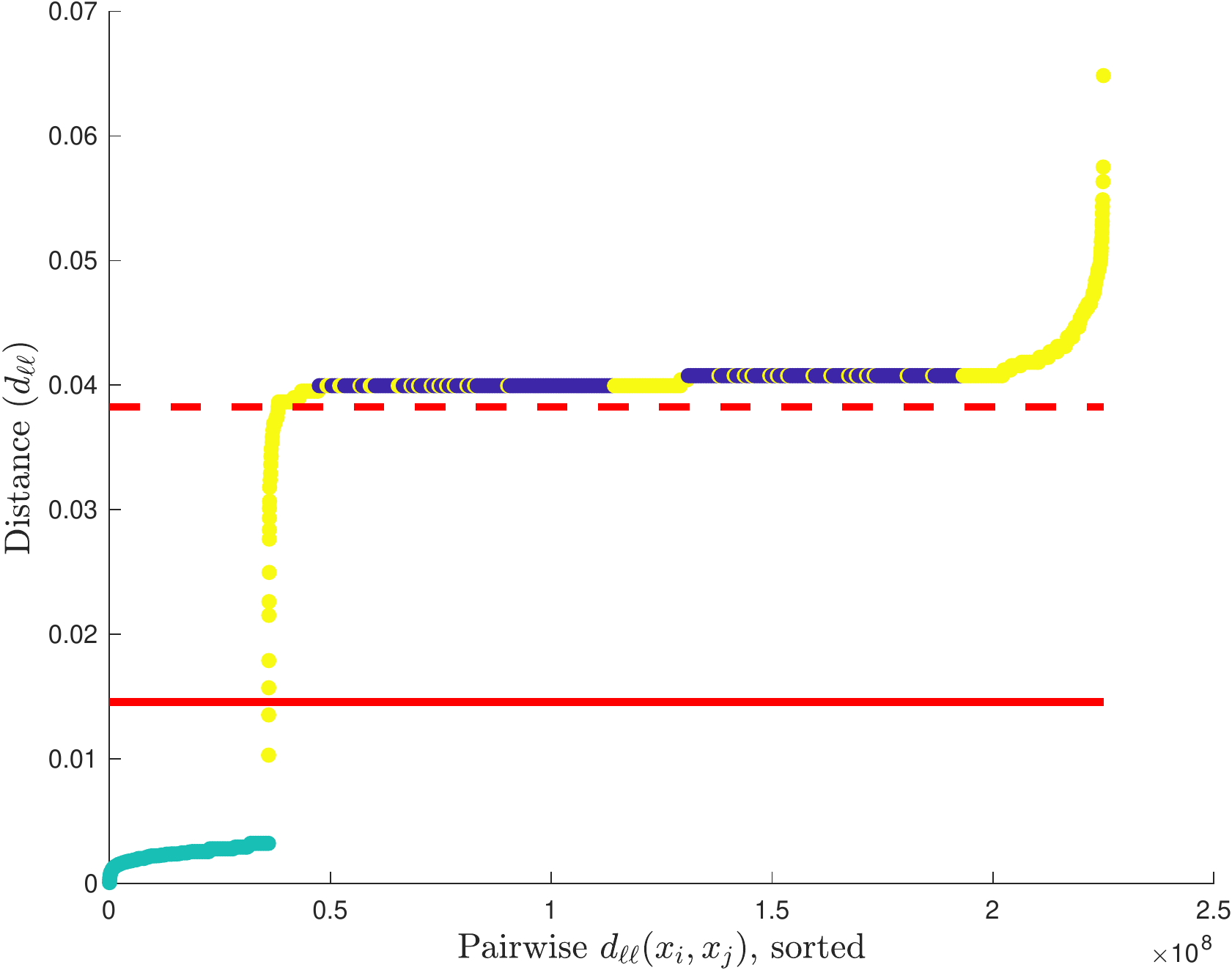}
\subcaption{Corresponding pairwise LLPD, sorted.}
\label{fig:sortedLLPDs}
\end{subfigure}
\begin{subfigure}{.32\textwidth}
\includegraphics[width=\textwidth]{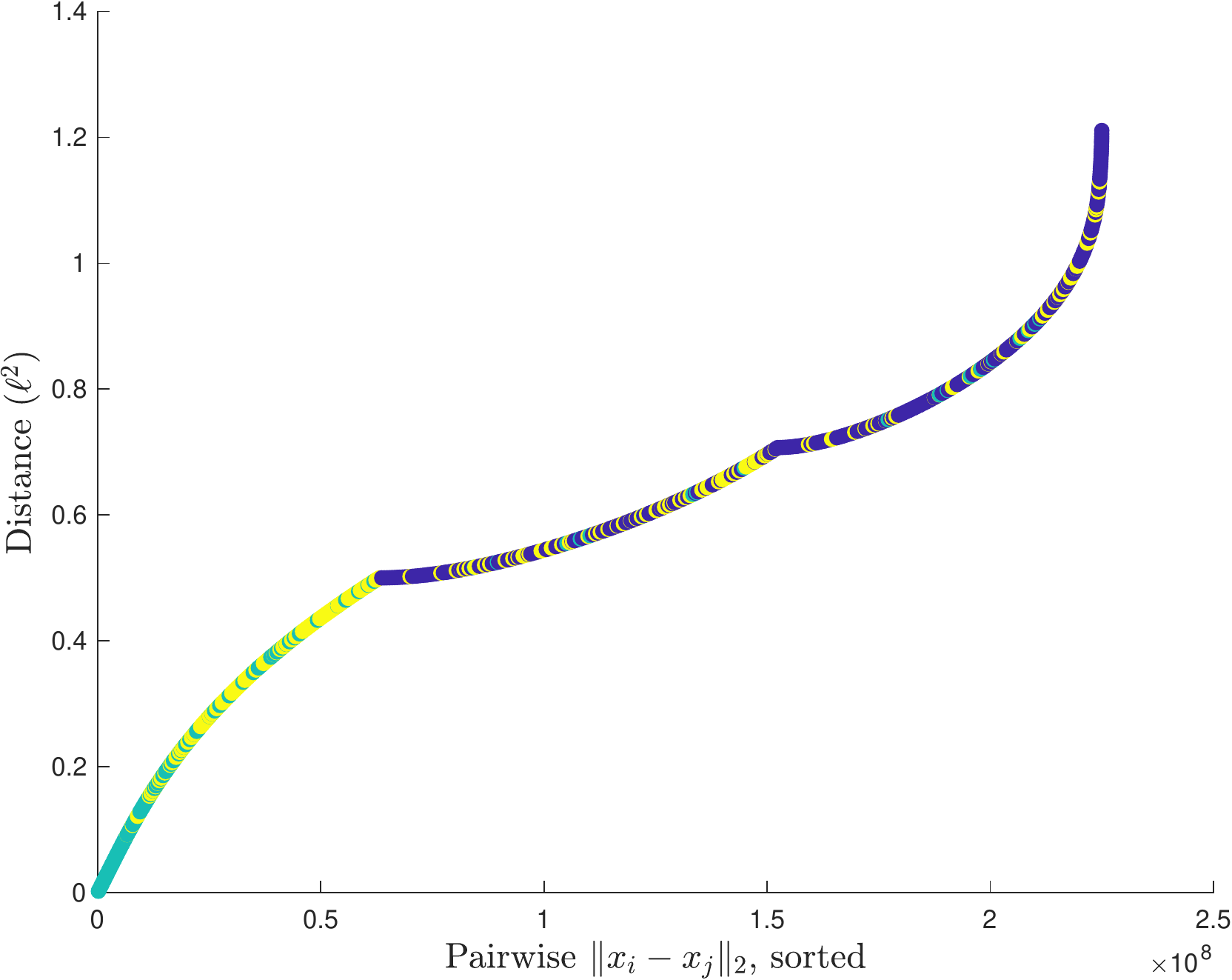}
\subcaption{Corresponding pairwise $\ell^{2}$ distances, sorted.}
\end{subfigure}
\caption{\label{fig:PhaseTransitionData} (a) The clusters are on edges of the rectangular prism so that the pairwise LLPDs between the clusters is at least 1.  The interior is filled with noise points.  Each cluster has 3000 points, as does the interior noise region.  (b) The sorted $\rho_{\ell\ell}$ plot shows within-cluster LLPDs in green, between-cluster LLPDs in blue, and LLPDs involving noise points in yellow.  There is a clear phase transition between the within-cluster and between-cluster LLPDs.  This empirical observation can be compared with the theoretical guarantees of Theorems \ref{thm:WithinLowDimCluster} and \ref{thm:BetweenClusters}.  Setting $t=.01$ in those theorems yield corresponding maximum within-cluster LLPD (shown with the solid red line) and minimum between-cluster distance (shown with the dashed red line).  The empirical results confirm our theoretical guarantees.  Notice moreover that there is no clear separation between the Euclidean distances, which are shown in (c).  This illustrates the challenges faced by classical spectral clustering, compared to LLPD spectral clustering, for this dataset.}

\end{figure}


\section{Performance Guarantees for Ultrametric and LLPD Spectral Clustering}\label{sec:SpectralClusteringAnalysis}

In this section we first derive performance guarantees for spectral clustering with any ultrametric. We show that when the data consists of cluster cores surrounded by noise, the weight matrix $W$ used in spectral clustering is, for a certain range of scales $\sigma$, approximately block diagonal with constant blocks.  In this range of $\sigma$, the number of clusters can be inferred from the maximal eigengap of $\Lsym$, and spectral clustering achieves high labeling accuracy.  On the other hand, for Euclidean spectral clustering it is hard to choose a scale parameter that is simultaneously large enough to guarantee a strong connection between every pair of points in the same cluster and small enough to produce weak connections between clusters (and even when possible the shape (e.g. elongation) of clusters affects the ability to identify the correct clusters). The resulting Euclidean weight matrix is not approximately block diagonal for any choice of $\sigma$, and the eigengap of $\Lsym$ becomes uninformative and the labeling accuracy potentially poor.  Moreover, using an ultrametric for spectral clustering leads to direct lower bounds on the degree of noise points, since if a noise point is close to any cluster point, it is close to all points in the given cluster.  It is well-known that spectral clustering is unreliable for points of low degree and in this case $\Lsym$ may have arbitrarily many small eigenvalues \citep{VonLuxburg2007}.

After proving results for general ultrametrics, we derive specific performance guarantees for LLPD spectral clustering on the LDLN data model.  We remove low density points by considering each point's LLPD-nearest neighbor distances, then derive bounds on the eigengap and labeling accuracy which hold even in the presence of noise points with weak connections to the clusters.  We prove there is a large range of values of both the thresholding and scale parameter for which we correctly recover the clusters, illustrating that LLPD spectral clustering is robust to the choice of parameters and presence of noise. In particular, when the clusters have a very low-dimensional structure and the noise is very high-dimensional, that is, when $d \ll D$, an enormous amount of noise points can be tolerated.  Throughout this section, we use the notation established in Subsection \ref{subsec:SpectralClustering}.

\subsection{Ultrametric Spectral Clustering}

Let $\rho:\mathbb{R}^{D}\times\mathbb{R}^{D}\rightarrow [0,\infty)$ be an ultrametric; see (\ref{defn:ultrametric}).  We analyze $\Lsym$ under the assumptions of the following cluster model. As will be seen in Subsection \ref{subsec:LLPDpectralClustering}, this cluster model holds for data drawn from the LDLN data model with the LLPD ultrametric, but it may be of interest in other regimes and for other ultrametrics. The model assumes there are $\numclust$ sets forming cluster cores and each cluster core has a halo of noise points surrounding it; for $1 \leq l \leq \numclust$, $A_l$ denotes the cluster core and $C_l$ the associated halo of noise points. For LDLN data, the parameter $\epsilonsep$ corresponds to the minimal between cluster distance after denoising.

\begin{assumption}[Ultrametric Cluster Model]
\label{assumption:ultrametric}
For $1 \leq l \leq \numclust$, assume $A_{l}$ and $C_{l}$ are disjoint finite sets, and let $\tilde{A_{l}}=A_{l}\cup C_{l}$.  Let $N = |\cup_{l} \tilde{A}_l|$. Assume  that for some $\epsilonincluster \leq \thetathres < \epsilonsep$:
\begin{align}
\rho(x^l_i,x^l_j) &\leq \epsilonincluster &\forall x^l_i, x^l_j \in A_l, 1 \leq l \leq \numclust, \label{equ:assump1}\\
\epsilonincluster < \rho(x^l_i,x^l_j) &\leq \thetathres &\forall x^l_i \in A_l, x^l_j \in C_l, 1 \leq l \leq \numclust, \label{equ:assump2}\\
\rho(x^l_i,x^s_j) &\geq \epsilonsep &\forall x^l_i \in \tilde{A}_l, x^s_j \in \tilde{A}_s, 1 \leq l\ne s \leq \numclust. \label{equ:assump3}
\end{align}Moreover, let $\zeta_{N} = \max_{1 \leq l \leq \numclust} \frac{N}{\strut |\tilde{A}_l|}.$
\end{assumption}

Theorem \ref{thm:SpectralGap} shows that under Assumption \ref{assumption:ultrametric}, the maximal eigengap of $\Lsym$ corresponds to the number of clusters $\numclust$ and spectral clustering with $\numclust$ principal eigenvectors achieves perfect labeling accuracy. The label accuracy result is obtained by showing the spectral embedding with $\numclust$ principal eigenvectors is a \textit{perfect representation} of the sets $\tilde{A}_l$, as defined in \cite{vu2018simple}. 

\begin{defn}[Perfect Representation] 
	\label{def:perfectrep}
	A clustering representation is perfect if there exists an $r>0$ such that
	\begin{itemize}
	\setlength\itemsep{0em}
		\item[$\cdot$] 	Vertices in the same cluster have distance at most $r$.
		\item[$\cdot$] Vertices from different clusters have distance at least $4r$ from each other.
	\end{itemize}
\end{defn} 

There are multiple clustering algorithms which are guaranteed to perfectly recover the labels of $\tilde{A}_l$ from a perfect representation, including $K$-means with furthest point initialization and single linkage clustering. Again following the terminology in \cite{vu2018simple}, we will refer to all such clustering algorithms as \textit{clustering by distances}. The proof of Theorem \ref{thm:SpectralGap} is in Appendix \ref{app:SpectralGapProof}.

\begin{thm}
\label{thm:SpectralGap}
Assume the ultrametric cluster model.  Then $\lambda_{\numclust+1}-\lambda_{\numclust}$ is the largest gap in the eigenvalues of $L_{\text{SYM}}(\cup_{l} \tilde{A}_l, \rho, f_{\sigma})$ provided 
\begin{align}
\label{equ:SpectralGap}
\frac{1}{2} \ &\geq \ \underset{\textbf{Cluster Coherence}}{\underbrace{5(1-f_{\sigma}(\epsilonincluster))}} \ \ +\ \underset{\textbf{Cluster Separation}}{\underbrace{6\zeta_{N} f_{\sigma}(\epsilonsep)}} \ +\ \ \underset{\textbf{Noise}}{\underbrace{4 (1-f_{\sigma}(\thetathres))}} \ +\ \beta\ , 
\end{align} \\
\vspace{.1cm}
where $\beta = O((1-f_{\sigma}(\epsilonincluster))^2+\zeta_{N}^2f_{\sigma}(\epsilonsep)^2+(1-f_{\sigma}(\thetathres))^2)$ denotes higher-order terms. Moreover, if 
\begin{align}
\label{equ:eigenvec_inequality_ultra} 
\frac{C}{\numclust^3\zeta_{N}^2} \geq (1-f_{\sigma}(\epsilonincluster))+\zeta_{N} f_{\sigma}(\epsilonsep)+(1-f_{\sigma}(\thetathres)) + \beta\ ,
\end{align}
where $C$ is an absolute constant, then clustering by distances on the $\numclust$ principal eigenvectors of $\Lsym(\cup_{l} \tilde{A}_l, \rho, f_{\sigma})$ perfectly recovers the cluster labels.
\end{thm}

For condition (\ref{equ:SpectralGap}) to hold, the following three terms must all be small:

\begin{itemize}
\item \textbf{Cluster Coherence}: This term is minimized by choosing $\sigma$ large so that $f_{\sigma}(\epsilonincluster)\approx 1$; the larger the scale parameter, the stronger the within-cluster connections.
\item  \textbf{Cluster Separation}: This term is minimized by choosing $\sigma$ small so that $f_{\sigma}(\epsilonsep)\approx 0$; the smaller the scale parameter, the weaker the between-cluster connections. Note $\zeta_{N}$ is minimized when clusters are balanced, in which case $\zeta_{N}=\numclust$. 
\item \textbf{Noise}: This term is minimized by choosing $\sigma$ large, so that once again $f_{\sigma}(\thetathres) \approx 1$. When the scale parameter is large, noise points around the cluster will be well connected to their designated cluster. 
\item \textbf{Higher Order Terms}: This term consists of terms that are quadratic in $(1-f_{\sigma}(\epsilonincluster)), f_{\sigma}(\epsilonsep), (1-f_{\sigma}(\thetathres))$, which are small in our regime of interest. 
\end{itemize}

Solving for the scale parameter $\sigma$ will yield a range of $\sigma$ values where the eigengap statistic is informative; this is done in Corollary \ref{cor:sigmarange} for LLPD spectral clustering on the LDLN data model.

Condition (\ref{equ:eigenvec_inequality_ultra}) guarantees that clustering the LLPD spectral embedding results in perfect label accuracy, and requires a stronger scaling with respect to $\numclust$ than Condition (\ref{equ:SpectralGap}), as when $\zeta_{N}=O(\numclust)$ there is an additional factor of $\numclust^{-5}$ on the left hand side of the inequality.  Determining whether this scaling in $\numclust$ is optimal is a topic of ongoing research, though the present article is more concerned with scaling in $n, d,$ and $D$.  Condition (\ref{equ:eigenvec_inequality_ultra}) in fact guarantees perfect accuracy for clustering by distances on the spectral embedding regardless of whether the $\numclust$ principal eigenvectors of $\Lsym$ are row-normalized or not. Row normalization is proposed in \cite{Ng2002} and generally results in better clustering results when some points have small degree \citep{VonLuxburg2007}; however it is not needed here because the properties of LLPD cause all points to have similar degree. 

\begin{rems}
	One can also derive a label accuracy result by applying Theorem 2 from \cite{Ng2002}, restated in \cite{Arias2011}, to show the spectral embedding satisfies the so-called \emph{orthogonal cone property (OCP)} \citep{Schiebinger2015}.  Indeed, let $\{\phi_{k}\}_{k=1}^{\numclust}$ be the principal eigenvectors of $\Lsym$.  The OCP guarantees that in the representation $x\mapsto \{\phi_{k}(x)\}_{k=1}^{\numclust}$, distinct clusters localize in nearly orthogonal directions, not too far from the origin.  Proposition 1 from \cite{Schiebinger2015} can then be applied to conclude $K$-means on the spectral embedding achieves high accuracy. Specifically, if (\ref{equ:SpectralGap}) is satisfied, then with probability at least $1-t$, $\numclust$-means on the $\numclust$ principal eigenvectors of $\Lsym(\cup_{l} \tilde{A}_l, \rho, f_{\sigma})$ achieves accuracy at least $1 - \frac{c\numclust^9\zeta_{N}^3 (f_{\sigma}(\epsilonsep)^2+\beta)}{t}$ where $c$ is an absolute constant and $\beta$
	denotes higher order terms. This approach results in a less restrictive scaling for $1-f_\sigma(\epsilonincluster), f_\sigma(\epsilonsep)$ in terms of $\numclust, \zeta_{N}$ than given in Condition (\ref{equ:eigenvec_inequality_ultra}), but does not guarantee perfect accuracy, and also requires row normalization of the spectral embedding as proposed in \cite{Ng2002}.  The argument using this approach to proving cluster accuracy is not discussed in this article, for reasons of space and as to not introduce additional excessive notation.  
\end{rems}


\subsection{LLPD Spectral Clustering with $k$NN LLPD Thresholding}
\label{subsec:LLPDpectralClustering}

We now return to the LDLN data model defined in Subsection \ref{subsec:assumptions} and show that it gives rise to the ultrametric cluster model described in Assumption \ref{assumption:ultrametric} when combined with the LLPD metric. Theorem \ref{thm:SpectralGap} can thus be applied to derive performance guarantees for LLPD spectral clustering on the LDLN data model. All of the notation and assumptions established in Subsection \ref{subsec:assumptions} hold throughout Subsection \ref{subsec:LLPDpectralClustering}.

\subsubsection{Thresholding}
\label{subsec:thresholding}
Before applying spectral clustering, we denoise the data by removing any points having sufficiently large LLPD to their $\kNoise^{\!\!\!\!\!\text{th}}$ LLPD-nearest neighbor. Motivated by the sharp phase transition illustrated in Subsection \ref{subsec:PhaseTransition}, we choose a threshold $\thetathres$ and discard a point $x\in X$ if $\beta_{\kNoise}(x, X) > \thetathres$.  Note that the definition of $\epsilonnoiseknn$ guarantees that we can never have a group of more than $\kNoise$ noise points where all pairwise LLPD are smaller than $\epsilonnoiseknn$, because if we did there would be a point $x \in \tilde{X}$ with  $\beta_{\kNoise}(x, \tilde{X})<\epsilonnoiseknn$. Thus if $\epsilonincluster \leq \thetathres < \epsilonnoiseknn$ then, after thresholding, the data will consist of the cluster cores $X_l$ with $\thetathres$-groups of at most $\kNoise$ noise points emanating from the cluster cores, where a $\thetathres$-group denotes a set of points where the LLPD between all pairs of points in the group is at most $\thetathres$. 

We assume LLPD is re-computed on the denoised data set $X_N$, whose cardinality we define to be $N$, and let $\rho^{X_N}_{\ell\ell}$ denote the corresponding LLPD metric. The points remaining after thresholding consist of the sets $\tilde{A}_l$, where 
\begin{equation}
\begin{array}{r@{}l}
A_l{}&= \{ x_i \in X_l \} \cup \{ x_i \in \tilde{X} \ |\ \rho_{\ell\ell}(x_i,x_j) \leq \epsilonincluster \text{ for some } x_j \in X_l \}, \\[.1cm]
C_l\ {}&= \{ x_i \in \tilde{X} \ |\  \epsilonincluster < \rho_{\ell\ell}(x_i, x_j)  \leq \thetathres \text{ for some } x_j \in X_l\}, \\[.1cm] \label{equ:SetDefs}
\tilde{A}_l\ {}&= A_l \cup C_l =\{ x_i \in X_l \} \cup \{ x_i \in \tilde{X} \ |\ \rho_{\ell\ell}(x_i,x_j) \leq \thetathres \text{ for some } x_j \in X_l \}.
\end{array}
\end{equation}

The cluster core $A_l$ consists of the points $X_l$ plus any noise points in $\tilde{X}$ that are indistinguishable from $X_l$, being within the maximal within-cluster LLPD of $X_l$. The set $C_l$ consists of the noise points in $\tilde{X}$ that are $\thetathres$-close to $X_l$ in LLPD. 


\subsubsection{Supporting Lemmata}
The following two lemmata are needed to prove Theorem \ref{thm:SpectralGapDataModel}, the main result of this subsection. The first one guarantees that the sets defined in (\ref{equ:SetDefs}) describe exactly the points which survive thresholding, that is $X_N = \cup_l \tilde{A}_l$.
\begin{lem}
\label{lem:thresholding}
Assume the LDLN data model and assumptions, and let $\tilde{A}_l$ be as in (\ref{equ:SetDefs}). If $\kNoise<\nmin$, $\epsilonincluster \leq \thetathres < \epsilonnoiseknn$, then $\beta_{\kNoise}(x, X) \leq \thetathres$ if and only if $x \in \tilde{A}_l$ for some $1 \leq l \leq \numclust$.
\end{lem}
\begin{proof}
Assume $\beta_{\kNoise}(x, X) \leq \thetathres$. If $x \in \cup_l X_l$, then clearly $x \in\cup_{l} \tilde{A}_l$, so assume $x \in \tilde{X}$. We claim there exists some $y \in \cup_lX_l $ such that $\rho_{\ell\ell}(x,y) \leq \thetathres$. Suppose not; then there exist $\kNoise$ points $\{x_i\}_{i=1}^{\kNoise}$ in $\tilde{X}$ distinct from $x$ with $\rho_{\ell\ell}(x, x_i) \leq \thetathres$; thus $\epsilonnoiseknn \leq \thetathres$, a contradiction. Hence, there exists $y \in X_l$ such that $\rho_{\ell\ell}(x,y) \leq \thetathres$ and $x \in \tilde{A}_l$. \\
Now assume $x \in \tilde{A}_l$ for some $1 \leq l \leq \numclust$. Then clearly there exists $y \in \cup_lX_l$ with $\rho_{\ell\ell}(x,y) \leq \thetathres$. Since $\epsilonincluster < \thetathres$, $x$ is within LLPD $\thetathres$ of all points in $X_l$, and since $\kNoise < \nmin, \beta_{\kNoise}(x, X) \leq \beta_{\nmin}(x,X) \leq \thetathres$.
\end{proof}

Next we show that when there is sufficient separation between the cluster cores, the LLPD between any two points in distinct clusters is bounded by $\delta/2$, and thus the assumptions of Theorem \ref{thm:SpectralGap} will be satisfied with $\epsilonsep = \delta/2$.

\begin{lem}
\label{lem:clustersep}
Assume the LDLN data model and assumptions, and assume $\epsilonincluster \leq \thetathres < \epsilonnoiseknn \wedge \delta/(4\kNoise)$, $A_l, C_l, \tilde{A}_l$ as defined in (\ref{equ:SetDefs}), and $\kNoise < \nmin.$  Then Assumption \ref{assumption:ultrametric} is satisfied with $\rho = \rho^{X_N}_{\ell\ell}$, $\epsilonsep = \delta/2$.
\end{lem}
\begin{proof}
First note that if $x \in A_l$, then $\rho_{\ell\ell}(x,y)\leq \epsilonincluster$ for all $y\in X_l$, and thus $x \notin C_l$, so $A_l$ and $C_l$ are disjoint.

Let $x^l_i, x^l_j \in A_l$. Then there exists $y_i, y_j \in X_l$ with $\rho_{\ell\ell}(x^l_i, y_i) \leq \epsilonincluster$ and $\rho_{\ell\ell}(x^l_j, y_j) \leq \epsilonincluster$, so $\rho_{\ell\ell}(x^l_i,x^l_j) \leq \rho_{\ell\ell}(x^l_i,y_i)\vee \rho_{\ell\ell}(y_i,y_j)\vee \rho_{\ell\ell}(y_j,x^l_j) \leq \epsilonincluster$. Since $x^l_i, x^l_j$ were arbitrary, $\rho_{\ell\ell}(x^l_i,x^l_j) \leq \epsilonincluster$ for all $x^l_i, x^l_j \in A_l$. We now show that in fact $\rho^{X_N}_{\ell\ell}(x^l_i,x^l_j) \leq \epsilonincluster$. Suppose not. Since $\rho_{\ell\ell}(x^l_i,x^l_j) \leq \epsilonincluster$, there exists a path in $X$ from $x^l_i$ to $x^l_j$ with all legs bounded by $\epsilonincluster$. Since $\rho^{X_N}_{\ell\ell}(x^l_i,x^l_j)>\epsilonincluster$, one of the points along this path must have been removed by thresholding, i.e. there exists $y$ on the path with $\beta_{\kNoise}(y,X)>\thetathres$. But then for all $x^l \in A_l$, $\rho_{\ell\ell}(y, x^l)\leq \rho_{\ell\ell}(y, x_i^l) \vee \rho_{\ell\ell}(x_i^l, x^l)\leq \epsilonincluster$, so $\beta_{\kNoise}(y,X) \leq \epsilonincluster$ since $\kNoise < \nmin$; contradiction.

Let $x^l_i \in A_l, x^l_j \in C_l$. Then there exist points $y_i, y_j \in X_l$ such that $\rho_{\ell\ell}(x^l_i, y_i) \leq \epsilonincluster$ and $\epsilonincluster < \rho_{\ell\ell}(x^l_j, y_j) \leq \thetathres$. Thus $\rho_{\ell\ell}(x^l_i,x^l_j) \leq \rho_{\ell\ell}(x^l_i,y_i)\vee \rho_{\ell\ell}(y_i,y_j)\vee \rho_{\ell\ell}(y_j,x^l_j)
\leq \epsilonincluster \vee \epsilonincluster \vee \thetathres=\thetathres.$

Now suppose $\rho_{\ell\ell}(x^l_i,x^l_j) \leq \epsilonincluster$. Then $\rho_{\ell\ell}(x^l_j, y_i) \leq \rho_{\ell\ell}(x^l_j, x^l_i) \vee \rho_{\ell\ell}(x^l_i, y_i)\leq \epsilonincluster$ so that $x^l_j \in A_l$ since $y_i \in X_l$; this is a contradiction since $x^l_j \in C_l$ and $A_l$ and $C_l$ are disjoint. We thus conclude $\epsilonincluster < \rho_{\ell\ell}(x^l_i,x^l_j) \leq \thetathres$. Since $x^l_i, x^l_j$ were arbitrary, $\epsilonincluster < \rho_{\ell\ell}(x^l_i,x^l_j) \leq \thetathres$ for all $x^l_i \in A_l, x^l_j \in C_l$. We now show in fact $\epsilonincluster < \rho^{X_N}_{\ell\ell}(x^l_i,x^l_j) \leq \thetathres$. Clearly, $\epsilonincluster < \rho_{\ell\ell}(x^l_i,x^l_j)\leq \rho^{X_N}_{\ell\ell}(x^l_i,x^l_j)$. Now suppose $\rho^{X_N}_{\ell\ell}(x^l_i,x^l_j)>\thetathres$. Since $\rho_{\ell\ell}(x^l_i,x^l_j)\leq\thetathres$, there exists a path in $X$ from $x^l_i$ to $x^l_j$ with all legs bounded by $\thetathres$. Since $\rho^{X_N}_{\ell\ell}(x^l_i,x^l_j)>\thetathres$, one of the points along this path must have been removed by thresholding, i.e. there exists $y$ on the path with $\beta_{\kNoise}(y,X)>\thetathres$. But then for all $x^l \in \tilde{A}_l$, $\rho_{\ell\ell}(y, x^l)\leq \rho_{\ell\ell}(y, x_i^l) \vee \rho_{\ell\ell}(x_i^l, x^l)\leq \thetathres$, so $\beta_{\kNoise}(y,X) \leq \thetathres$ since $\kNoise < \nmin$, which is a contradiction.

Finally, we show we can choose $\epsilonsep =  \delta/2$, that is,  $\rho^{X_N}_{\ell\ell}(x^l_i,x^s_j) \geq \delta/2$ for all $x^l_i \in \tilde{A}_l, x^s_j \in \tilde{A}_s, l\ne s$. We first verify that every point in $\tilde{A}_l$ is within Euclidean distance $\thetathres\kNoise$ of a point in $X_l$. Let $x \in \tilde{A}_l$ and assume $x \in \tilde{X}$ (otherwise there is nothing to show). Then there exists a point $y \in X_l$ with $\rho_{\ell\ell}(x,y)\leq \thetathres$, i.e. there exists a path of points from $x$ to $y$ with the length of all legs bounded by $\thetathres$. Note there can be at most $\kNoise$ consecutive noise points along this path, since otherwise we would have a $z\in \tilde{X}$ with $\beta_{\kNoise}(z,\tilde{X}) \leq \thetathres$ which contradicts $\epsilonnoiseknn>\thetathres$. Let $y^*$ be the last point in $X_l$ on this path. Since $\theta<\delta/(4\kNoise)<\delta/(2\kNoise+1)$, $\text{dist}({X}_l,{X}_s)\geq \delta >2\thetathres \kNoise+\thetathres$, and the path cannot contain any points in $X_s, l\ne s$; thus the path from $y^*$ to $x$ consists of at most $\kNoise$ points in $\tilde{X}$, so $\|x-y^{*}\|_2 \leq \kNoise\thetathres$.  Thus: $\min_{1\leq l\ne s\leq \numclust}\text{dist}(\tilde{A}_l,\tilde{A}_s)\geq \min_{1\leq l\ne s\leq \numclust}\text{dist}(X_l,X_s) - 2\thetathres\kNoise \geq \delta -2\thetathres\kNoise > \delta/2$ since $\theta<\delta/(4\kNoise)$.  Now by Lemma \ref{lem:thresholding}, there are no points outside of $\cup_l \tilde{A}_l$ which survive thresholding, so we conclude $\rho^{X_N}_{\ell\ell}(x^l_i,x^s_j) \geq \delta/2$ for all $x^l_i \in \tilde{A}_l, x^s_j \in \tilde{A}_s$.
\end{proof}


\subsubsection{Main Result}

We now state our main result for LLPD spectral clustering with $k$NN LLPD thresholding.

\begin{thm}
\label{thm:SpectralGapDataModel}
Assume the LDLN data model and assumptions.
For a chosen $\thetathres$ and $\kNoise$, perform thresholding at level $\thetathres$ as above to obtain $X_N$,
and assume $\kNoise < \nmin$,  $\epsilonincluster \leq \thetathres < \epsilonnoiseknn \wedge \delta/(4\kNoise)$.
Then $\lambda_{\numclust+1}-\lambda_\numclust$ is the largest gap in the eigenvalues of $L_{\text{SYM}}(X_N,\rho^{X_N}_{\ell\ell},f_{\sigma})$ provided that
\begin{align} 
\label{equ:spectral_gap_inequal_PD}
\frac{1}{2} &\geq 5(1-f_{\sigma}(\epsilonincluster)) + 6\zeta_{N} f_{\sigma}(\delta/2) + 4 (1-f_{\sigma}(\thetathres)) +\beta\ ,
\end{align}
and clustering by distances on the $\numclust$ principal eigenvectors of $L_{\text{SYM}}(X_N,\rho^{X_N}_{\ell\ell},f_{\sigma})$ perfectly recovers the cluster labels provided that
\begin{align*}
\frac{C}{\numclust^3\zeta_{N}^2} \geq (1-f_{\sigma}(\epsilonincluster))+\zeta_{N} f_{\sigma}( \delta/2)+(1-f_{\sigma}(\thetathres)) + \beta\ ,
\end{align*}
where $\beta = O\left((1-f_{\sigma}(\epsilonincluster)^2+\zeta_{N}^2f_{\sigma}( \delta/2)^2+(1-f_{\sigma}(\thetathres))^2\right)$ denotes higher-order terms and $C$ is an absolute constant. 
In addition, for $n_{\min}$ large enough with probability at least $1-O(n_{\min}^{-1})$, $\zeta_{N} \leq 2\zeta_n+3\kNoise\zeta_\thetathres$ for the LDLN data model balance parameters $\zeta_n, \zeta_\thetathres$.

\end{thm}

\begin{proof}
Define the sets $A_l, C_l, \tilde{A}_l$ as in (\ref{equ:SetDefs}). By Lemma \ref{lem:thresholding}, removing all points satisfying $\beta_{\kNoise}(x,X)>\thetathres$ leaves us with exactly $X_N=\cup_l \tilde{A}_l$. By Lemma \ref{lem:clustersep}, all assumptions of Theorem \ref{thm:SpectralGap} are satisfied for ultrametric $\rho^{X_N}_{\ell\ell}$ and we can apply Theorem \ref{thm:SpectralGap} with $\epsilonincluster, \epsilonnoiseknn$ as defined in Subsection \ref{subsec:assumptions} and $\epsilonsep = \delta/2$. 
All that remains is to verify the bound on $\zeta_{N}$. 

Recall $\tilde{A}_l= X_l \cup \{ x_i \in \tilde{X} \ |\ \rho_{\ell\ell}(x_i,x_j) \leq \thetathres \text{ for some } x_j \in X_l \}$; let $m_l$ denote the cardinality of $\{ x_i \in \tilde{X} \ |\ \rho_{\ell\ell}(x_i,x_j) \leq \thetathres \text{ for some } x_j \in X_l \}$ so that
$\zeta_{N} = \max_{1\leq l\leq\numclust} \frac{\sum_{i=1}^\numclust n_i+m_i}{n_l+m_l}.$
For $1 \leq l \leq \numclust$, let $\omega_l = \sum_{x \in \tilde{X}} \mathbbm{1}_{x \in B(\X_l,\thetathres)\setminus \X_l }$ 
denote the number of noise points that fall within a tube of width $\thetathres$ around the cluster region $\X_l$. Note that $\omega_l\distras{}\Bin(\tilde{n},p_{l,\thetathres})$ where $p_{l,\thetathres} = \volD(B(\X_l,\thetathres)\setminus\X_l)/\volD(\tilde{\X})$ is as defined in Section \ref{subsec:assumptions}. The assumptions of Theorem \ref{thm:SpectralGapDataModel} guarantee that $m_l \leq \kNoise\omega_l$, since $\omega_l$ is the number of groups attaching to $\X_l$, and each group consists of at most $\kNoise$ noise points. To obtain a lower bound for $m_l$, note that $m_l \geq \sum_{x \in \tilde{X}} \mathbbm{1}_{x \in B(X_l,\thetathres)\setminus \X_l }$, where  $B(X_l,\thetathres) \subset B(\X_l,\thetathres)$ is formed from the discrete sample points $X_l$. Since $B(X_l,\thetathres) \rightarrow B(\X_l,\thetathres)$ as $n_l \rightarrow \infty$, for $n_{\min}$ large enough $\volD(B(X_l,\thetathres)\setminus \X_l) \geq \frac{1}{2}\volD(B(\X_l,\thetathres)\setminus \X_l )$, and $m_l \geq \omega_{l,2}$ where $\omega_{l,2}\distras{}\Bin(\tilde{n},p_{l,\thetathres}/2)$. \\
We first consider the high noise case $\tilde{n}p_{\min,\thetathres} \geq n_{\min}$, and define $\zeta_l = \frac{\sum_{i=1}^\numclust n_i+m_i}{n_l+m_l}$.  We have
	\begin{align*}
	\zeta_l &\leq \frac{\sum_{i=1}^\numclust n_i}{n_l} + \frac{\sum_{i=1}^\numclust m_i}{m_l} 
	\leq \frac{\sum_{i=1}^\numclust n_i}{n_l}+ \kNoise\frac{\sum_{i=1}^\numclust \omega_i}{\omega_{l,2}}.
	\end{align*}
	A multiplicative Chernoff bound \citep{hagerup1990guided} gives $\Prob(\omega_i\geq(1+\delta_1)\tilde{n}p_{i,\thetathres})\leq \exp(-\delta_1^2\tilde{n}p_{i,\thetathres}/3)\leq \exp(-\delta_1^2n_{\min}/3)$ for any $0\leq \delta_1\leq 1$. Choosing $\delta_1 = \sqrt{3\log(\numclust n_{\min})/n_{\min}}$ and taking a union bound gives $\omega_i \leq (1+\delta_1)\tilde{n}p_{i,\thetathres}$ for all $1\leq i\leq \numclust$ with probability at least $1-n_{\min}^{-1}$. A lower Chernoff bound also gives $\Prob(\omega_{i,2}\leq(1-\delta_2)\tilde{n}p_{i,\thetathres}/2) \leq \exp(-\delta_2^2\tilde{n}p_{i,\thetathres}/4) \leq \exp(-\delta_2^2n_{\min}/4)$ for any $0\leq \delta_2\leq 1$ and choosing $\delta_2 = \sqrt{4\log(\numclust n_{\min})/n_{\min}}$ gives $\omega_{i,2} \geq (1-\delta_2)\tilde{n}p_{i,\thetathres}/2$ for all $1\leq i\leq \numclust$ with probability at least $1-n_{\min}^{-1}$. Thus with probability at least $1-O(n_{\min}^{-1})$, one has
	\begin{align*}
	\frac{\sum_{i=1}^\numclust \omega_i}{\omega_{2,l}} \leq \frac{\sum_{i=1}^\numclust 2(1+\delta_1)\tilde{n}p_{i,\thetathres}}{(1-\delta_2)\tilde{n}p_{l,\thetathres}} \leq 3\frac{\sum_{i=1}^\numclust p_{i,\thetathres}}{p_{l,\thetathres}}
	\end{align*}
	for all $1\leq l \leq \numclust$ for $n_{\min}$ large enough, giving $\zeta_{N} =\max_{1\leq l\leq \numclust} \zeta_l \leq \zeta_n + 3\zeta_\thetathres.$
	
	We next consider the small noise case $\tilde{n}p_{\min,\thetathres} \leq n_{\min}$. A Chernoff bound gives $\Prob(\omega_i \geq (1+(\delta_i\vee\sqrt{\delta_i})\tilde{n}p_{i,\thetathres}) \leq \exp(-\delta_i^2\tilde{n}p_{i,\thetathres}/3)$ for any $\delta_i \geq 0$. We choose $\delta_i=3\log(\numclust n_{\min})/(\tilde{n}p_{i,\thetathres})$, so that with probability at least $1-n_{\min}^{-1}$ we have
	\begin{align*}
	\omega_i &\leq (1+(\delta_i\vee\sqrt{\delta_i})\tilde{n}p_{i,\thetathres})
	\leq 2\tilde{n}p_{i,\thetathres} + 6\log(\numclust n_{\min})
	\end{align*}
	for all $1\leq i\leq \numclust$ and we obtain for $n_{\min}$ large enough
	\begin{align*}
	\zeta_{N} 
	&\leq \frac{\sum_{i=1}^\numclust n_i+\kNoise \omega_i}{n_{\min}} \\
	&\leq \frac{\sum_{i=1}^\numclust n_i+\kNoise (2\tilde{n}p_{i,\thetathres} + 6\log(\numclust n_{\min}))}{n_{\min}} \\
	&\leq \frac{\sum_{i=1}^\numclust 2n_i+2\kNoise\tilde{n}p_{i,\thetathres}}{n_{\min}} \\
	&\leq \frac{\sum_{i=1}^\numclust 2n_i+2\kNoise n_{\min} p_{i,\thetathres}/p_{\min,\thetathres}}{n_{\min}} \\
	&= 2\zeta_n + 2\kNoise\zeta_\thetathres.
	\end{align*}
	Combining the two cases $\zeta_{N}  \leq 2\zeta_n + 3\kNoise\zeta_\thetathres$ that with probability at least $1-O(n_{\min}^{-1})$.
\end{proof}

Theorem \ref{equ:spectral_gap_inequal_PD} illustrates that after thresholding, the number of clusters can be reliably estimated by the maximal eigengap for the range of $\sigma$ values where \eqref{equ:spectral_gap_inequal_PD} holds. The following corollary combines what we know about the behavior of $\epsilonincluster$ and $\epsilonnoiseknn$ for the LDLN data model (as analyzed in Section \ref{sec:FiniteSampleAnalysis}) with the derived performance guarantees for spectral clustering to give the range of $\thetathres, \sigma$ values where $\lambda_{\numclust+1} - \lambda_{\numclust}$ is the largest gap with high probability. We remind the reader that although the LDLN data model assumes uniform sampling, Theorem \ref{equ:spectral_gap_inequal_PD} and Corollary \ref{cor:sigmarange} can easily be extended to a more general sampling model.

\begin{cor}
\label{cor:sigmarange}
Assume the notation of Theorem \ref{thm:SpectralGapDataModel} holds. Then for $\nmin$ large enough, for any $\tau < \frac{C_1}{8} \nmin^{-(d+1)} \wedge \frac{\epsilon_0}{5}$ and any
\begin{align}
\label{equ:thetarange}
 C_1\nmin^{-\frac{1}{d+1}} &\leq \thetathres \leq \left[C_2\tilde{n}^{-\left(\frac{\kNoise+1}{\kNoise}\right)\frac{1}{D}}\right] \wedge \delta (4\kNoise)^{-1},
\end{align}
we have that 
$\lambda_{\numclust+1}-\lambda_\numclust$ is the largest gap in the eigenvalues of $L_{\text{SYM}}$ with high probability, provided that
\begin{align}
\label{equ:sigmarange}
 C_3\thetathres &\leq \sigma \leq \frac{C_4\delta}{f_1^{-1}(C_5(\zeta_n+\kNoise\zeta_\thetathres)^{-1})}
\end{align}
where all $C_i$ are constants independent of $n_1, \ldots, n_{\numclust}, \tilde{n}, \thetathres, \sigma$.
\end{cor}

\begin{proof}

By Corollary \ref{cor:WithinCluster_tau_small}, for $\nmin$ large enough, $\epsilonincluster$ satisfies $ \nmin \lesssim \epsilonincluster^{-d}\log(\epsilonincluster^{-d}) \leq \epsilonincluster^{-(d+1)},$ i.e. $\epsilonincluster \leq C_1 \nmin^{-\frac{1}{d+1}}$ with high probability, as long as $\tau < \frac{C_1}{8} \nmin^{-\frac{1}{d+1}} \wedge \frac{\epsilon_0}{5}$. 
%
%
%
%
By Theorem \ref{thm:knnLLPD}, with high probability $\epsilonnoiseknn \geq C_2 \tilde{n}^{-\frac{\kNoise+1}{\kNoise D}}$.
We now apply Theorem \ref{thm:SpectralGapDataModel}. Note that for an appropriate choice of constants, the assumptions of Corollary \ref{cor:sigmarange} guarantee $\epsilonincluster \leq \thetathres < \epsilonnoiseknn \wedge \delta/(4 \kNoise)$ with high probability. 
There exist constants $C_6, C_7, C_8$ independent of $n_1, \ldots, n_{\numclust}, \tilde{n}, \thetathres, \sigma$, such that inequality (\ref{equ:spectral_gap_inequal_PD}) is guaranteed as long as $f_{\sigma}(\epsilonincluster) \geq C_6$, $\zeta_{N} f_{\sigma}(\delta/2) \leq C_7$, and $f_{\sigma}(\thetathres) \geq C_8$. Solving for $\sigma$, we obtain $\left(\frac{\epsilonincluster}{f_1^{-1}(C_6)}\vee \frac{\thetathres}{f_1^{-1}(C_8)}\right) \leq \sigma \leq \frac{\delta}{2f_1^{-1}(C_7\zeta_{N}^{-1})}.$
Combining with our bound for $\epsilonincluster$ and recalling $\zeta_{N} \leq 2\zeta_n+3\kNoise\zeta_\thetathres$ with high probability by Theorem \ref{thm:SpectralGapDataModel}, this is implied by (\ref{equ:thetarange}) and (\ref{equ:sigmarange}) and for an appropriate choice of relabelled constants.

\end{proof}

This corollary illustrates that when $\tilde{n}$ is small relative to $\nmin^{ \frac{D}{d+1}\left(\frac{\kNoise}{\kNoise+1}\right)}$, we obtain a large range of values of both the thresholding parameter $\thetathres$ and scale parameter $\sigma$ where the maximal eigengap heuristic correctly identifies the number of clusters, i.e. LLPD spectral clustering is robust with respect to both of these parameters.


\subsubsection{Parameter Selection}\label{subsubsec:ParameterSelection}

In terms of implementation, only the parameters $\kNoise$ and $\thetathres$ must be chosen, and then $\Lsym$ can be computed for a range of $\sigma$ values.  Ideally $\kNoise$ is chosen to maximize the upper bound in (\ref{equ:thetarange}), since $\tilde{n}^{-\left(\frac{\kNoise+1}{\kNoise}\right)\frac{1}{D}}$ is increasing in $\kNoise$ while $\delta/\kNoise$ is decreasing in $\kNoise$.  Numerical experiments indicate robustness with respect to this parameter choice, and $\kNoise=20$ was used for all experiments reported in Section \ref{sec:NumericalExperiments}. 

Regarding the thresholding parameter $\thetathres$, ideally $\thetathres =\epsilonincluster$, since this guarantees that all cluster points will be kept and the maximal number of noise points will be removed, i.e. we have perfectly denoised the data.  However, $\epsilonincluster$ is not known explicitly and must be estimated from the data.  In practice the thresholding can be done by computing $\beta_{\kNoise}(x,X)$ for all data points and clustering the data points into groups based on these values, or by choosing $\thetathres$ to correspond to the elbow in a graph of the sorted nearest neighbor distances as illustrated in Section \ref{sec:NumericalExperiments}. This latter approach for estimating $\thetathres$ is very similar to the proposal in \citet{Ester1996} for estimating the scale parameter in DBSCAN, although we use LLPD instead of Euclidean nearest neighbor distances. Note the thresholding procedure precedes the application of the spectral clustering kernel; it can be done once and then $\Lsym$ computed for various $\sigma$ values. 


\subsection{Comparison with Related Methods}  

We now theoretically compare the proposed method with related methods.  LLPD spectral clustering combines spectral graph methods with a notion of distance that incorporates density, so we naturally focus on comparisons to spectral clustering and density-based methods.  

\subsubsection{Comparison with Theoretical Guarantees for Spectral Clustering}

Our results on the eigengap and misclassification rate for LLPD spectral clustering are naturally comparable to existing results for Euclidean spectral clustering. \citet{Arias2011, AriasChen2011, Arias2017} made a series of contributions to the theory of spectral clustering performance guarantees.  We focus on the results in \citet{Arias2011}, where the author proves performance guarantees on spectral clustering by considering the same data model as the one proposed in the present article, and proceeds by analyzing the corresponding Euclidean weight matrix. 

Our primary result, Theorem \ref{thm:SpectralGapDataModel}, is most comparable to Proposition 4 in \citet{Arias2011}, which estimates $\lambda_{\numclust}\le Cn^{-3}, \lambda_{\numclust+1}\ge Cn^{-2}$ for some constant $C$.  From the theoretical point of view, this does not necessarily mean $\lambda_{\numclust+1}-\lambda_{\numclust}\ge \lambda_{l+1}-\lambda_{l}, l\neq \numclust$.  Compared to that result, Theorem \ref{thm:SpectralGapDataModel} enjoys a much stronger conclusion for guaranteeing the significance of the eigengap.  From a practical point of view, it is noted in \citet{Arias2011} that Proposition 4 is not a useful condition for actual data.  Our method is shown to correctly estimate the eigengap in both high-dimensional and noisy settings, where the eigengap with Euclidean distance is uninformative; see Section \ref{sec:NumericalExperiments}.

Theorem \ref{thm:SpectralGapDataModel} also provides conditions guaranteeing LLPD spectral clustering achieves perfect labeling accuracy.  The proposed conditions are sufficient to guarantee the representation of the data in the coordinates of the principal eigenvectors of the LLPD Laplacian is a perfect representation.  An alternative approach to ensuring spectral clustering accuracy is presented in \cite{Schiebinger2015}, which develops the notion of the \emph{orthogonal cone property (OCP)}.  The OCP characterizes low-dimensional embeddings that represent points in distinct clusters in nearly orthogonal directions.  Such embeddings are then easily clustered with, for example, $\numclust$-means.  The two crucial parameters in the approach of \cite{Schiebinger2015} measure how well-separated each cluster is from the others, and how internally well-connected each distinct cluster is.  The results of Section \ref{sec:FiniteSampleAnalysis} prove that under the LDLN data model, points in the same cluster are very close together in LLPD, while points in distinct clusters are far apart in LLPD.  In this sense, the results of  \cite{Schiebinger2015}  suggest that LLPD spectral clustering ought to perform well in the LDLN regime.  Indeed, the LLPD is nearly invariant to cluster geometry, unlike Euclidean distance.  As clusters become more anisotropic, the within-cluster distances stay almost the same when using LLPD, but increase when using Euclidean distances.  In particular, the framework of \cite{Schiebinger2015} implies that performance of LLPD spectral clustering will degrade slowly as clusters are stretched, while performance of Euclidean spectral clustering will degrade rapidly.  We remark that the OCP framework has been generalized to continuum setting for the analysis of mixture models \citep{Trillos2019_Geometric}.  

This observation may also be formulated in terms of the spectrum of the Laplacian.  For the (continuous) Laplacian $\Delta$ on a domain $\mathcal{M}\subset\mathbb{\mathbb{R}^{D}}$, \citet{Szego1954_Inequalities, Weinberger1956_Isoperimetric} prove that among unit-volume domains, the second Neumann eigenvalue $\lambda_{2}(\Delta)$ is minimal when the underlying $\mathcal{M}$ is the ball.  One can show that as the ball becomes more elliptical in an area-preserving way, the second eigenvalue of the Laplacian decreases.  Passing to the discrete setting \citep{Trillos2018_Error}, this implies that as clusters become more elongated and less compact, the second eigenvalues on the individual clusters (ignoring between-cluster interactions, as proposed in \cite{Maggioni2018_Learning}) decreases.  Spectral clustering performance results are highly dependent on these second eigenvalues of the Laplacian when localized on individual clusters \citep{Arias2011, Schiebinger2015}, and in particular performance guarantees weaken dramatically as they become closer to 0.  In this sense, Euclidean spectral clustering is not robust to elongating clusters.  LLPD spectral clustering, however, uses a distance that is nearly invariant to this kind of geometric distortion, so that the second eigenvalues of the LLPD Laplacian localized on distinct clusters stay far from 0 even in the case of highly elongated clusters.  In this sense, LLPD spectral clustering is more robust than Euclidean spectral clustering for elongated clusters.

The same phenomenon is observed from the perspective of graph-cuts.  It is well-known \citep{Shi2000} that spectral clustering on the graph with weight matrix $W$ approximates the minimization of the multiway normalized cut functional $$\Ncut(C_{1},C_{2},\dots,C_{\numclust})=\argmin_{(C_{1},C_{2},\dots,C_{\numclust})} \sum_{k=1}^{\numclust}\frac{W(C_{k},X\setminus C_{k})}{\vol(C_{k})},$$ where $$W(C_{k},X\setminus C_{k})=\sum_{x_{i}\in C_{k}}\sum_{x_{j}\notin C_{k}}W_{ij}, \ \vol(C_{k})=\sum_{x_{i}\in C_{k}}\sum_{x_{j}\in X}W_{ij}.$$  As clusters become more elongated, cluster-splitting cuts measured in Euclidean distance become cheaper and the optimal graph cut shifts from one that separates the clusters to one that splits them.  On the other hand, when using the LLPD a cluster-splitting cut only becomes marginally cheaper as the cluster stretches, so that the optimal graph cut preserves the clusters rather than splits them.  

A somewhat different approach to analyzing the performance of spectral clustering is developed in \cite{Balakrishnan2011_Noise}, which proposes as a model for spectral clustering \emph{noisy hierarchical block matrices (noisy HBM)} of the form $W=A+R$ for ideal $A$ and a noisy perturbation $R$.  The ideal $A$ is characterized by on and off-diagonal block values that are constrained to fall in certain ranges, which models concentration of within-cluster and between-cluster distances.  The noisy perturbation $R$ is a random, mean 0 matrix having rows with independent, subgaussian entries, characterized by a variance parameter $\sigma_{\text{noise}}$.  The authors propose a modified spectral clustering algorithm (using the $K$-centering algorithm) which, under certain assumptions on the idealized within-cluster and between-cluster distances, learns all clusters above a certain size for a range of $\sigma_{\text{noise}}$ levels.  The proposed theoretical analysis of LLPD in Section \ref{sec:FiniteSampleAnalysis} shows that under the LDLN data model and for $n$ sufficiently large, the Laplacian matrix (and weight matrix) is nearly block constant with large separation between clusters.  Our Theorems \ref{thm:WithinCluster}, \ref{thm:BetweenClusters}, \ref{thm:knnLLPD}, \ref{thm:CombinedResult} may be interpreted as showing that the (LLPD-denoised) weight matrix associated to data generated from the LDLN model may fit the idealized model suggested by \cite{Balakrishnan2011_Noise}.  In particular, when $R=0$, the results for this noisy HBM are comparable with, for example, Theorem \ref{thm:SpectralGapDataModel}.  However, the proposed method does not consider hierarchical clustering, but instead shows localization properties of the eigenvectors of $\Lsym$.  In particular, the proposed method is shown to correctly learn the number of clusters $\numclust$ through the eigengap, assuming the LDLN model, which is not considered in \cite{Balakrishnan2011_Noise}.

\subsubsection{Comparison with Density-Based Methods}

The \emph{DBSCAN} algorithm labels points as cluster or noise based on density to nearby points, then creates clusters as maximally connected regions of cluster points.  While popular, DBSCAN is extremely sensitive to the selection of parameters to distinguish between cluster and noise points, and often performs poorly in practice.  The DBSCAN parameter for noise detection is comparable to the denoising parameter $\thetathres$ used in LLPD spectral clustering, though LLPD spectral clustering is quite robust to $\thetathres$ in theory and practice.  Moreover, DBSCAN does not enjoy the robust theoretical guarantees provided in this article for LLPD spectral clustering on the LDLN data model, although some results are known for techniques related to DBSCAN \citep{Rinaldo2010, Sriperumbudur2012}.

In order to address the shortcomings of DBSCAN, the \emph{fast search and find of density peaks clustering (FSFDPC)} algorithm was proposed \citep{Rodriguez2014_Clustering}.  This method first learns modes in the data as points of high density far (in Euclidean distance) from other points of high density, then associates each point to a nearby mode in an iterative fashion.  While more robust than DBSCAN, FSFDPC cannot learn clusters that are highly nonlinear or elongated.  \cite{Maggioni2018_Learning} proposed a modification to FSFDPC called \emph{learning by unsupervised nonlinear diffusion (LUND)} which uses diffusion distances instead of Euclidean distances to learn the modes and make label assignments, allowing for the clustering of a wide range of data, both theoretically and empirically.  While LUND enjoys theoretical guarantees and strong empirical performance \citep{Murphy2018_Diffusion, Murphy2019_Unsupervised}, it does not perform as robustly on the proposed LDLN model for estimation of $K$ or for labeling accuracy.  In particular, the eigenvalues of the diffusion process which underlies diffusion distances (and thus LUND) do not exhibit the same sharp cutoff phenomenon as those of the LLPD Laplacian under the LDLN data model.  

\emph{Cluster trees} are a related density-based method that produces a multiscale hierarchy of clusterings, in a manner related to single linkage clustering.  Indeed, for data sampled from some density $\mu$ on a Euclidean domain $X$, a cluster tree is the family of clusterings $\mathcal{T}=\{C_{r}\}_{r=0}^{\infty}$, where $C_{r}$ are the connected regions of the set $\{x \ | \ \mu(x)\ge r\}$.  \cite{Chaudhuri2010_Rates} studied cluster trees where $X\subset\mathbb{R}^{D}$ is a subset of Euclidean space, showing that if sufficiently many samples are drawn from $\mu$, depending on $D$, then the clusters in the empirical cluster tree closely match the population level clusters given by thresholding $\mu$.  \cite{Balakrishnan2013_Cluster} generalized this work to the case when the underlying distribution is supported on an intrinsically $d$-dimensional set, showing that the performance guarantees depend only on $d$, not $D$.  

The cluster tree itself is related to the LLPD as $\rho_{\ell\ell}(x_{i},x_{j})$ is equal to the smallest $r$ such that $x_{i},x_{j}$ are in the same connected component of a complete Euclidean distance graph with edges of length $\ge r$ removed.  Furthermore, the model of \cite{Balakrishnan2013_Cluster} assumes the support of the density is near a low-dimensional manifold, which is comparable to the LDLN model of assuming the clusters are $\tau$-close to low-dimensional elements of $\mathcal{S}_{d}(\kappa,\epsilon_{0})$.  The notion of separation in \cite{Chaudhuri2010_Rates, Balakrishnan2013_Cluster} is also comparable to the notion of between-cluster separation in the present manuscript.  On the other hand, the proposed method considers a more narrow data model (LDLN versus arbitrary density $\mu$), and proves strong results for LLPD spectral clustering including inference of $\numclust$, labeling accuracy, and robustness to noise and choice of parameters.  The proof techniques are also rather different for the two methods, as the LDLN data model provides simplifying assumptions which do not hold for a general probability density function.  Indeed, the approach presented in this article achieves precise finite sample estimates for the LLPD using percolation theory and the chain arguments of Theorems \ref{thm:BetweenClusters} and \ref{thm:knnLLPD}, in contrast to the general consistency results on cluster trees derived from a wide class of probability distributions \citep{Chaudhuri2010_Rates, Balakrishnan2013_Cluster}.


\section{Numerical Implementations of LLPD}\label{sec:Algorithm}

We first demonstrate how LLPD can be accurately approximated from a sequence of $m$ multiscale graphs in Section \ref{sec:LLPD_scales}. Section \ref{sec:LLPD_KNN} discusses how this approach can be used for fast LLPD-nearest neighbor queries. When the data has low intrinsic dimension $d$ and $m=O(1)$, the LLPD nearest neighbors of all points can be computed in $O(DC^{d}n\log(n))$ for a constant $C$ independent of $n,d,D$.  Although there are theoretical methods for obtaining the exact LLPD \citep{Demaine2009, Demaine2014} with the same computational complexity, they are not practical for real data.  The method proposed here is an efficient alternative, whose accuracy can be controlled by choice of $m$. 

The proposed LLPD approximation procedure can also be leveraged to define a fast eigensolver for LLPD spectral clustering, which is discussed in Section \ref{subsec:fast_eigensolver}.  The ultrametric structure of the weight matrix allows for fast matrix-vector multiplication, so that $\Lsym x$ (and thus the eigenvectors of $\Lsym$) can be computed with complexity $O(mn)$ for a dense $\Lsym$ defined using LLPD.  When the number of scales $m\ll n$, this is a vast improvement over the typical $O(n^{2})$ needed for a dense Laplacian.  Once again when the data has low intrinsic dimension and $\numclust, m=O(1)$, LLPD spectral clustering can be implemented in $O(DC^{d}n\log(n))$. 

Connections with single linkage clustering are discussed in Section \ref{subsec:ApproximateSLC}, as the LLPD approximation procedure gives a pruned single linkage dendrogram. Matlab code implementing both the fast LLPD nearest neighbor searches and LLPD spectral clustering is publicly available at \url{https://bitbucket.org/annavlittle/llpd_code/branch/v2.1}. The software auto-selects both the number of clusters $\numclust$ and kernel scale $\sigma$. 

\subsection{Approximate LLPD from a Sequence of Thresholded Graphs}
\label{sec:LLPD_scales}

The notion of \emph{nearest neighbor graph} is important for the formal analysis which follows.

\begin{defn}\label{def:kNNGraph} 
Let $(X,\rho$) be a metric space.  The \emph{(symmetric) $k$-nearest neighbors graph on $X$ with respect to $\rho$} is the graph with nodes $X$ and an edge between $x_{i}, x_{j}$ of weight $\rho(x_{i},x_{j})$ if $x_{j}$ is among the $k$ points with smallest $\rho$-distance to $x_{i}$ or if $x_{i}$ is among the $k$ points with smallest $\rho$-distance to $x_{j}$.
\end{defn}

Let $X=\{x_{i}\}_{i=1}^{n}\subset\mathbb{R}^{D}$ and $G$ be some graph defined on $X$.  Let $D^{\ell \ell}_G$ denote the matrix of exact LLPDs {obtained from all paths in the graph $G$; note this is a generalization of Definition \ref{def:LLPD}, which considers $G$ to be a complete graph.  We define an approximation $\hat{D}^{\ell \ell}_G$ of $D^{\ell \ell}_G$ based on a sequence of thresholded graphs.   Let E-nearest neighbor denote a nearest neighbor in the Euclidean metric, and LLPD-nearest neighbor denote a nearest neighbor in the LLPD metric.  
	
\begin{defn}
	Let $X$ be given and let $\kEuc$ be a positive integer.  Let $G(\infty)$ denote the complete graph on $X$, with edge weights defined by Euclidean distance, and $G_{\kEuc}(\infty)$ the $\kEuc$ E-nearest neighbors graph on $X$ as in Definition \ref{def:kNNGraph}.  For a threshold $t > 0$, let $G(t), G_{\kEuc}(t)$ be the graphs obtained from $G(\infty), G_{\kEuc}(\infty)$, respectively, by discarding all edges of magnitude greater than $t$.   
\end{defn}
	
We approximate $\rho_{\ell\ell}(x_i,x_j)=(D^{\ell\ell}_{G(\infty)})_{ij}$ as follows. Given a sequence of thresholds $t_{1}<t_{2}<\dots<t_{m}$, compute  $G_{\kEuc}(\infty)$ and $\{G_{\kEuc}(t_{s})\}_{s=1}^{m}$.  Then this sequence of graphs may be used to approximate $\rho_{\ell\ell}$ by finding the smallest threshold $t_s$ for which two path-connected components $C_{1},C_{2}$ merge: for $x\in C_{1}, y\in C_{2}$, we have $\rho_{\ell \ell}(x,y)\approx t_s$.  We thus approximate $\rho_{\ell\ell}(x_i,x_j)$ by $(\hat{D}_{G_ij}^{\ell \ell})_{ij}=\inf_{s}\{t_{s}\ | \ x_{i}\sim x_{j} \text{ in } G_{ij}(t_{s})\},$ where $x_{i}\sim x_{j}$ denotes that the two points are path connected.  We let $\mathbb{D} = \{\mathbf{C}_{t_s}\}_{s=1}^m$ denote the dendrogram which arises from this procedure. More specifically, $\mathbf{C}_{t_s}=\{C_{t_s}^1, \ldots C_{t_s}^{\nu_s}\}$ are the connected components of $G_{\kEuc}(t_s)$, so that $\nu_s$ is the number of connected components at scale $t_s$. 

The error incurred in this estimation of  $\rho_{\ell\ell}$ is a result of two approximations: (a) approximating LLPD in $G(\infty)$ by LLPD in $G_{\kEuc }(\infty)$; (b) approximating LLPD in $G_{\kEuc }(\infty)$ from the sequence of thresholded graphs $\{G_{\kEuc }(t_s)\}_{s=1}^m$.  Since the optimal paths which determine $\rho_{\ell\ell}$ are always paths in a \textit{minimal spanning tree} (MST) of $G(\infty)$ \citep{Hu1961}, we do not incur any error from (a) whenever an MST of $G(\infty)$ is a subgraph of $G_{\kEuc }(\infty)$. \citet{gonzalez2003clustering} show that when sampling a compact, connected manifold with sufficiently smooth boundary, the MST is a subgraph of $G_{\kEuc }(\infty)$ with high probability for $\kEuc  = O(\log(n))$. Thus for $\kEuc  = O(\log(n))$, we do not incur any error from (a) in within-cluster LLPD, as the nearest neighbor graph for each cluster will contain the MST for the given cluster. When the clusters are well-separated, we generally will incur some error from (a) in the between-cluster LLPD, but this is precisely the regime where a high amount of error can be tolerated. The following proposition controls the error incurred by (b). 

\begin{prop}\label{prop:AlgError}
	Let $G$ be a graph on $X$ and $x_{i},x_{j}\in X$ such that $(\hat{D}_G^{\ell \ell})_{ij}=t_{s}$. Then 
	$(D_G^{\ell\ell})_{ij}\le(\hat{D}_G^{\ell \ell})_{ij}\le t_{s}/(t_{s-1})(D_G^{\ell\ell})_{ij}.$
\end{prop}	
\begin{proof}		
	There is a path in $G$ connecting $x_{i},x_{j}$ with every leg of length $\le t_{s}$, since $(\hat{D}_{G}^{\ell\ell})_{ij}=t_{s}$.  Hence, $(D_G^{\ell\ell})_{ij}\le t_{s}=(\hat{D}_G^{\ell\ell})_{ij}$.  Moreover, $t_{s-1}\le (D_G^{\ell\ell})_{ij}$, since no path in $G$ with all legs $\le t_{s-1}$ connects $x_{i},x_{j}$.  It follows that $t_{s}\le\frac{t_{s}}{t_{s-1}}(D_G^{\ell\ell})_{ij}$, hence $(\hat{D}_G^{\ell\ell})_{ij}\le \frac{t_{s}}{t_{s-1}}(D_G^{\ell\ell})_{ij}$.
\end{proof}

Thus if $\{t_{s}\}_{s=1}^{m}$ grows exponentially at rate $(1+\epsilon)$, the ratio $\frac{t_{s}}{t_{s-1}}$ is bounded uniformly by $(1+\epsilon)$, and a uniform bound on the relative error is: $(D_G^{\ell\ell})_{ij}\le(\hat{D}_G^{\ell \ell})_{ij}\le (1+\epsilon)(D_G^{\ell\ell})_{ij}.$ Alternatively, one can choose the $\{t_{s}\}_{s=1}^{m}$ to be fixed percentiles in the distribution of edge magnitudes of $G$. 

Algorithm \ref{alg:CC} summarizes the multiscale construction which is used to approximate LLPD. At each scale $t_s$, the connected components of $G_{\kEuc }(t_s)$ are computed; the component identities are then stored in an $n\times m$ matrix, and the rows of the matrix are then sorted to obtain a hierarchical clustering structure. This sorted matrix of connected components (denoted $CC_{\text{sorted}}$ in Algorithm \ref{alg:CC}) can be used to quickly obtain the LLPD-nearest neighbors of each point, as discussed in Section \ref{sec:LLPD_KNN}. Note that if $G_{\kEuc }(\infty)$ is disconnected, one can add additional edges to obtain a connected graph.  

\begin{algorithm}[htb!]
	\caption{\label{alg:CC}Approximate LLPD}
	\textbf{Input:} $X,\{t_{s}\}_{s=1}^{m},\kEuc $\\
	\textbf{Output:} $\mathbb{D} = \{\mathbf{C}_{t_s}\}_{s=1}^m$, point order $\pi(i)$, $CC_{\text{sorted}}$\\
	\begin{algorithmic}[1]
		\STATE Form a $\kEuc $ E-nearest neighbors graph on $X$; call it $G_{\kEuc }(\infty)$.
		\STATE Sort the edges of $G_{\kEuc }(\infty)$ into the bins defined by the thresholds $\{t_{s}\}_{s=1}^{m}$.
		\FOR{$s=1:m$}
		\STATE Form $G_{\kEuc }(t_s)$ and compute its connected components $\mathbf{C}_{t_s}=\{C_{t_s}^1, \ldots C_{t_s}^{\nu_s}\}$.
		\ENDFOR
		\STATE Create an $n\times m$ matrix $CC$ storing each point's connected component at each scale.
		\STATE Sort the rows of $CC$ based on $\mathbf{C}_{t_m}$ (the last column).
		\FOR{$s=m:2$}
			\FOR{$i=1:\nu_s$}
				\STATE Sort the rows of $CC$ corresponding to $C^{i}_{t_s}$ according to $\mathbf{C}_{t_{s-1}}$.
			\ENDFOR
		\ENDFOR	
		\STATE Let $CC_{\text{sorted}}$ denote the $n\times m$ matrix containing the final sorted version of $CC$.
		\STATE Let $\pi(i)$ denote the point order encoded by $CC_{\text{sorted}}$.
	\end{algorithmic}
\end{algorithm}


\subsection{LLPD Nearest Neighbor Algorithm}
\label{sec:LLPD_KNN}
We next describe how to perform fast LLPD nearest neighbor queries using the multiscale graphs introduced in Section \ref{sec:LLPD_scales}.  Algorithm \ref{alg:NN_LLPD} gives pseudocode for the approximation of each point's $\kLLPD$ LLPD-nearest neighbors, with the approximation based on  the multiscale construction in Algorithm \ref{alg:CC}. 

\begin{algorithm}[tb]
	\caption{\label{alg:NN_LLPD}Fast LLPD nearest neighbor queries}
	\textbf{Input:} $X,\{t_{s}\}_{s=1}^{m},\kEuc ,\kLLPD$\\
	\textbf{Output:} $n\times n$ sparse matrix $\hat{D}^{\ell \ell}_{G_{\kEuc }}$ giving approximate $\kLLPD$ LLPD-nearest neighbors  \\
	\begin{algorithmic}[1]
		\STATE Use Algorithm \ref{alg:CC} to obtain $\pi(i)$ and $CC_{\text{sorted}}$.
		\FOR{$i=1:n$}
		\STATE $\hat{D}^{\ell \ell}_{\pi(i), \pi(i)} = t_1$
		\STATE $\textit{NN}=1$ \% Number of nearest neighbors found
		\STATE $i_{\text{up}}=1$
		\STATE $i_{\text{down}}=1$
		\FOR{$s=1:m$}
		\WHILE{$CC_{\text{sorted}}(i_{\text{up}},s)=CC_{\text{sorted}}(i_{\text{up}}-1,s)$ and $\textit{NN}<\kLLPD$ and $i_{\text{up}}>1$}
		\STATE $i_{\text{up}}=i_{\text{up}}-1$
		\STATE $\hat{D}^{\ell \ell}_{\pi(i), \pi(i_{\text{up}})} = t_s$
		\STATE $\textit{NN}=\textit{NN}+1$
		\ENDWHILE
		\WHILE{$CC_{\text{sorted}}(i_{\text{down}},s)=CC_{\text{sorted}}(i_{\text{down}}+1,s)$ and $\textit{NN}<\kLLPD$ and $i_{\text{down}}<n$}
		\STATE $i_{\text{down}}=i_{\text{down}}+1$
		\STATE $\hat{D}^{\ell \ell}_{\pi(i), \pi(i_{\text{down}})} = t_s$
		\STATE $\textit{NN}=\textit{NN}+1$
		\ENDWHILE
		\ENDFOR
		\ENDFOR	
		\RETURN  $\hat{D}^{\ell \ell}_{G_{\kEuc }}$
	\end{algorithmic}
\end{algorithm}

Figure \ref{fig:algpic} illustrates how Algorithm \ref{alg:NN_LLPD} works on a data set consisting of 11 points and 4 scales. Letting $\pi$ denote the ordering of the points in $CC_{\text{sorted}}$ as produced by Algorithm \ref{alg:CC}, $CC_{\text{sorted}}$ is queried to find $x_{\pi(6)}$'s 8 LLPD-nearest neighbors (nearest neighbors are shown in bold). 
Starting in the first column of $CC_{\text{sorted}}$ which corresponds to the finest scale ($s=1$), points in the same connected component as the base point are added to the nearest neighbor set, and the LLPD to these points is recorded as $t_1$. Assuming the nearest neighbors set does not yet contain $\kLLPD$ points, one then adds to it any points not yet in the nearest neighbor set which are in the same connected component as the base point at the second finest scale, and records the LLPD to these neighbors as $t_2$ (see the second column of Figure \ref{fig:algpic} which illustrates $s=2$ in the pseudocode). One continues in this manner until $\kLLPD$ neighbors are found.

\begin{rems}
For a fixed $x$, there might be many points of equal LLPD to $x$.  This is in contrast to the case for Euclidean distance, where such phenomena typically occur only for highly structured data, for example, for data consisting of points lying on a sphere and $x$ the center of the sphere.  In the case that $\kLLPD$ LLPD nearest neighbors for $x$ are sought and there are more than $\kLLPD$ points at the same LLPD from $x$, Algorithm \ref{alg:NN_LLPD} returns a sample of these LLPD-equidistant points in $O(m+\kLLPD)$ by simply returning the first $\kLLPD$ neighbors encountered in a fixed ordering of the data; a random sample could be returned for an additional cost.  
\end{rems}

\begin{figure}[!htb]
	\centering
	\begin{subfigure}[b]{.5\textwidth}
		\[
		\begin{bmatrix}
		t_1 & & t_2 & & t_3 & & t_4 & \\ \hline
		1 &  & 2 & & \bf{1} & & 1 & x_{\pi(1)} \\
		&  &    &  &  \uparrow & &    &       \\
		1 &  & 2 & & \bf{1} & & 1 & x_{\pi(2)} \\
		&  &    &  & \uparrow  & &    &       \\
		1 &  & 2 & & \bf{1} & & 1 & x_{\pi(3)} \\
		&  &  \uparrow  &  &  \uparrow & &    &       \\
		4 &  & \bf{3} & \rightarrow & 1 & & 1 & x_{\pi(4)} \\
		&  &  \uparrow  &  &   & &    &       \\
		4 &  & \bf{3} & & 1 & & 1 & x_{\pi(5)} \\
		\uparrow   &  &  \uparrow  &  &   & &    &       \\
		\fbox{\bf{5}} & \rightarrow & 3 & & 1 & & 1 & x_{\pi(6)} \\
		\downarrow  &  &   &  &   & &    &       \\
		\bf{5} & \rightarrow & 3 & \rightarrow & 1 & \rightarrow & 1 & x_{\pi(7)} \\
		\downarrow   &  &   \downarrow &  &  \downarrow & &  \downarrow  &       \\
		2 &  & 1 & & 2 & & \bf{1} & x_{\pi(8)} \\
		&  &    &  &   & &  \downarrow  &       \\
		2 &  & 1 & & 2 & & \bf{1} & x_{\pi(9)} \\
		&  &    &  &   & &    &       \\
		3 &  & 1 & & 2 & & 1 & x_{\pi(10)} \\
		&  &    &  &   & &    &       \\

		\end{bmatrix}
		\]
		\caption{Illustration of Algorithm \ref{alg:NN_LLPD}}
		\label{fig:algpic}
	\end{subfigure}
	\qquad
	\begin{subfigure}[b]{.36\textwidth}
		\includegraphics[width=\textwidth]{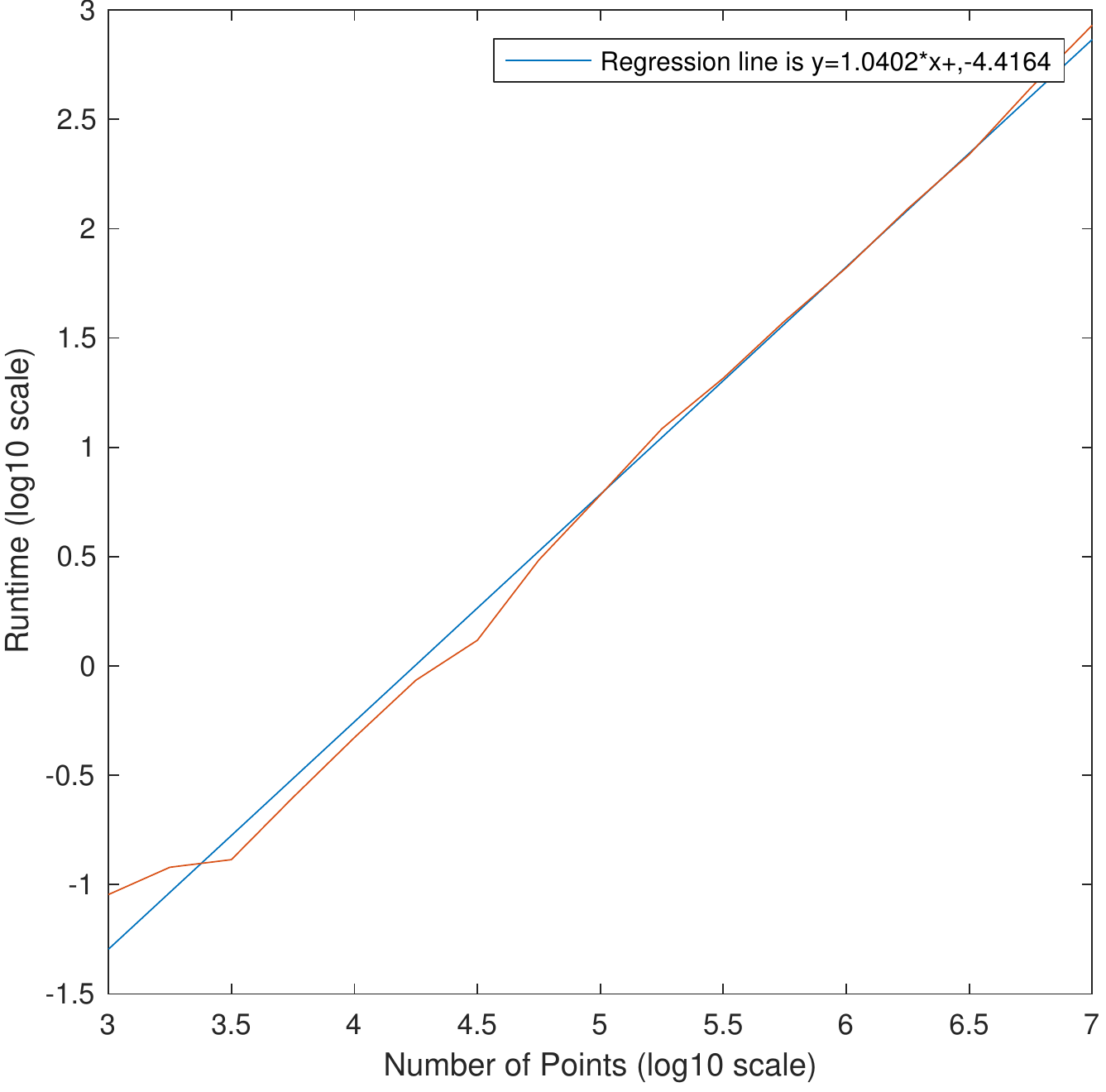}
		\includegraphics[width=\textwidth]{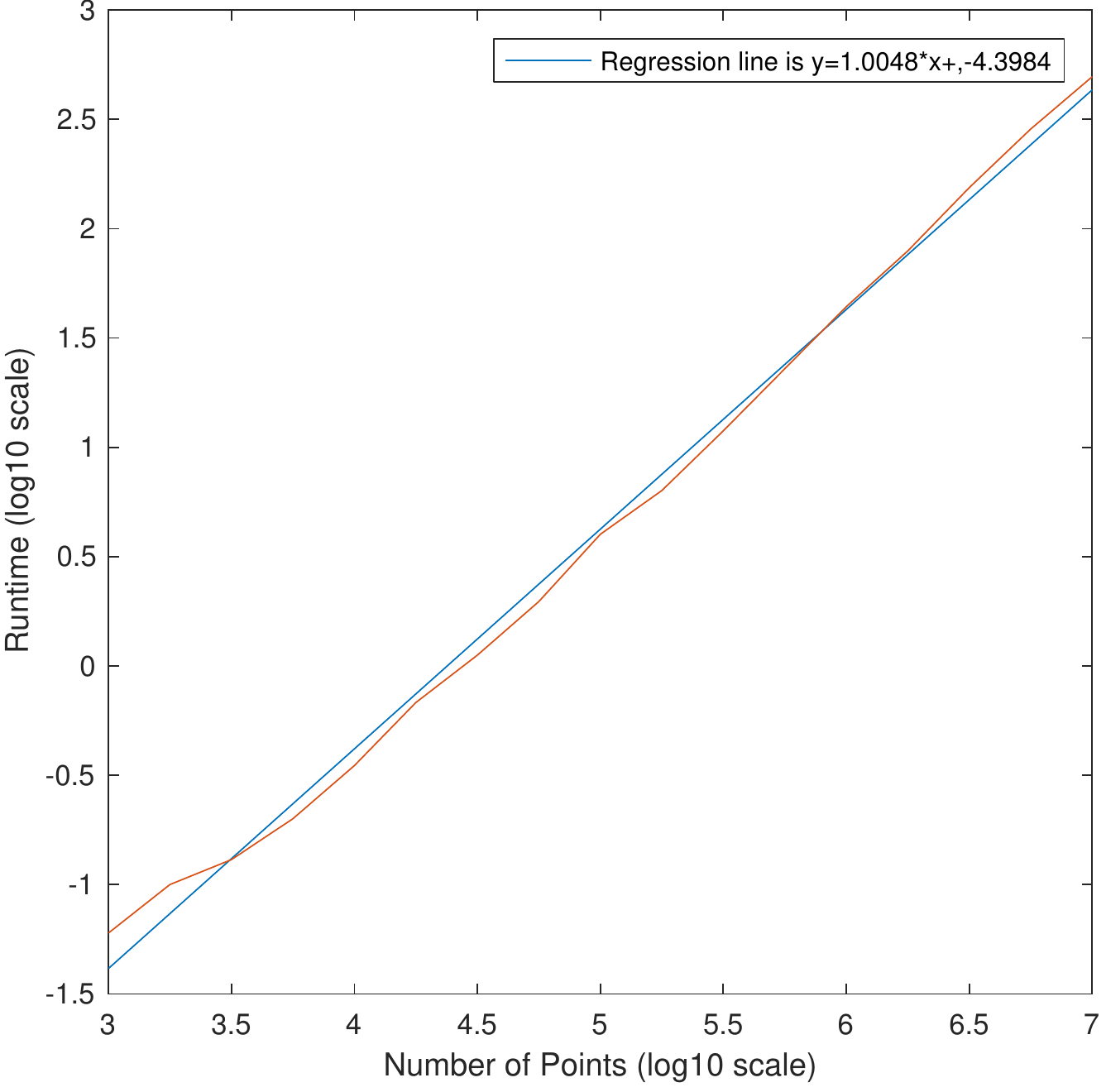}

		\caption{Complexity plots for Algorithm \ref{alg:NN_LLPD}}
		\label{fig:log(runtime)}
	\end{subfigure}
	\caption{Algorithm \ref{alg:NN_LLPD} is demonstrated on a simple example in (a). The figure illustrates how $CC_{\text{sorted}}$ is queried to return $x_{\pi(6)}$'s 8 LLPD-nearest neighbors. Nearest neighbors are shown in bold, and $\hat{\rho}_{\ell\ell}(x_{\pi(6)}, x_{\pi(7)})=t_1$, $\hat{\rho}_{\ell\ell}(x_{\pi(6)}, x_{\pi(5)})=t_2$, etc. Note each upward or downward arrow represents a comparison which checks whether two points are in the same connected component at the given scale. In (b), the runtime of Algorithm \ref{alg:NN_LLPD} on uniform data in $[0,1]^{2}$ is plotted against number of points in log scale.  The slope of the line is approximately 1, indicating that the algorithm is essentially quasilinear in the number of points.  Here, $\kEuc =20, \kLLPD=10, D=2$, and the thresholds $\{t_s\}_{s=1}^m$ correspond to fixed percentiles of edge magnitudes in $G_{\kEuc }(\infty)$.  The top plot has $m=10$ and the bottom plot $m=100$.}
	\label{fig:combinedfig}
\end{figure}

Figure \ref{fig:log(runtime)} shows a plot of the empirical runtime of the proposed algorithm against number of points in log scale, suggesting nearly linear runtime.  This is confirmed theoretically as follows.

\begin{thm} 
 Algorithm \ref{alg:NN_LLPD} has complexity $O(n(\kEuc C_{\text{NN}}+m(\kEuc \vee\log(n))+\kLLPD))$.
\end{thm}	

\begin{proof} 

The major steps of Algorithm \ref{alg:NN_LLPD} (which includes running Algorithm \ref{alg:CC}) are:
\begin{itemize}
	\setlength\itemsep{0em}
	\item Generating the $\kEuc $ E-nearest neighbors graph $G_{\kEuc }(\infty)$: $O(\kEuc nC_{\text{NN}}$), where $C_{\text{NN}}$ is the cost of an E-nearest neighbor query. For high-dimensional data $C_{\text{NN}}=O(nD)$. When the data has low intrinsic dimension $d<D$ cover trees \citep{Beygelzimer2006} allows $C_{\text{NN}}=O(DC^d\log(n))$, after a pre-processing step with cost $O(C^d Dn\log(n))$. 
	
	\item Binning the edges of $G_{\kEuc }(\infty): O(\kEuc n(m\wedge\log(\kEuc n)))$. Binning without sorting is $O(\kEuc nm)$; if the edges are sorted first, the cost is $O(\kEuc n\log(\kEuc n))$.
	
	\item Forming $G_{\kEuc }(t_{s})$, for $s=1,\dots,m$, and computing its connected components: $O(\kEuc  mn)$.

	\item Sorting the connected components matrix to create $CC_{\text{sorted}}$: $O(mn\log(n))$.
	
	\item Finding each point's $\kLLPD$ LLPD-nearest neighbors by querying $CC_{\text{sorted}}$: $O(n(m+\kLLPD))$.
		
\end{itemize}

Observe that $O(C_{\text{NN}})$ always dominates $O(m\wedge\log(\kEuc n))$.  Hence, the overall  complexity is $O(n(\kEuc C_{\text{NN}}+m(\kEuc \vee\log(n))+\kLLPD)).$\end{proof}

\begin{cor}
If $\kEuc ,\kLLPD,m=O(1)$ with respect to $n$ and the data has low intrinsic dimension so that $C_{\text{NN}}=O(DC^d\log(n))$, Algorithm \ref{alg:NN_LLPD} has complexity $O(DC^{d}n \log(n))$.
\end{cor}

If $\kLLPD=O(n)$ or the data has high intrinsic dimension, the complexity is $O(n^2)$. Hence, $d,m,\kEuc ,$ and $\kLLPD$ are all important parameters affecting the computational complexity. 

\begin{rems} One can also incorporate a minimal spanning tree (MST) into the construction, i.e. replace $G_{\kEuc }(\infty)$ with its MST. This will reduce the number of edges which must be binned to give a total computational complexity of $O(n(\kEuc C_{\text{NN}}+m\log(n)+\kLLPD))$.  Computing the LLPD with and without the MST has the same complexity when $\kEuc  \leq O(\log(n))$, so for simplicity we do not incorporate MSTs in our implementation.
\end{rems}


\subsection{A Fast Eigensolver for LLPD Laplacian}\label{subsec:fast_eigensolver}

In this section we describe an algorithm for computing the eigenvectors of a dense Laplacian defined using approximate LLPD with complexity $O(mn)$.  The ultrametric property of the LLPD makes $\Lsym$ highly compressible, which can be exploited for fast eigenvector computations. Assume LLPD is approximated using $m$ scales $\{t_s\}_{s=1}^m$ and the corresponding thresholded graphs $G_{\kEuc}(t_s)$ as described in Algorithm \ref{alg:CC}. Let $n_i= |C_{t_1}^i|$ for $1\leq i\leq \nu_1$ denote the cardinalities of the connected components of $G_{\kEuc}(t_1)$, and $V=\sum_{k=1}^m\nu_k$ the total number of connected components across all scales.  

In order to develop a fast algorithm for computing the eigenvectors of the LLPD Laplacian $\Lsym=I-D^{-\frac{1}{2}}WD^{-\frac{1}{2}}$, it suffices to describe a fast method for computing the matrix-vector multiplication $x\mapsto \Lsym x$, where $\Lsym$ is defined using $W_{ij} = e^{-\rho_{\ell\ell}(x_i,x_j)^2/\sigma^2}$ \citep{Trefethen1997_Numerical}.  Assume without loss of generality that we order the entries of both $x$ and $W$ according to the point order $\pi$ defined in Algorithm \ref{alg:CC}. Note because $\Lsym$ is block constant with $\nu_1^{2}$ blocks, any eigenvector will also be block constant with $\nu_1$ blocks, and it suffices to develop a fast multiplier for $x\mapsto \Lsym x$ when $x\in\mathbb{R}^n$ has the form: $x=[x_1 1_{n_1}\  x_2 1_{n_2} \ldots x_{\nu_1} 1_{\nu_1}]$ where $1_{n_i} \in \mathbb{R}^{n_i}$ is the all one's vector.  Assuming LLPD's have been precomputed using Algorithm \ref{alg:CC}, Algorithm \ref{alg:FastMatVec} gives pseudocode for computing $Wx$ with complexity $O(mn)$. Since $\Lsym x = x - D^{-1/2}WD^{-1/2}x$, $W(D^{-1/2}x)$ is computable via Algorithm \ref{alg:FastMatVec}, and $D^{-1/2}$ is diagonal, a straight forward generalization of Algorithm \ref{alg:FastMatVec} gives $\Lsym x$ in $O(mn)$.

\begin{algorithm}[htb!]
\small{
	\caption{\label{alg:FastMatVec}Fast LLPD Matrix-Vector Multiplication}
	\textbf{Input:} $\{t_s\}_{s=1}^{m}$, $\mathbb{D} = \{\mathbf{C}_{t_s}\}_{s=1}^m$, $f_{\sigma}(t)$, $x$\\
	\textbf{Output:} $Wx$
	\begin{algorithmic}[1]
	\STATE  Enumerate all connected components at all scales: $\mathbf{C} = [C_{t_1}^1 \ldots C_{t_1}^{\nu_1} \ldots C_{t_m}^1 \ldots C_{t_m}^{\nu_m}]$.
	\STATE Let $\mathcal{V}$ be the collection of $V$ nodes corresponding to the elements of $\mathbf{C}$.
	\STATE For $i=1,\dots,V$, let $\mathcal{C}(i)$ be the set of direct children of node $i$ in dendrogram $\mathbb{D}$. 
	\STATE For $i=1,\dots,V$, let $\mathcal{P}(i)$ be the direct parent of node $i$ in dendrogram $\mathbb{D}$.\\[.2cm]
	
	\FOR{$i=1:\nu_1$}
		\STATE $\Sigma(i) = n_ix_i$
	\ENDFOR \\
	\FOR{$i=(\nu_1+1):V$}
		\STATE $\Sigma(i) = \sum_{j \in \mathcal{C}(i)} \Sigma(j)$
	\ENDFOR	\\[.2cm]
	\FOR{$i=1:\nu_1$}
		\STATE $\alpha_i(1) = i$
		\FOR{$j=2:m$}
			\STATE $\alpha_i(j) =\mathcal{P}(\alpha_i(j-1)$)
		\ENDFOR
	\ENDFOR	\\[.2cm]

	\STATE Let $K=[f_{\sigma}(t_1)\ f_{\sigma}(t_2) \cdots f_{\sigma}(t_m)]$ be a vector of kernel evaluations at each scale.
	\\[.2cm]
	
	\FOR{$i=1:\nu_1$}
		\STATE $\xi_i(j) = \Sigma(\alpha_i(j))$
		\STATE $d\xi_i(1) = \xi_i(1)$
		\FOR{$j=2:m$}
			\STATE $d\xi_i(j) = \xi_i(j+1) - \xi_i(j)$ 
		\ENDFOR	
		\STATE Let $I_i$ be the index set corresponding to $C_{t_1}^i$.
		\STATE $(Wx)_{I_i} = \sum_{s=1}^{m} d\xi_i(s)K(s)$ 
	\ENDFOR	
	\end{algorithmic}
	}
\end{algorithm}

Since the matrix-vector multiplication has reduced complexity $O(mn)$, the decomposition of the principal $K$ eigenvectors can be done with complexity $O(K^{2}mn)$ \citep{Trefethen1997_Numerical}, which in the practical case $K,m=O(1)$, is essentially linear in $n$. Thus the total complexity of implementing LLPD spectral clustering including the LLPD approximation discussed in Section \ref{sec:LLPD_scales} becomes $O(n(\kEuc C_{\text{NN}}+m(\kEuc \vee\log(n)\vee \numclust^{2})))$. We defer timing studies and theoretical analysis of the fast eigensolver algorithm to a subsequent article, in the interest of space.  We remark that the strategy proposed for computing the eigenvectors of the LLPD Laplacian could in principal be used for Laplacians derived from other distances.  However, without the compressible ultrametric structure, the approximation using only $m\ll n$ scales is likely to be poor, leading to inaccurate approximate eigenvectors.   


\subsection{LLPD as Approximate Single Linkage Clustering}\label{subsec:ApproximateSLC}

The algorithmic implementation giving $\hat{D}^{\ell \ell}_{G_{\kEuc }}$ approximates the true LLPD $\rho_{\ell\ell}$ by merging path connected components at various scales.  In this sense, our approach is reminiscent of single linkage clustering \citep{Hastie2009}. Indeed, the connected component structure defined in Algorithm \ref{alg:CC} can be viewed as an approximate single linkage dendrogram.

Single linkage clustering generates, from $X=\{x_{i}\}_{i=1}^{n}$, a dendrogram $\mathbb{D}_{\text{SL}}=\{\mathbf{C}_{k}\}_{k=0}^{n-1}$, where $\mathbf{C}_{k}:\{1,2,\dots,n\}\rightarrow \{C_k^1,C_k^2,\dots,C_k^{n-k}\}$ assigns $x_{i}$ to its cluster at level $k$ of the dendrogram ($\mathbf{C}_{0}$ assigns each point to a singleton cluster).  Let $d_{k}$ be the Euclidean distance between the clusters merged at level $k$: $d_{k}=\min_{i \ne j}\min_{x\in C^{k-1}_i, y\in C^{k-1}_j}\|x-y\|_2.$  Note that $\{d_{k}\}_{k=1}^{n-1}$ is non-decreasing, and when strictly increasing, the clusters produced by single linkage clustering at the $k^{th}$ level are the path connected components of $G(d_k)$.  In the more general case, the path connected components of $G(d_k)$ may correspond to multiple levels of the single linkage hierarchy.  Let $\{t_{s}\}_{s=1}^{m}$ be the thresholds used in Algorithm \ref{alg:CC}, and assume that $G_{\kEuc }(\infty)$ contains an MST of $G(\infty)$ as a subgraph.  Let $\mathbb{D}=\{\mathbf{C}_{t_s}\}_{s=1}^{m}$ be the path-connected components with edges $\le t_{s}$. $\mathbb{D}$ is a compressed dendrogram, obtained from the full dendrogram $\mathbb{D}_{\text{SL}}$ by pruning at certain levels. Let $\tau_{s}=\inf\{k \ | \ d_{k}\ge t_{s},\ d_{k} < d_{k+1}\}$, and define the \emph{pruned dendrogram} as $P(\mathbb{D}_{\text{SL}})=\{\mathbf{C}_{\tau_{s}}\}_{s=1}^{m}$.  In this case, the dendrogram obtained from the approximate LLPD is a pruning of an exact single linkage dendrogram.  We omit the proof of the following in the interest of space.

\begin{prop}
If $G_{\kEuc }(\infty)$ contains an MST of $G(\infty)$ as a subgraph,
$P(\mathbb{D}_{\text{SL}})=\mathbb{D}$.  
\end{prop}

Note that the approximate LLPD algorithm also offers an inexpensive approximation of single linkage clustering. A naive implementation of single linkage clustering is $O(n^{3})$, while the SLINK algorithm \citep{Sibson1973} improves this to $O(n^{2})$. Thus to generate $\mathbb{D}$ by first performing exact single linkage clustering, then pruning, is $O(n^{2})$, whereas to approximate $\mathbb{D}$ directly via approximate LLPD is $O(n\log(n))$; see Figure \ref{fig:CommutativeDiagram}.

\begin{figure}[!htb]
	\centering
	\begin{minipage}{.49\textwidth}
		\centering
		\begin{tikzcd}
			& \mathbb{D} \arrow{dr}{Pruning} \\
			X \arrow{ur}{SLC} \arrow{rr}{LLPD} && \tilde{\mathbb{D}}
		\end{tikzcd}
	\end{minipage}
	\begin{minipage}{.49\textwidth}
		\centering
		\begin{tikzcd}
			& \mathbb{D} \arrow{dr}{O(mn)} \\
			X \arrow{ur}{O(n^2)} \arrow{rr}{O(n\log(n))} && \tilde{\mathbb{D}}
		\end{tikzcd}
	\end{minipage}
	\caption{\label{fig:CommutativeDiagram}The cost of constructing the full single linkage dendrogram with SLINK is $O(n^{2})$, and the cost of pruning is $O(mn)$, where $m$ is the number of pruning cuts, so that acquiring $\mathbb{D}$ in this manner has overall complexity $O(n^{2})$.  The proposed method, in contrast, computes $\mathbb{D}$ with complexity $O(n\log(n))$.}
\end{figure}
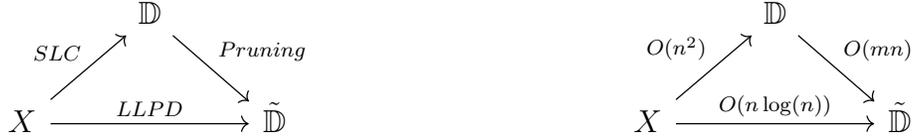


\section{Numerical Experiments}\label{sec:NumericalExperiments}

In this section we illustrate LLPD spectral clustering on four synthetic data sets and five real data sets.  LLPD was approximated using Algorithm \ref{alg:CC}, and data sets were denoised by removing all points whose $\kNoise^{\!\!\!\!\!\text{th}}$ nearest neighbor LLPD exceeded $\theta$. Algorithm \ref{alg:FastMatVec} was then used to compute approximate eigenpairs of the LLPD Laplacian for a range of $\sigma$ values. The parameters $\hat{\numclust},\hat{\sigma}$ were then estimated from the multiscale spectral decompositions via
\begin{align*}
\hat{\numclust}&=\argmax_{i}\max_\sigma (\lambda_{i+1}(\sigma)-\lambda_{i}(\sigma))\quad,\quad \hat{\sigma} = \argmax_{\sigma} \left(\lambda_{\hat{\numclust}+1}(\sigma)-\lambda_{\hat{\numclust}}(\sigma)\right)\, ,
\end{align*}
and a final clustering was obtained by running $K$-means on the spectral embedding defined by the principal $\numclust$ eigenvectors of $\Lsym(\hat{\sigma})$. For each data set, we investigate (1) whether $\hat{\numclust}=\numclust$ and (2) the labeling accuracy of LLPD spectral clustering given $\numclust$. We compare the results of (1) and (2) with those obtained from Euclidean spectral clustering, where $\hat{\numclust}, \hat{\sigma}$ are estimated using an identical procedure, and also compare the results of (2) with the labeling accuracy obtained by applying $K$-means directly. To make results as comparable as possible, Euclidean spectral clustering and $K$-means were run on the LLPD denoised data sets. All results are reported in Table \ref{tab:Results}. 

Labeling accuracy was evaluated using three statistics: \emph{overall accuracy (OA)}, \emph{average accuracy (AA)}, and \emph{Cohen's $\kappa$}.  OA is the metric used in the theoretical analysis, namely the proportion of correctly labeled points after clusters are aligned, as defined by the agreement function (\ref{eqn:agreement}).  AA computes the overall accuracy on each cluster separately, then averages the results, in order to give small clusters equal weight to large ones.  Cohen's $\kappa$ measures agreement between two labelings, corrected for random agreement \citep{Banerjee1999_Beyond}. Note that AA and $\kappa$ are computed using the alignment that is optimal for OA.  We note that accuracy is computed only on the points with ground truth labels, and in particular, any noise points remaining after denoising are ignored in the accuracy computations.  For the synthetic data, where it is known which points are noise and which are from the clusters, one can assign labels to noise points according to Euclidean distance to the nearest cluster.  For all synthetic datasets considered, the empirical results observed changed only trivially, and we do not report these results.  

Parameters were set consistently across all examples, unless otherwise noted.  The initial E-nearest neighbor graph was constructed using $\kEuc =20$. The scales $\{t_{s}\}_{s=1}^{m}$ for approximation were chosen to increase exponentially while requiring $m=20$.  Nearest neighbor denoising was performed using $\kNoise=20$. The denoising threshold $\thetathres$ was chosen by estimating the elbow in a graph of sorted nearest neighbor distances. For each data set, $\Lsym$ was computed for 20 $\sigma$ values equally spaced in an interval.  All code and scripts to reproduce the results in this article are publicly available\footnote{\url{https://bitbucket.org/annavlittle/llpd_code/branch/v2.1}}.  

\subsection{Synthetic Data}
 
 The four synthetic data sets considered are:
 
\begin{itemize}
	\setlength\itemsep{0em}

	\item \textbf{Four Lines} This data set consists of four highly elongated clusters in $\mathbb{R}^{2}$ with uniform two-dimensional noise added; see Figure \ref{fig:four_lines_data}.  The longer clusters have $n_{i}=40000$ points, the smaller ones $n_{i}=8000$, with $\tilde{n}=20000$ noise points.  This dataset is too large to cluster with a dense Euclidean Laplacian.  
		\item \textbf{Nine Gaussians} Each of the nine clusters consist of $n_i=50$ random samples from a two-dimensional Gaussian distribution; see Figure \ref{fig:nine_gaussians_data}.  All of the Gaussians have distinct means. Five have covariance matrix $0.01I$ while four have covariance matrix $0.04I$, resulting in clusters of unequal density. The noise consists of $\tilde{n}=50$ uniformly sampled points.
	\item \textbf{Concentric Spheres} Letting $\mathbb{S}^d_r \subset \mathbb{R}^{d+1}$ denote the $d$-dimensional sphere of radius $r$ centered at the origin, the clusters consist of points uniformly sampled from three concentric 2-dimensional spheres embedded in $\mathbb{R}^{1000}$: $n_1=250$ points from $\mathbb{S}^2_1$, $n_2 = 563$ points from $\mathbb{S}^2_{1.5}$, and $n_3 = 1000$ points from $\mathbb{S}^2_2$, so that the cluster densities are constant. The noise consists of an additional $\tilde{n}=2000$ points uniformly sampled from $[-2,2]^{1000}$.
	\item \textbf{Parallel Planes} Five $d=5$ dimensional planes are embedded in $[0,1]^{25}$ by setting the last $D-d = 20$ coordinates to a distinct constant for each plane; we sample uniformly $n_i=1000$, $1 \leq i \leq 5$ points from each plane and add $\tilde{n}=200000$ noise points uniformly sampled from $[0,1]^{25}$. Only 2 of the last 20 coordinates contribute to the separability of the planes, so that the Euclidean distance between consecutive parallel planes is approximately 0.35.  We note that for this dataset, it is possible to run Euclidean spectral clustering after denoising with the LLPD.  
\end{itemize}

\begin{figure}[!htb]
	\centering
	\begin{subfigure}[t]{.24\textwidth}
		\includegraphics[width=\textwidth]{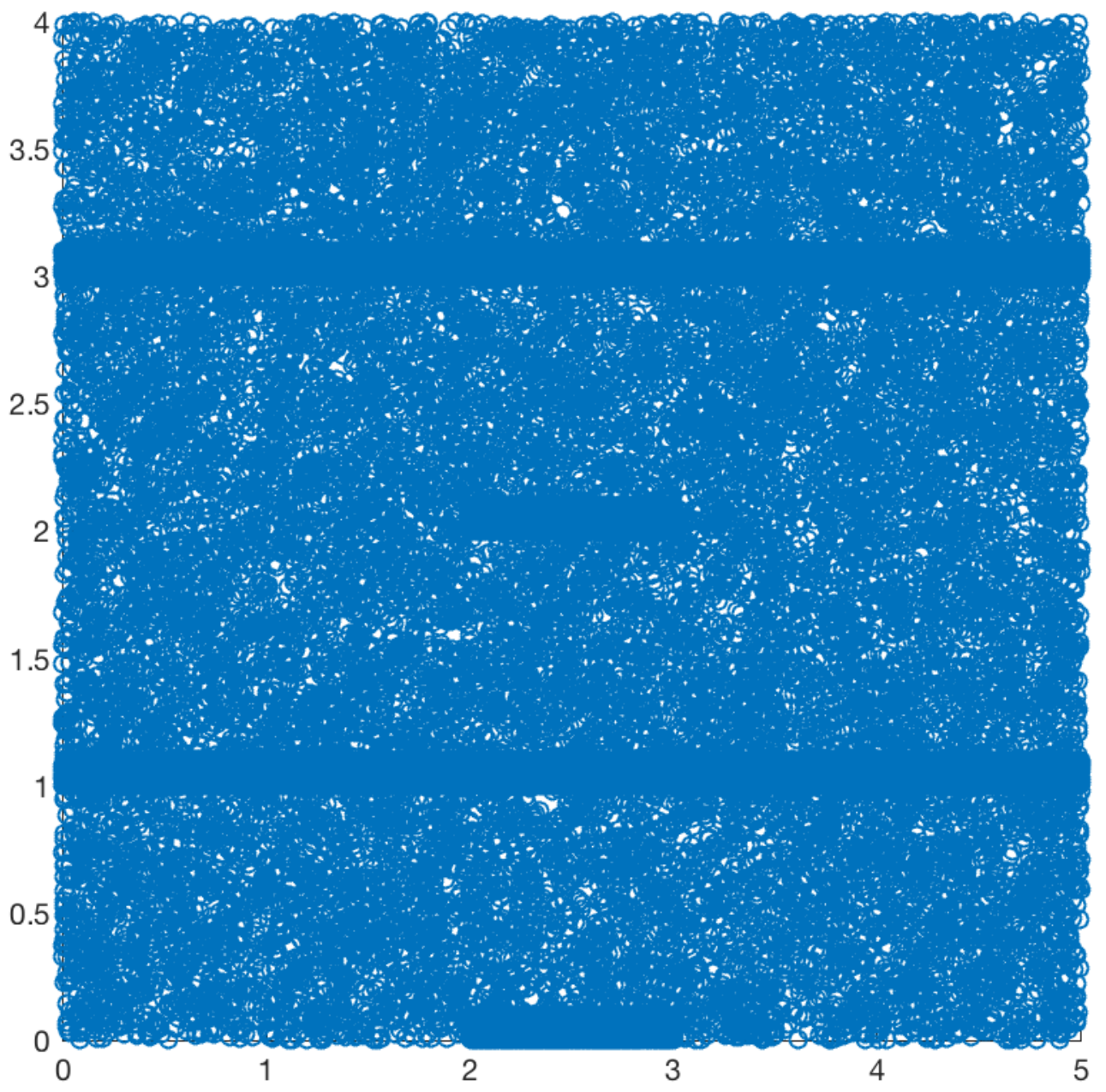}
		\subcaption{Four Lines}
		\label{fig:four_lines_data}
	\end{subfigure}
	\begin{subfigure}[t]{.24\textwidth}
		\includegraphics[width=\textwidth]{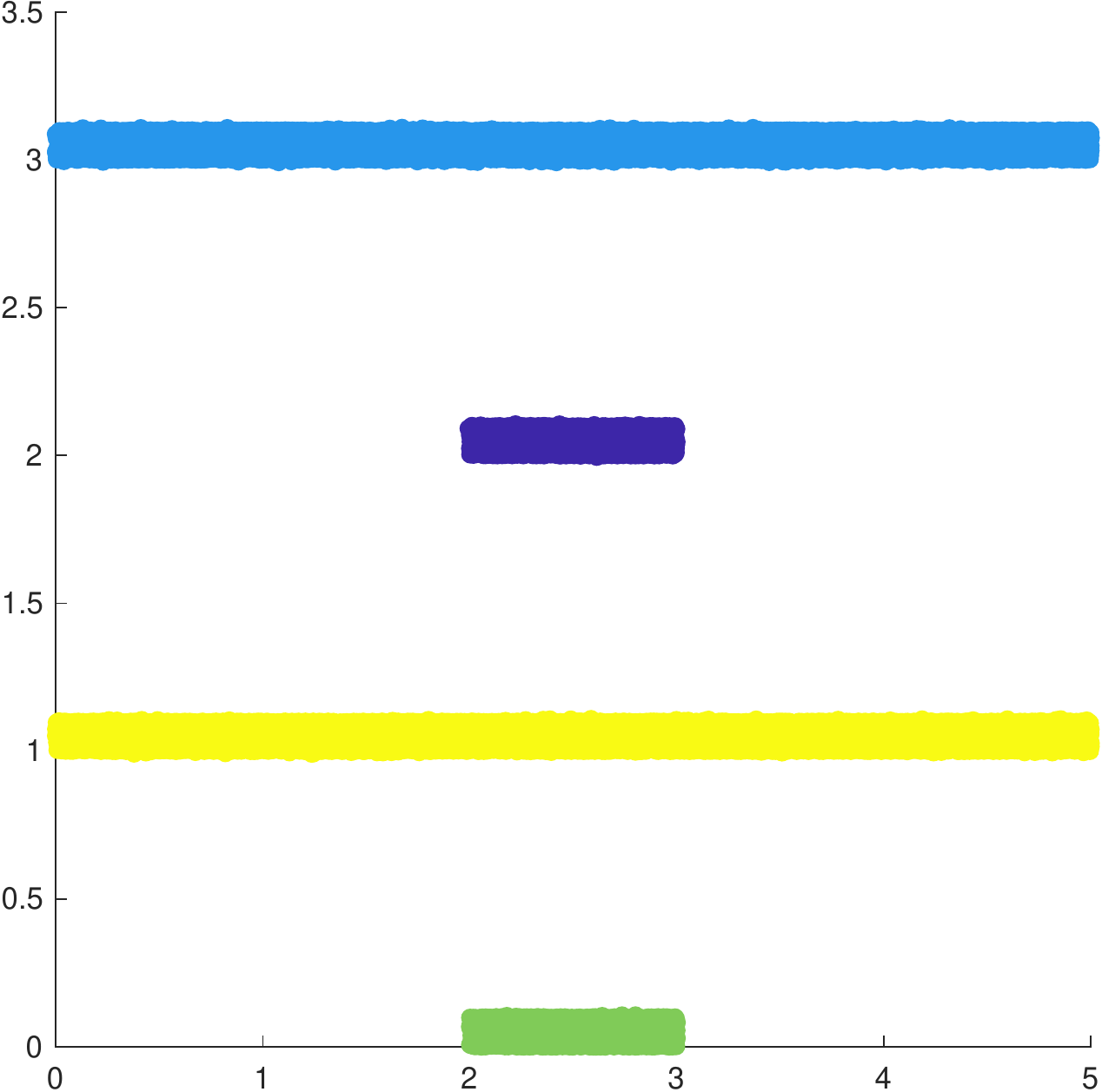}
		\subcaption{LLPD spectral clustering on denoised Four Lines}
		\label{fig:four_lines_denoised_data}
	\end{subfigure}
	\begin{subfigure}[t]{.24\textwidth}
		\includegraphics[width=\textwidth]{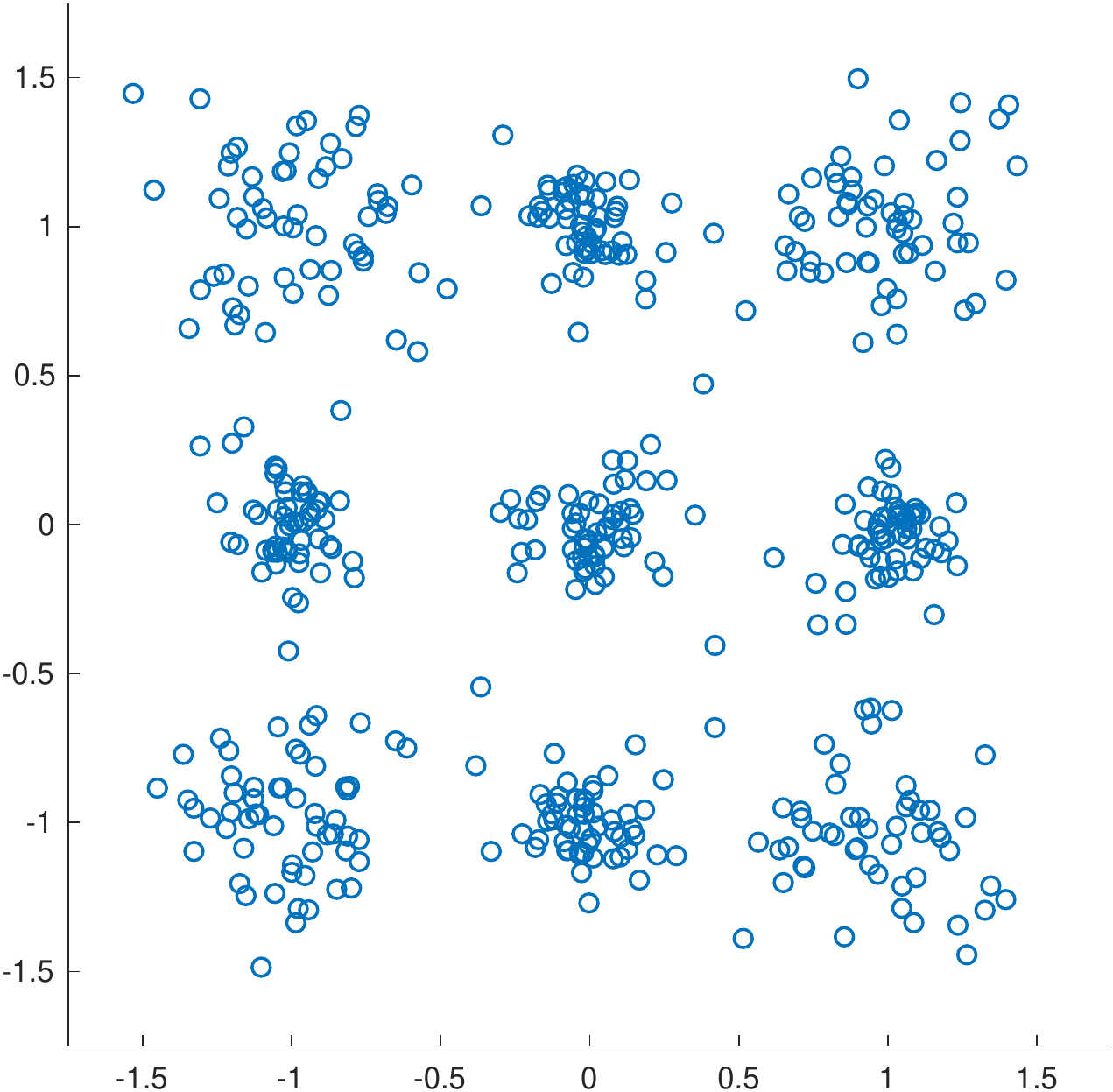}
		\subcaption{Nine Gaussians}
		\label{fig:nine_gaussians_data}
	\end{subfigure}	
	\begin{subfigure}[t]{.24\textwidth}
		\includegraphics[width=\textwidth]{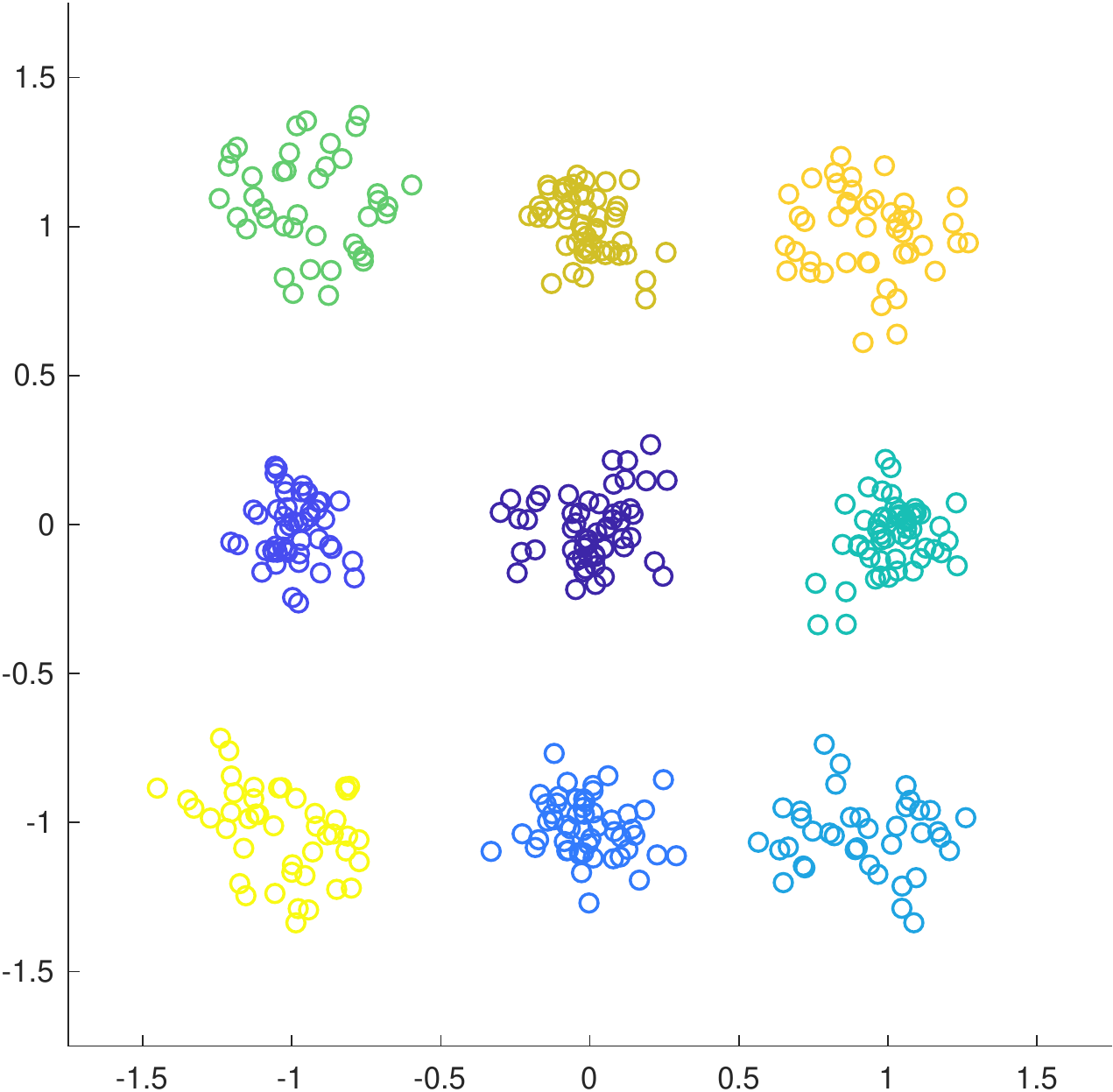}
		\subcaption{LLPD spectral clustering on denoised Nine Gaussians}
		\label{fig:nine_gaussians_denoised_data}
	\end{subfigure}
	\caption{Two dimensional synthetic data sets and LLPD spectral clustering results for the denoised data sets. In Figures \ref{fig:four_lines_denoised_data} and \ref{fig:nine_gaussians_denoised_data}, color corresponds to the label returned by LLPD spectral clustering.}
	\label{fig:SyntheticDatasets}		
\end{figure}

Figure \ref{fig:Denoising} illustrates the denoising procedure. Sorted LLPD-nearest neighbor distances are shown in blue, and the denoising threshold $\thetathres$ (selected by choosing the graph elbow) is shown in red. All plots exhibit an elbow pattern, which is shallow when $D/d$ is small (Figure \ref{fig:DenoisingNG}) and sharp when $D/d$ is large (Figure \ref{fig:DenoisingCS}; the sharpness is due to the drastic difference in nearest neighbor distances for cluster and noise points).

\begin{figure}[!htb]
	\begin{subfigure}[t]{.24\textwidth}
		\captionsetup{width=\linewidth}
		\includegraphics[width=\textwidth]{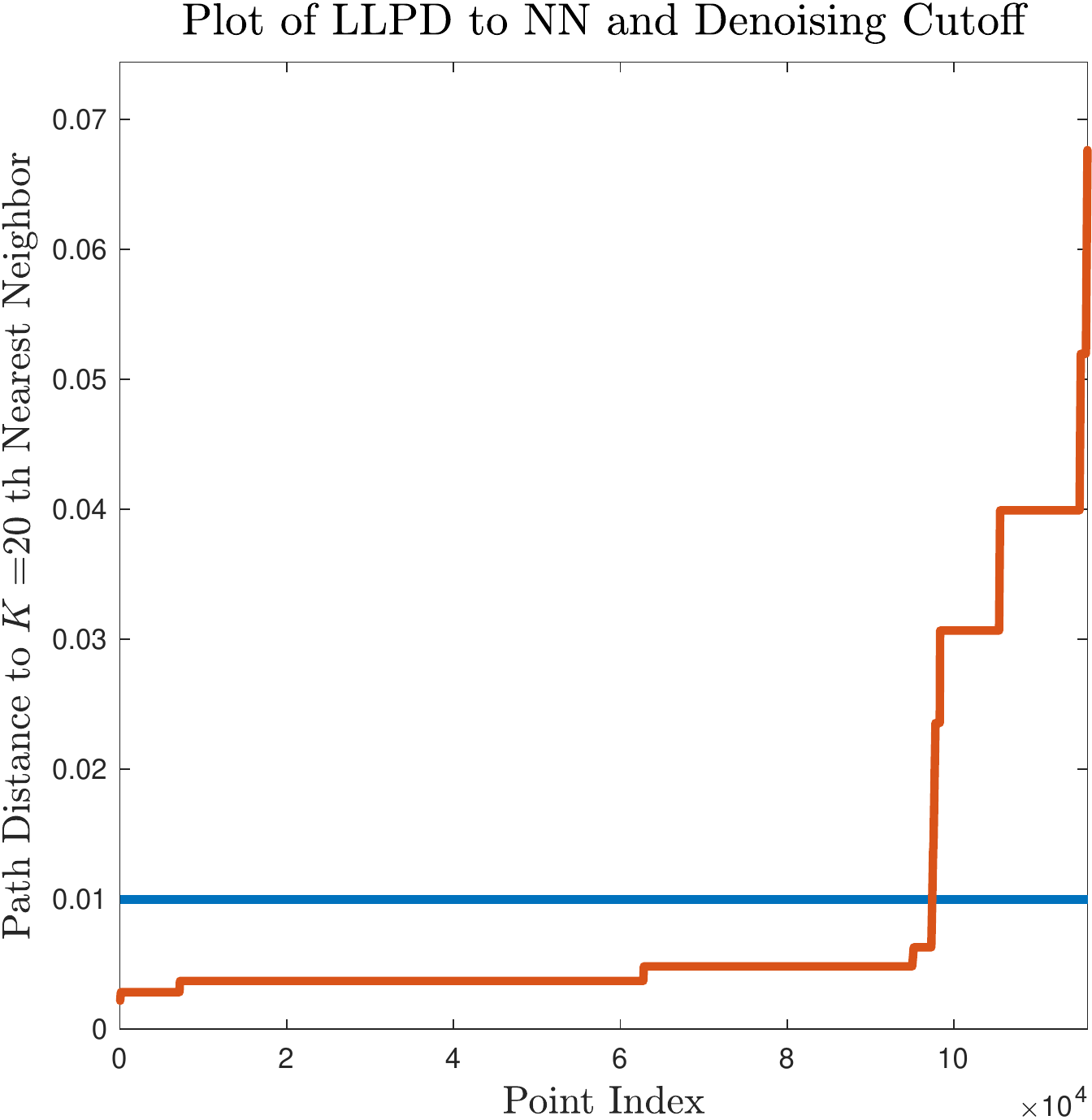}
		\subcaption{\label{fig:DenoisingFL}Four Lines}
	\end{subfigure}
	\begin{subfigure}[t]{.24\textwidth}
		\captionsetup{width=\linewidth}
		\includegraphics[width=\textwidth]{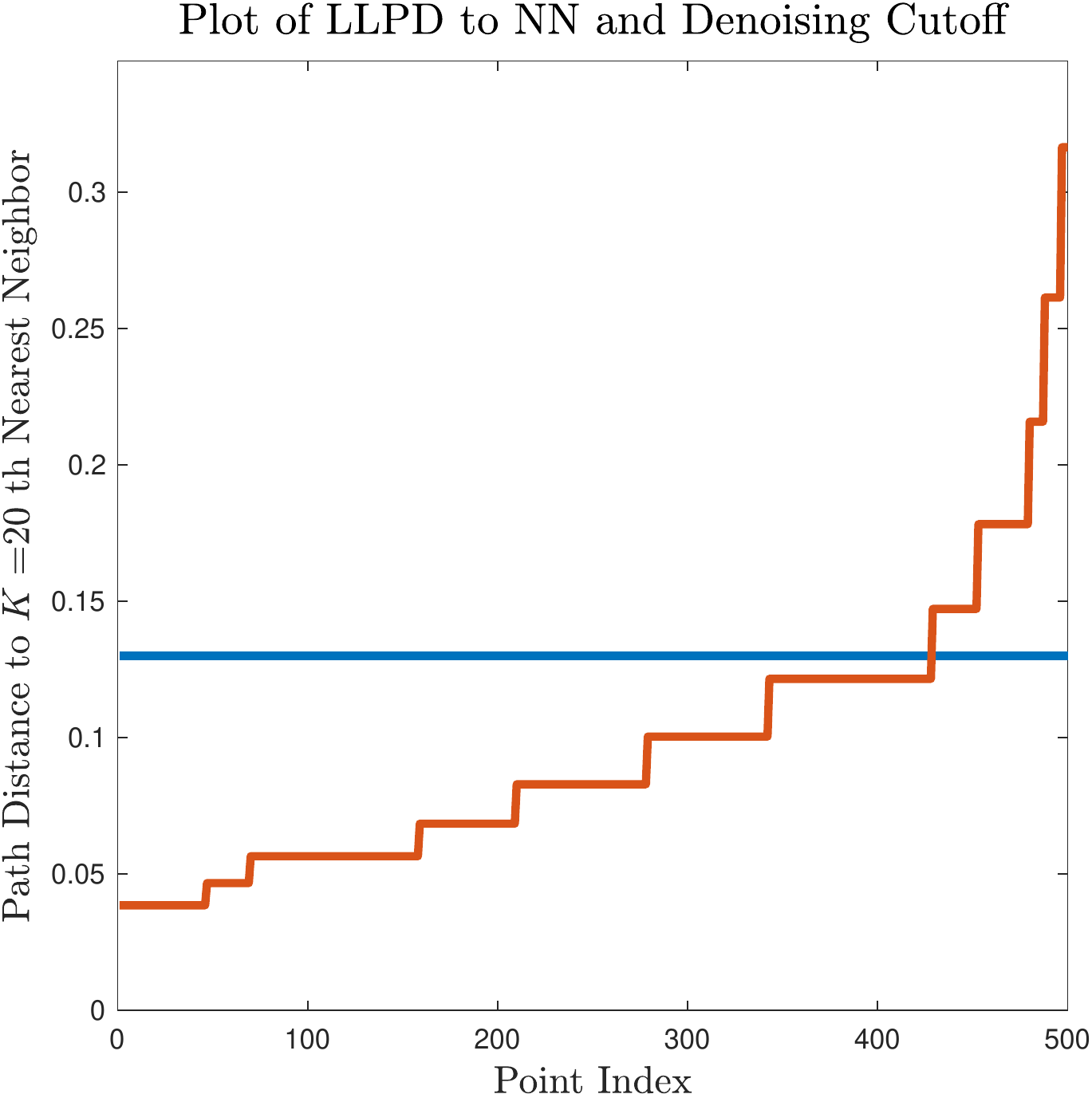}
		\subcaption{\label{fig:DenoisingNG}Nine Gaussians}
	\end{subfigure}
	\begin{subfigure}[t]{.24\textwidth}
		\captionsetup{width=\linewidth}
		\includegraphics[width=\textwidth]{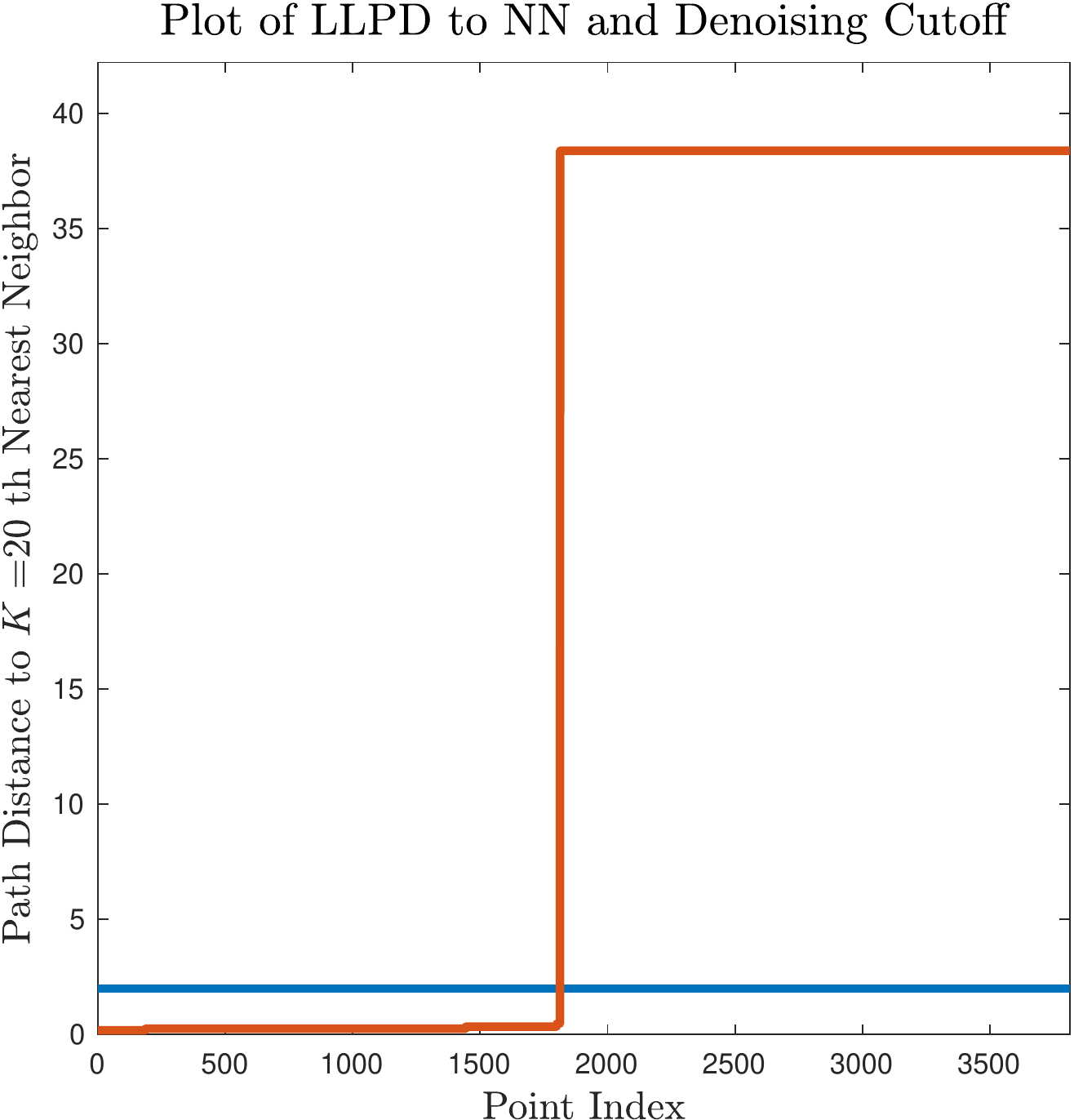}
		\subcaption{\label{fig:DenoisingCS}Concentric Spheres}
	\end{subfigure}	
	\begin{subfigure}[t]{.24\textwidth}
		\captionsetup{width=\linewidth}
		\includegraphics[width=\textwidth]{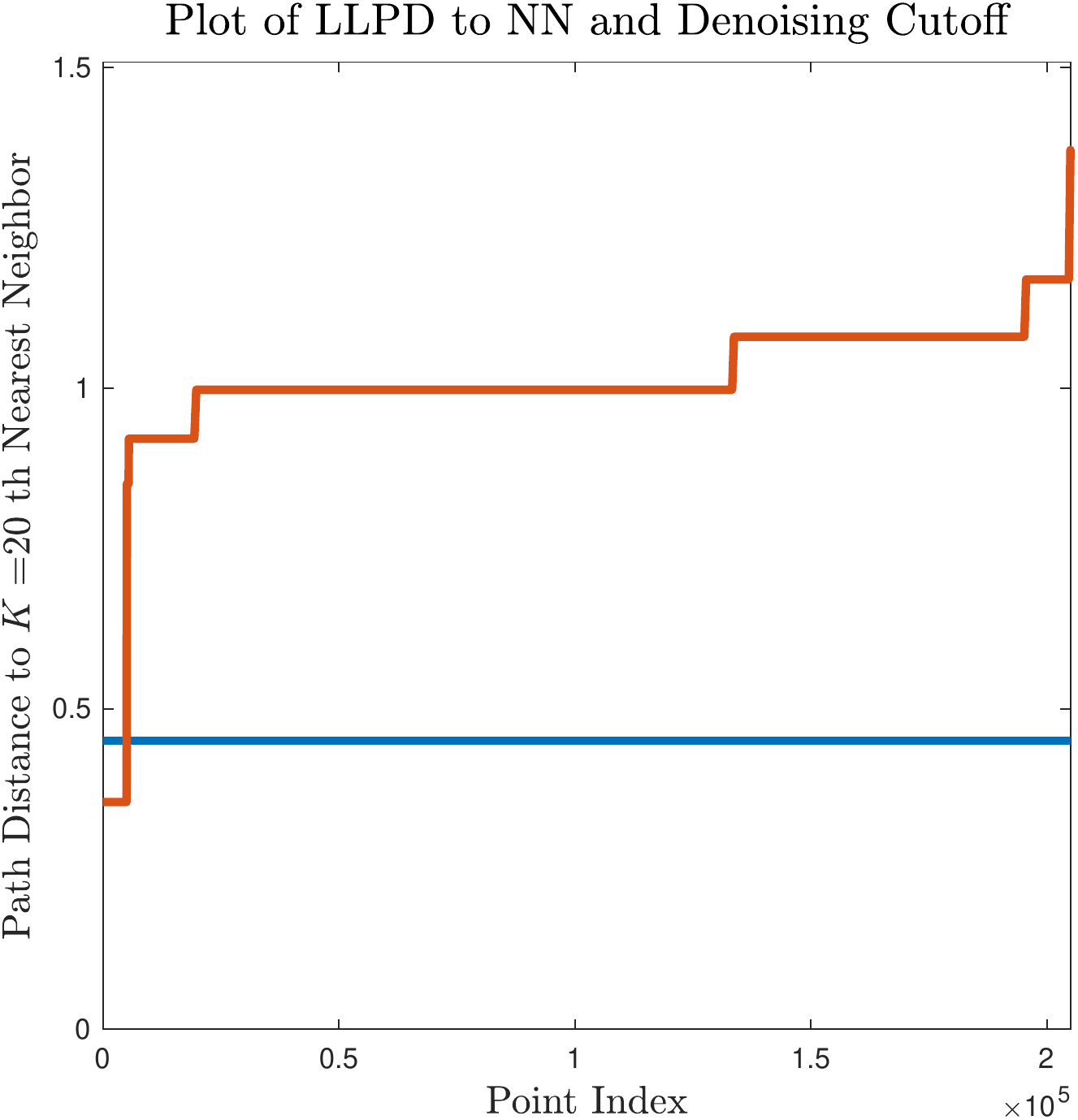}
		\subcaption{\label{fig:DenoisingPP}Parallel Planes}
	\end{subfigure}
	\caption{\label{fig:Denoising}LLPD to $\kNoise^{\!\!\!\!\!\text{th}}$ LLPD-nearest neighbor (blue) and threshold $\thetathres$ used for denoising the data (red).}
\end{figure}

Figure \ref{fig:MultiscaleEigs} shows the multiscale eigenvalue plots for the synthetic data sets. For the four lines data, Euclidean spectral clustering is run with $n=1160$ since it is prohibitively slow for $n=116000$; however all relevant proportions such as $\tilde{n}/n_{i}$ are the same.  LLPD spectral clustering correctly infers $\numclust$ for all synthetic data sets; Euclidean spectral clustering fails to correctly infers $\numclust$ except for the nine Gaussians example.  See Table \ref{tab:Results} for all $\hat{\numclust}$ values and empirical accuracies. Although accuracy is reported on the cluster points only, we remark that labels can be extended to any noise points which survive denoising by considering the label of the closest cluster set, and the empirical accuracies reported in Table \ref{tab:Results} remain essentially unchanged.

In addition to learning the number of clusters $\numclust$, the multiscale eigenvalue plots can also be used to infer a good scale $\sigma$ for LLPD spectral clustering as $\hat{\sigma} = \argmax_{\sigma} \left(\lambda_{\hat{\numclust}+1}(\sigma)-\lambda_{\hat{\numclust}}(\sigma)\right).$  For the two dimensional examples, the right panel of Figure \ref{fig:SyntheticDatasets} shows the results of LLPD spectral clustering with $\hat{\numclust}$, $\hat{\sigma}$ inferred from the maximal eigengap with LLPD.  Robustly estimating $\numclust$ and $\sigma$ makes LLPD spectral clustering essentially parameter free, and thus highly desirable for the analysis of real data.  

\begin{figure}[!htb]
\centering
\begin{subfigure}[t]{.24\textwidth}
\captionsetup{width=.95\linewidth}
\includegraphics[width=\textwidth]{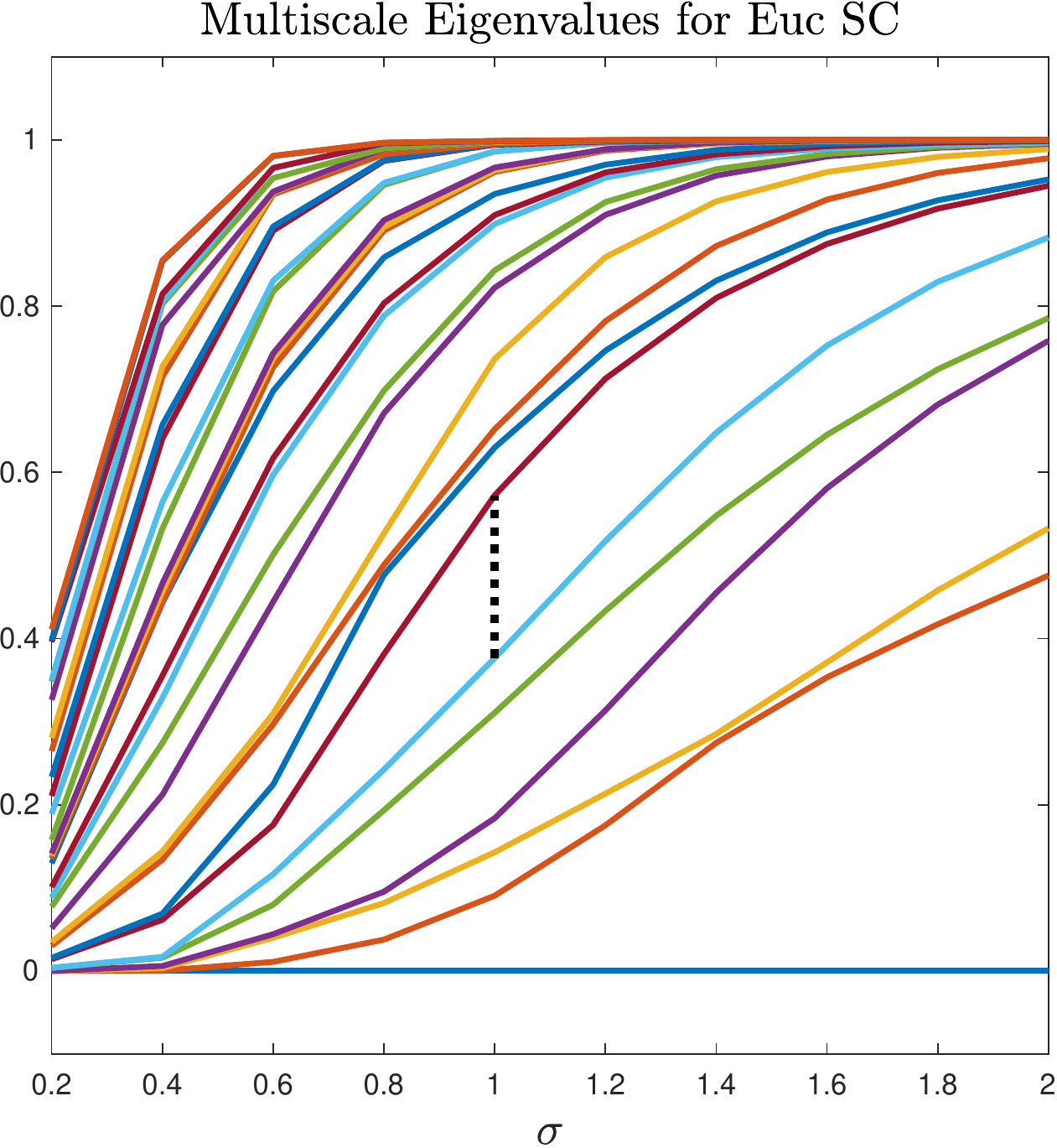}
\includegraphics[width=\textwidth]{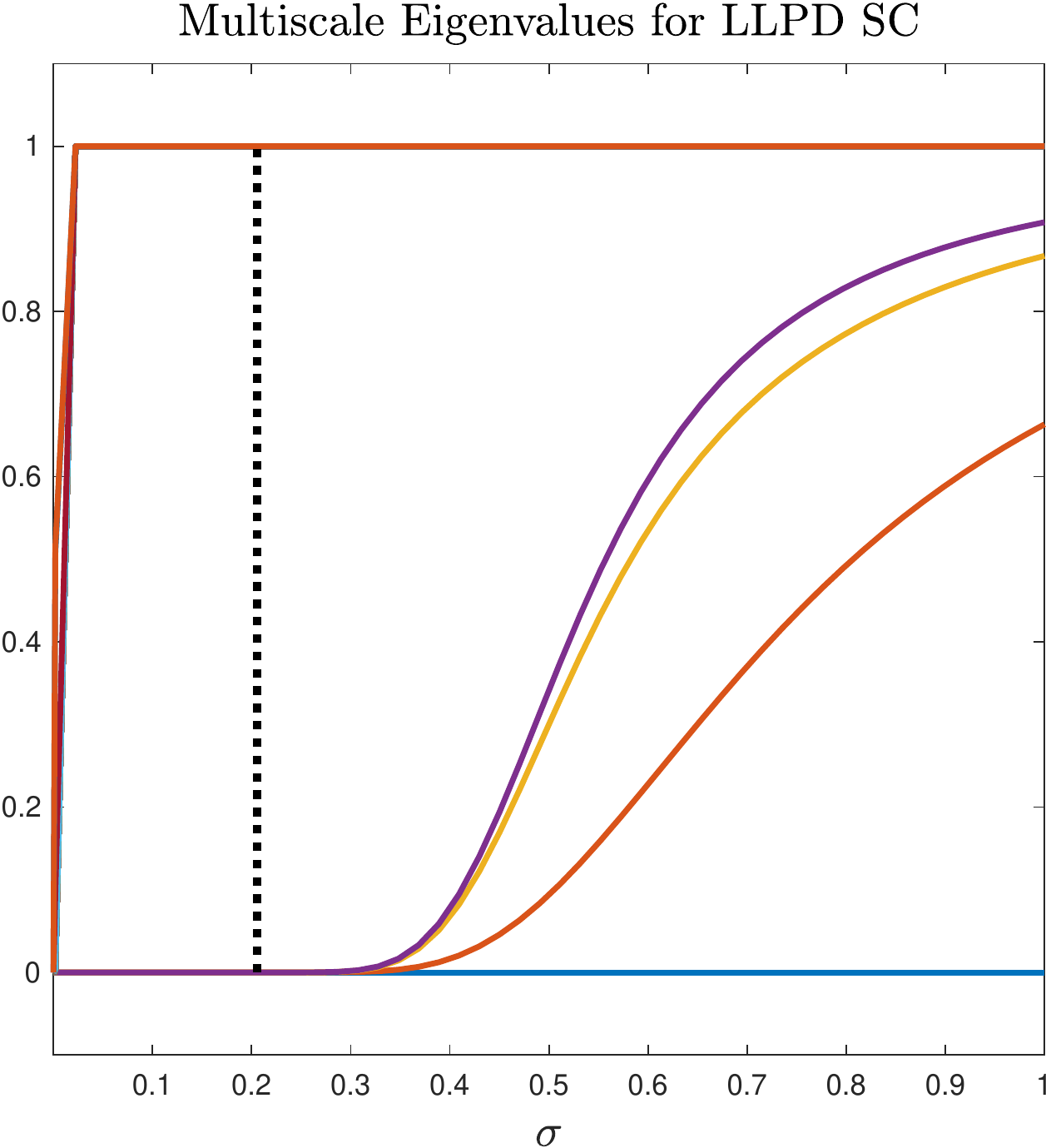}
\subcaption{Four Lines}
\end{subfigure}
\begin{subfigure}[t]{.24\textwidth}
\captionsetup{width=.95\linewidth}
\includegraphics[width=\textwidth]{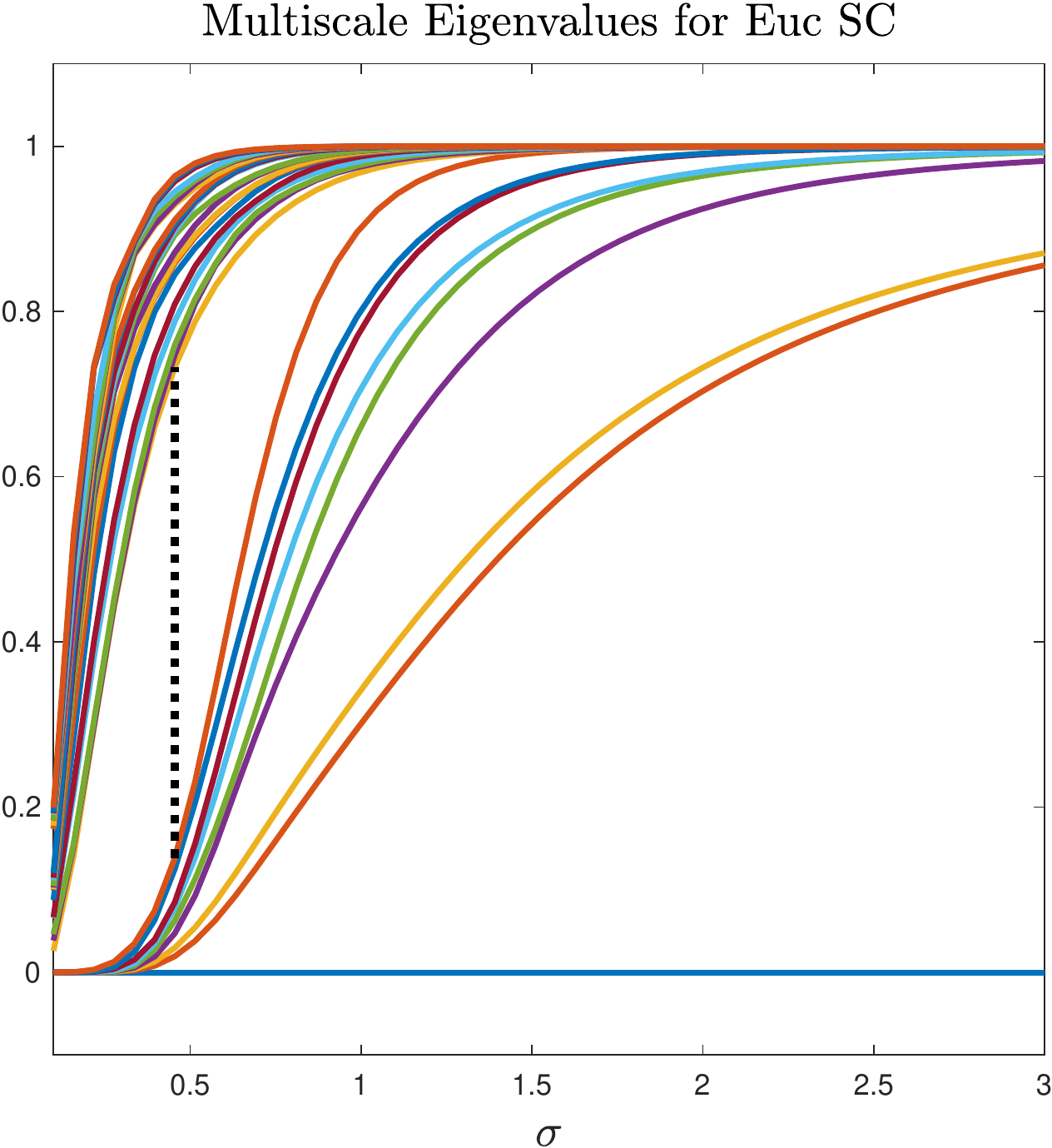}
\includegraphics[width=\textwidth]{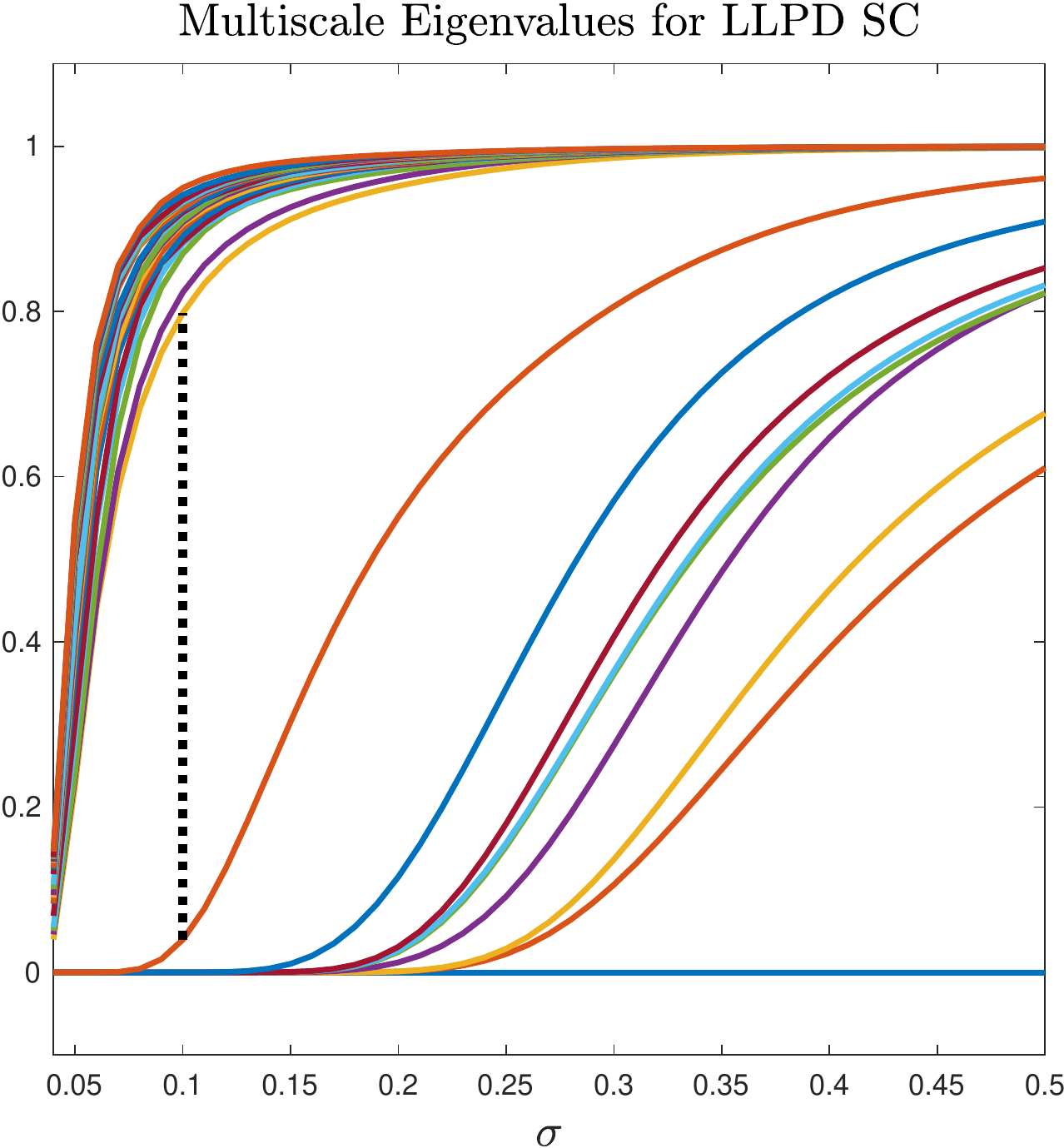}
\subcaption{Nine Gaussians}
\end{subfigure}
\begin{subfigure}[t]{.24\textwidth}
\captionsetup{width=.95\linewidth}
\includegraphics[width=\textwidth]{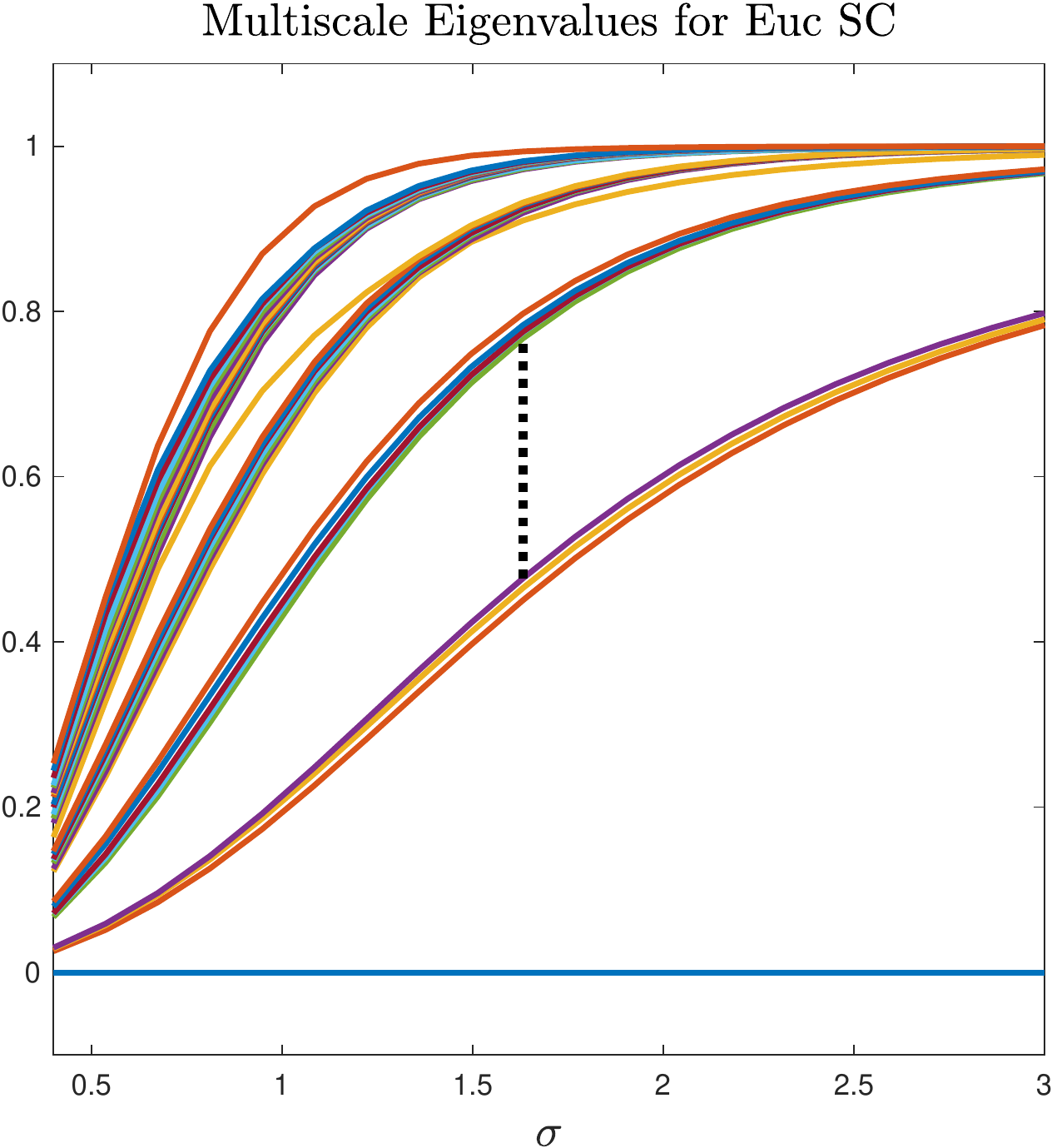}
\includegraphics[width=\textwidth]{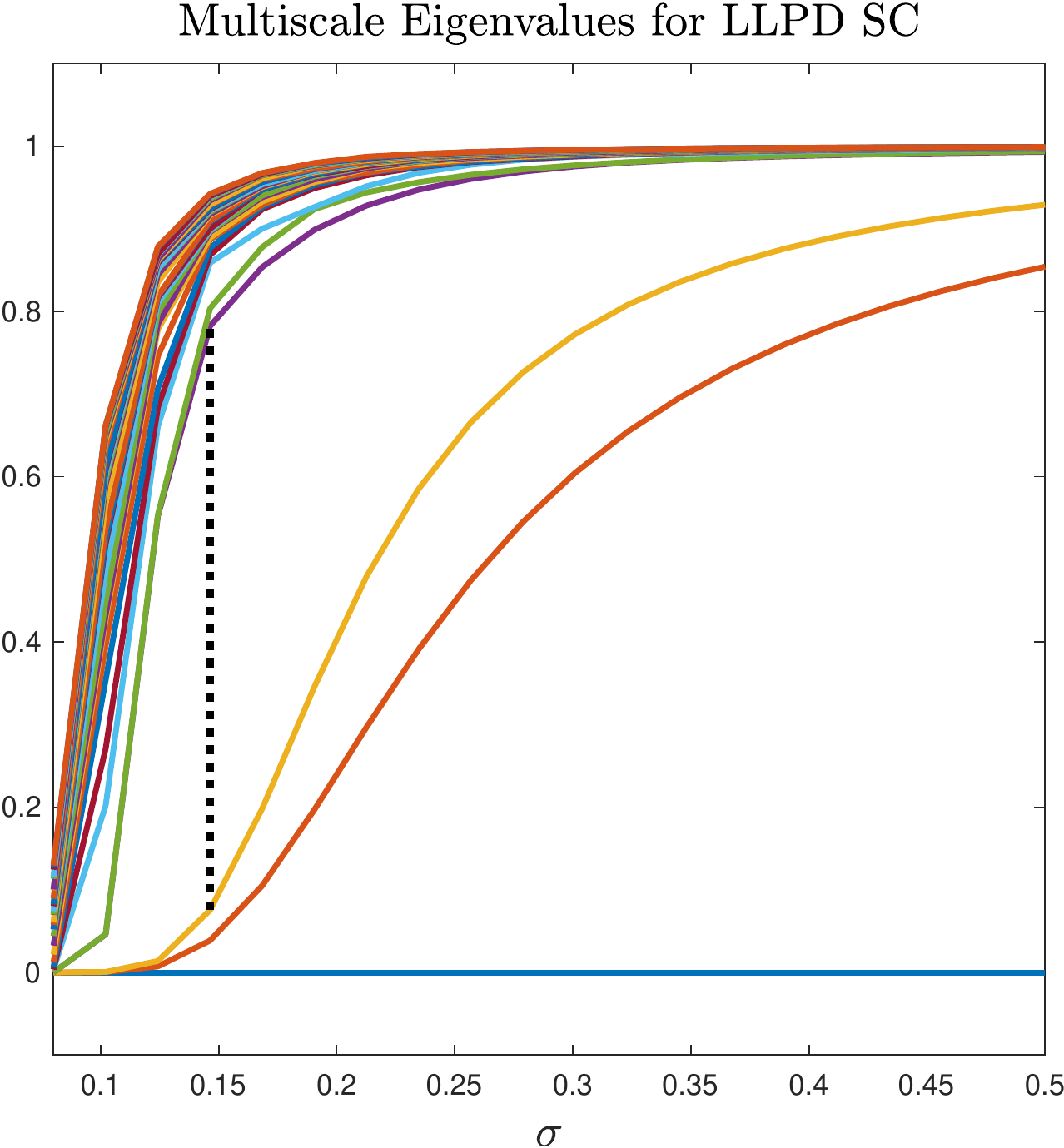}
\subcaption{Concentric Spheres}
\end{subfigure}
\begin{subfigure}[t]{.24\textwidth}
\captionsetup{width=.95\linewidth}
\includegraphics[width=\textwidth]{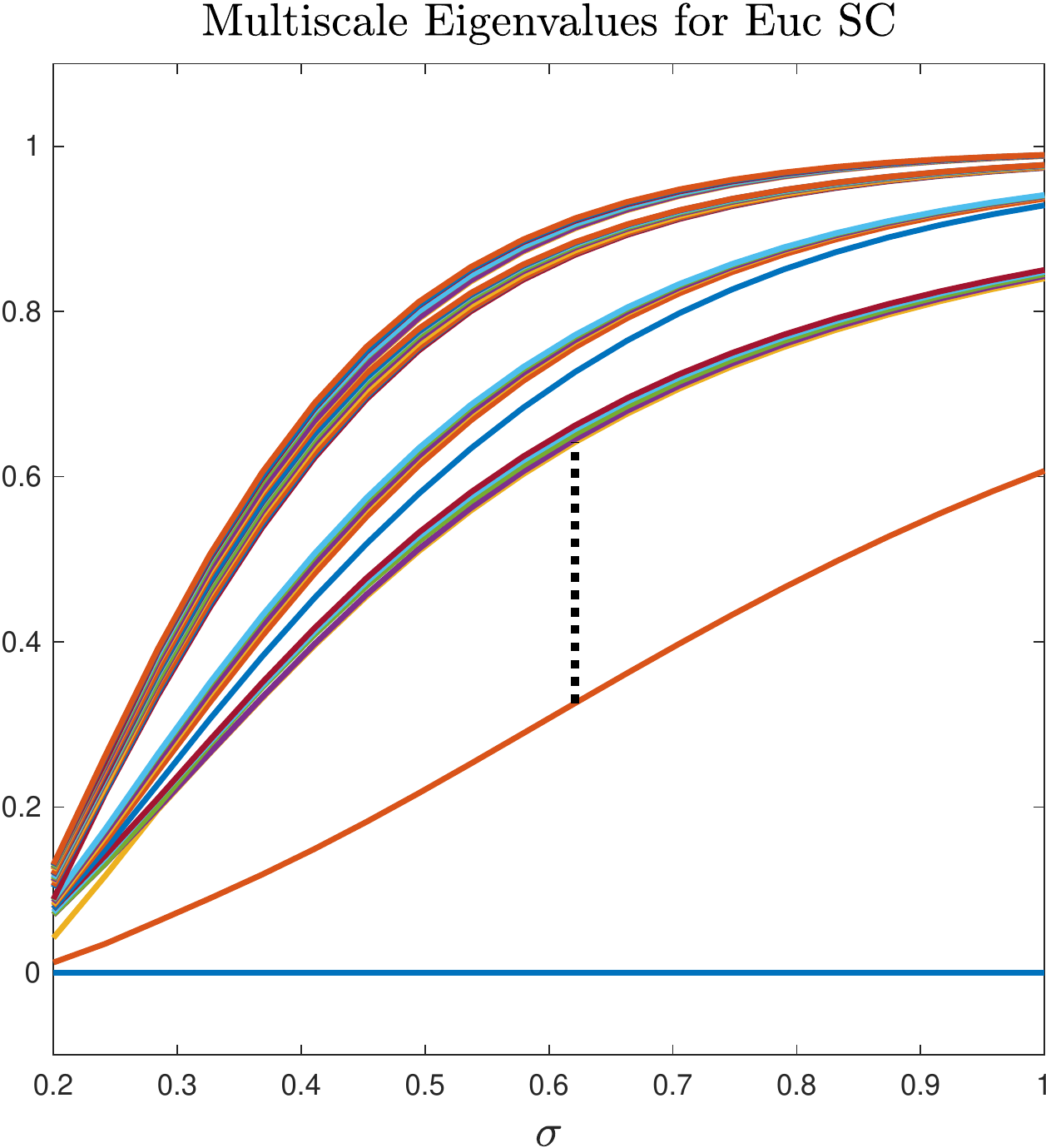}
\includegraphics[width=\textwidth]{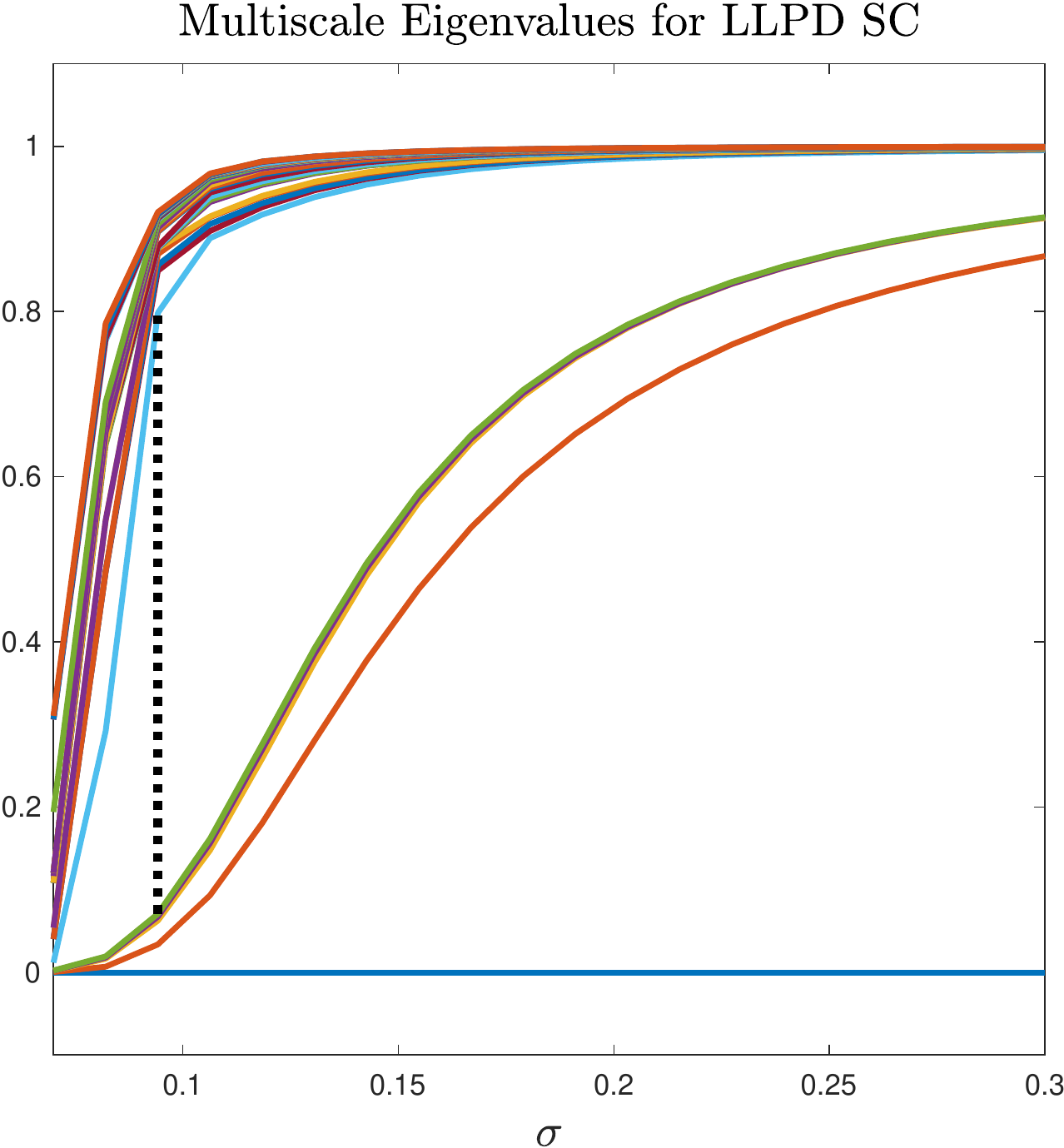}
\subcaption{Parallel Planes}
\end{subfigure}
\caption{\label{fig:MultiscaleEigs}	
Multiscale eigenvalues of $\Lsym$ for synthetic data sets using Euclidean distance (top) and LLPD (bottom).}
\end{figure}

\subsection{Real Data}

\begin{figure}[!htb]
	\begin{subfigure}[t]{.24\textwidth}
		\captionsetup{width=\linewidth}
		\includegraphics[width=\textwidth]{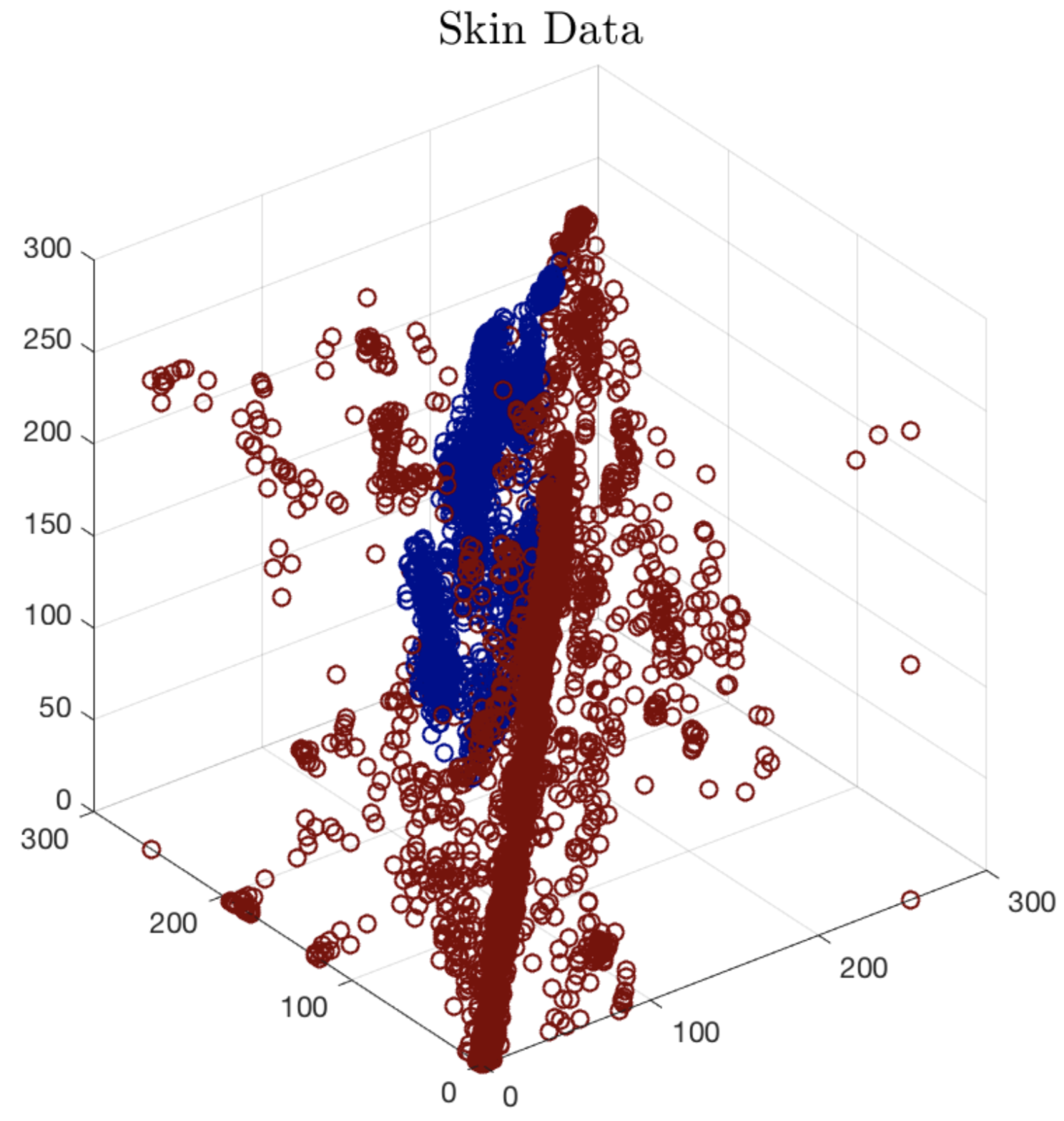}
		\subcaption{\label{fig:SkinsEmbedding}Skins data}
	\end{subfigure}	
	\begin{subfigure}[t]{.24\textwidth}
		\captionsetup{width=\linewidth}
		\includegraphics[width=\textwidth]{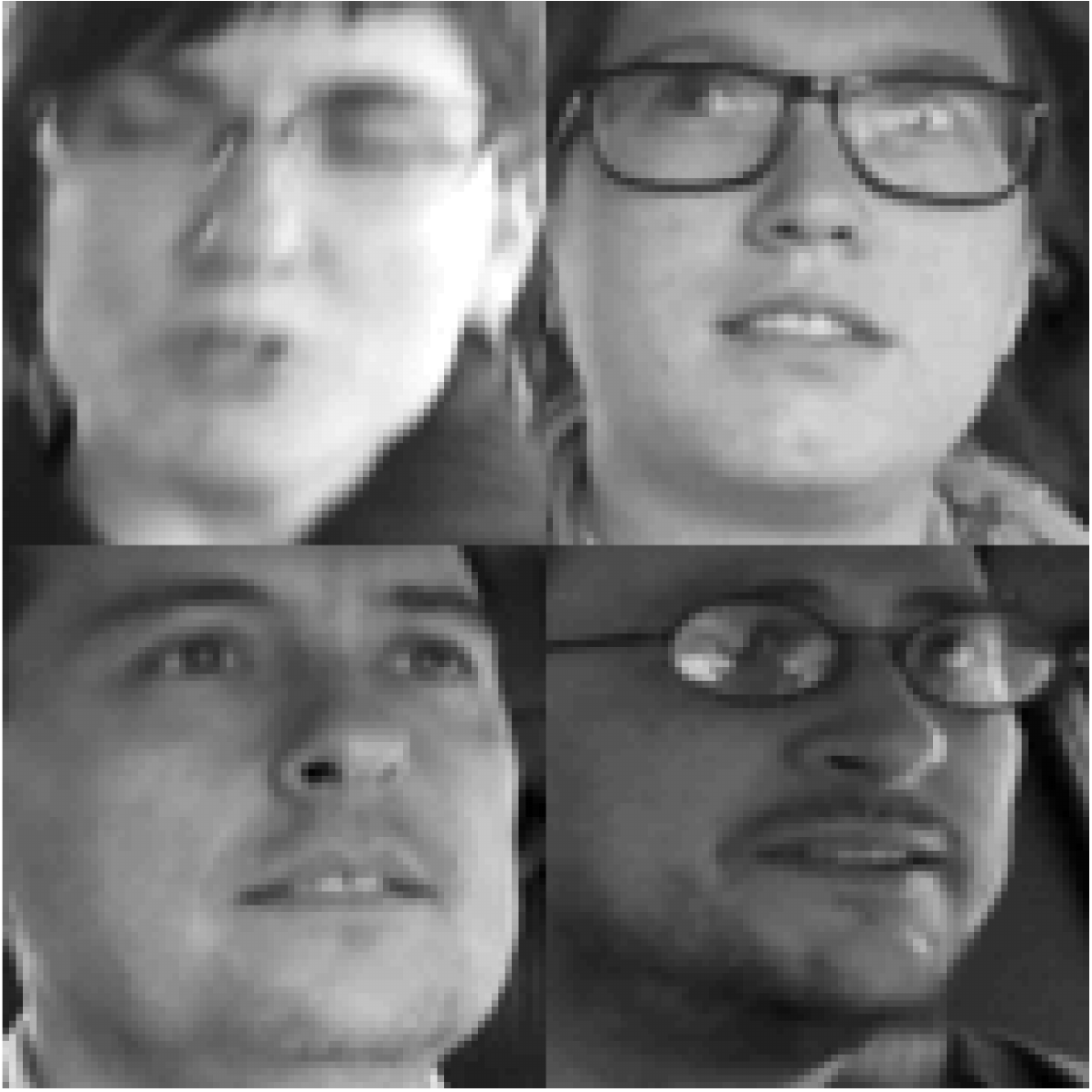}
		\subcaption{\label{fig:DrivFaceClasses}DrivFace Representative Faces}
	\end{subfigure}	
	\begin{subfigure}[t]{.24\textwidth}
		\captionsetup{width=\linewidth}
		\includegraphics[width=\textwidth]{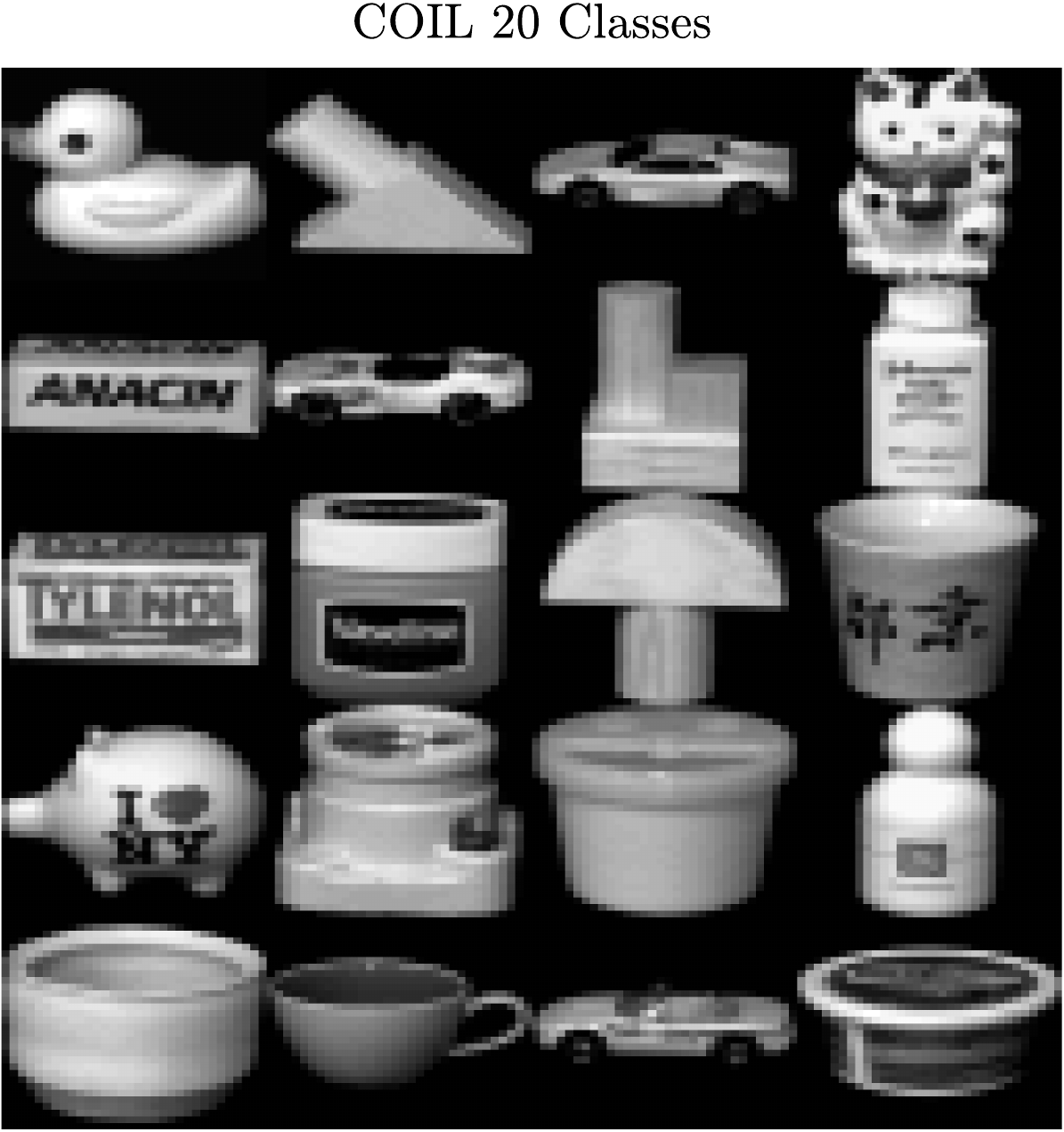}
		\subcaption{\label{fig:COILclasses}COIL objects}
	\end{subfigure}
	\begin{subfigure}[t]{.24\textwidth}
		\captionsetup{width=\linewidth}
		\includegraphics[width=\textwidth]{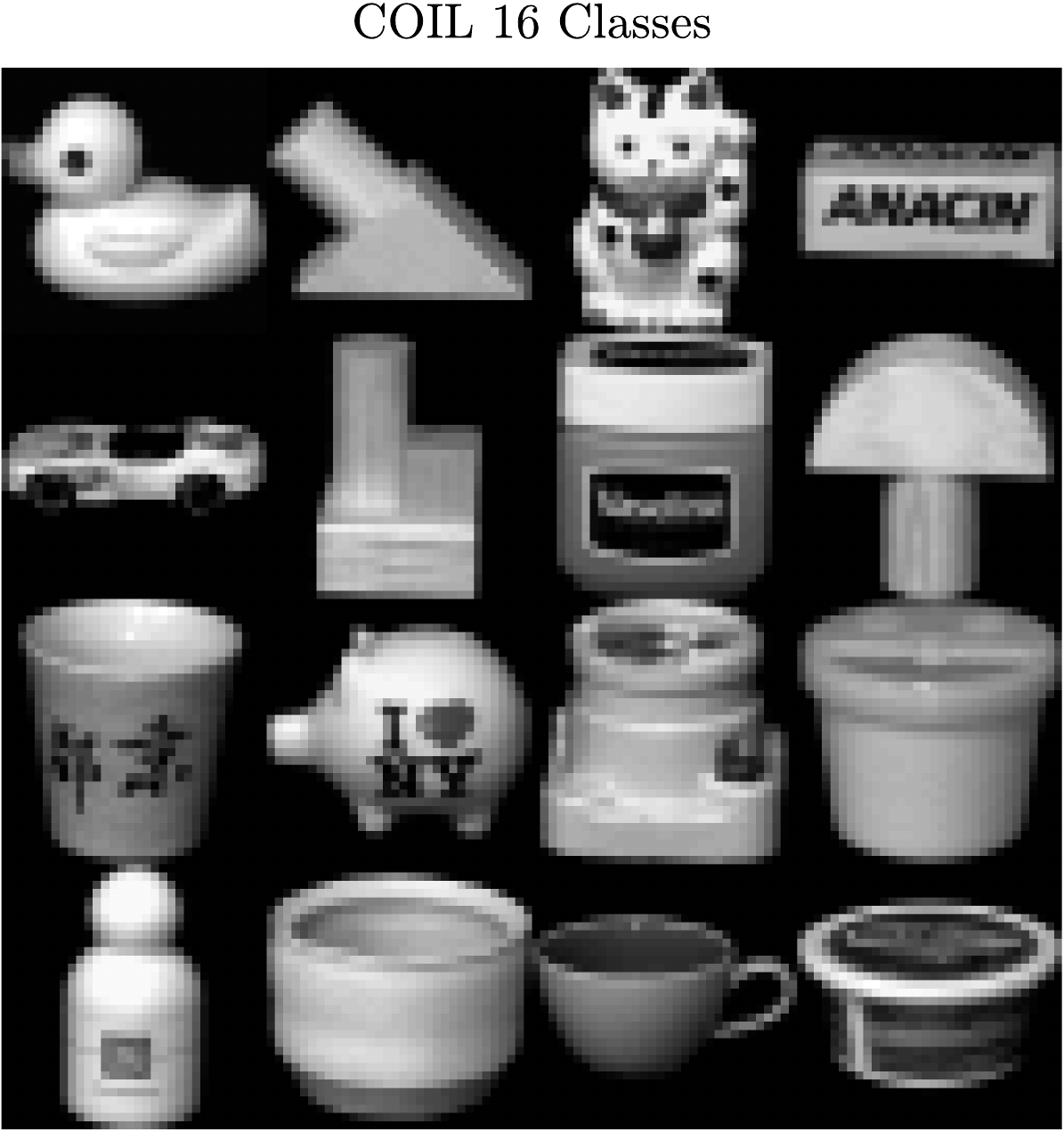}
		\subcaption{\label{fig:COIL16classes}COIL 16 objects}
	\end{subfigure}
	\caption{\label{fig:Miscellaneous}Representative objects from (a) Skins, (b) DrivFace, (c) COIL, and (d) COIL 16 data sets.}
\end{figure}

We apply our method on the following real data sets:
\begin{itemize}
	\setlength\itemsep{0em}

	\item \textbf{Skins} This large dataset consists of RGB values corresponding to pixels sampled from two classes: human skin and other\footnote{\url{https://archive.ics.uci.edu/ml/datasets/skin+segmentation}}.  The human skin samples are widely sampled with respect to age, gender, and skin color; see \cite{Bhatt2009_Efficient} for details on the construction of the dataset.  This dataset consists of 245057 data points in $D=3$ dimensions, corresponding to the RGB values. Note LLPD was approximated from scales $\{t_s\}_{s=1}^m$ defined by 10 percentiles, as opposed to the default exponential scaling.  See Figure \ref{fig:SkinsEmbedding}.
	\item \textbf{DrivFace} The DrivFace data set is publicly available\footnote{\url{https://archive.ics.uci.edu/ml/datasets/DrivFace}} from the UCI Machine Learning Repository \citep{Lichman:2013}. This data set consists of 606 $80\times 80$ pixel images of the faces of four drivers, 2 male and 2 female.  See Figure \ref{fig:DrivFaceClasses}
	\item \textbf{COIL} The COIL (Columbia University Image Library) dataset\footnote{\url{http://www.cs.columbia.edu/CAVE/software/softlib/coil-20.php}} consists of images of 20 different objects captured at varying angles \citep{Nene1996_Columbia}.  There are 1440 different data points, each of which is a $32\times 32$ image, thought of as a $D=1024$ dimensional point cloud. See Figure \ref{fig:COILclasses}.
	\item \textbf{COIL 16} To ease the problem slightly, we consider a 16 class subset of the full COIL data, shown in Figure \ref{fig:COIL16classes}. 
	\item \textbf{Pen Digits} This dataset\footnote{\url{https://archive.ics.uci.edu/ml/datasets/Pen-Based+Recognition+of+Handwritten+Digits}} consists of 3779 spatially resampled digital signals of hand-written digits in 16 dimensions \citep{Alimoglu1996_Methods}.  We consider a subset consisting of five digits: $\{0,2,3,4,6\}$.
	\item \textbf{Landsat} The landsat satellite data we consider consists of pixels in $3\times 3$ neighborhoods in a multispectral camera with four spectral bands\footnote{\url{https://archive.ics.uci.edu/ml/datasets/Statlog+(Landsat+Satellite)}}.  This leads to a total ambient dimension of $D=36$.  The data considered consists of $K=4$ classes, consisting of pixels of different physical materials: red soil, cotton, damp soil, and soil with vegetable stubble. 
\end{itemize}

\begin{figure}[!htb]
	\centering
	\begin{subfigure}[t]{.19\textwidth}
	\vspace{1pt}
		\captionsetup{width=\linewidth}
		\includegraphics[width=\textwidth]{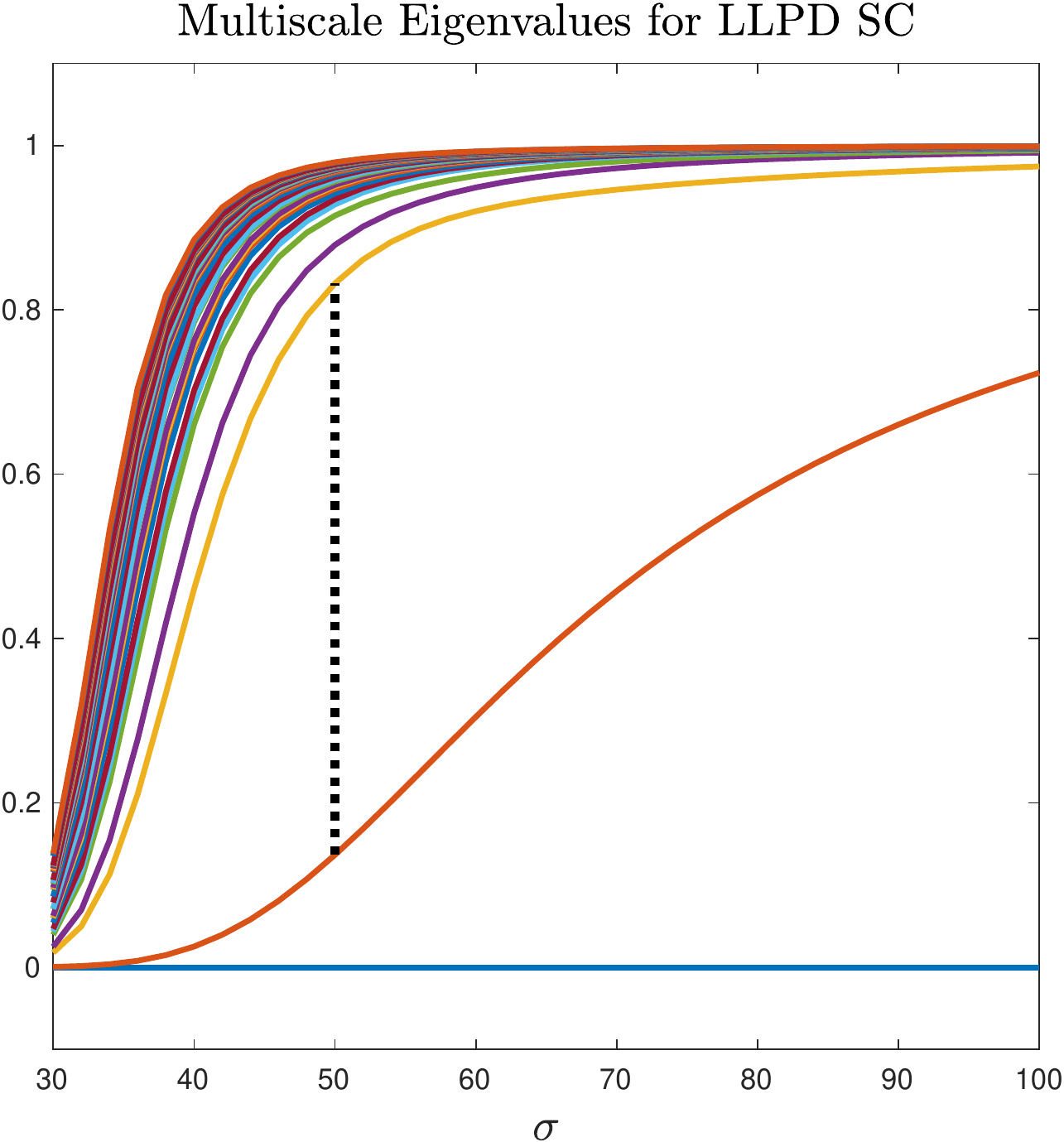}
		\subcaption{Skins: LLPD eigenvalues}
	\end{subfigure}
	\begin{subfigure}[t]{.19\textwidth}
		\captionsetup{width=.95\linewidth}
		\includegraphics[width=\textwidth]{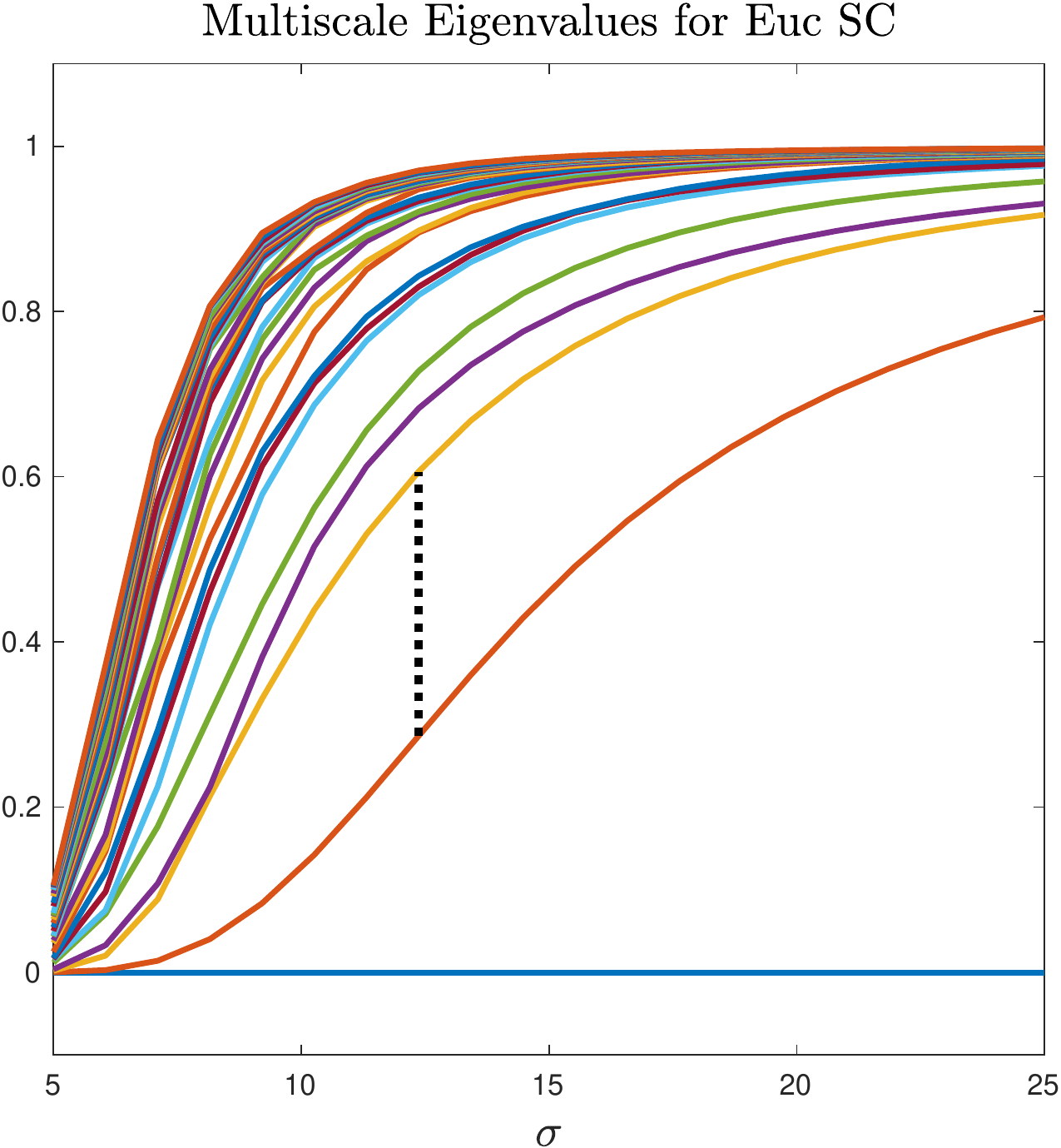}
		\includegraphics[width=\textwidth]{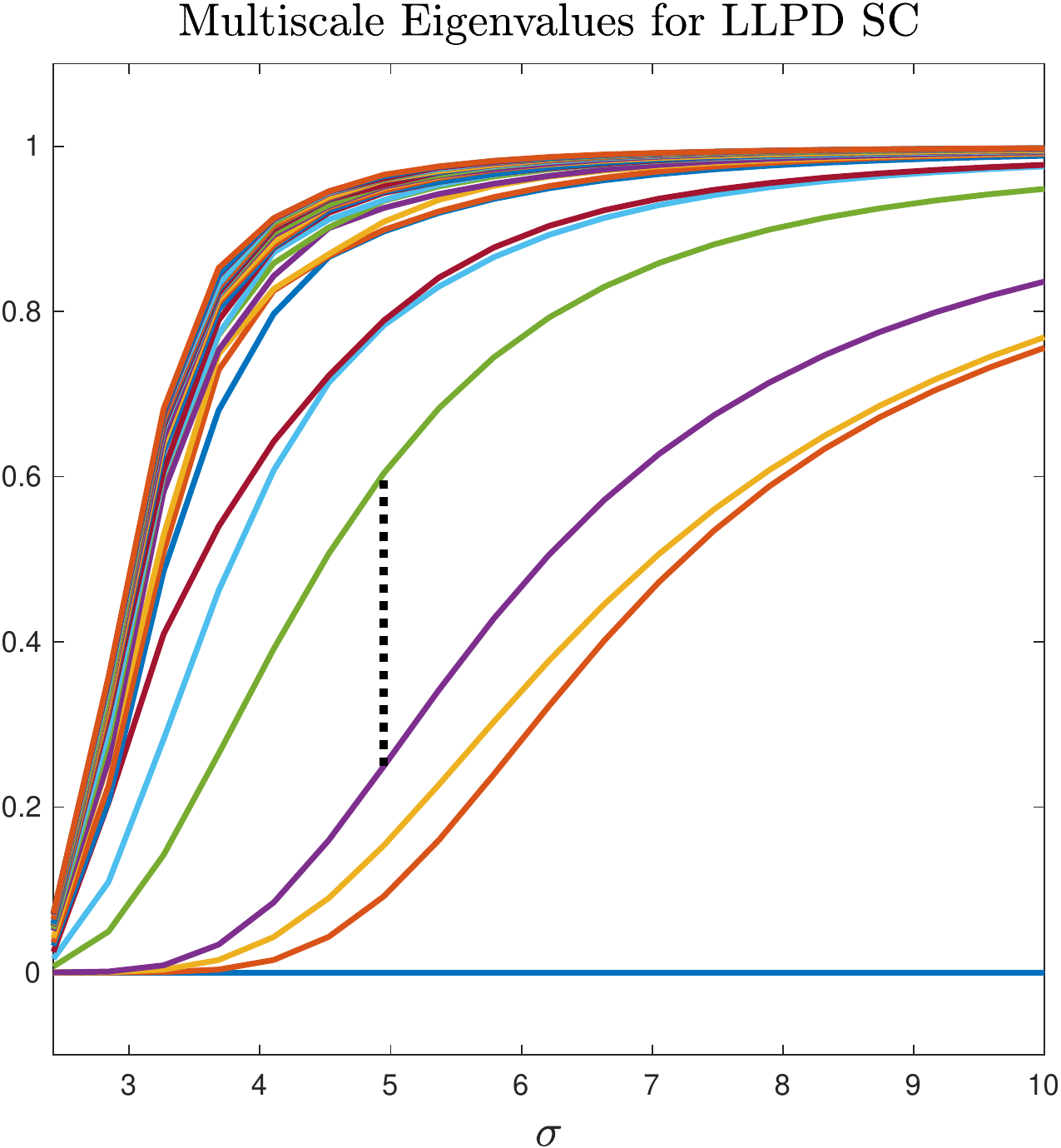}
		\subcaption{\label{fig:DrivFace}DrivFace}
	\end{subfigure}
	\begin{subfigure}[t]{.19\textwidth}
		\captionsetup{width=.95\linewidth}
		\includegraphics[width=\textwidth]{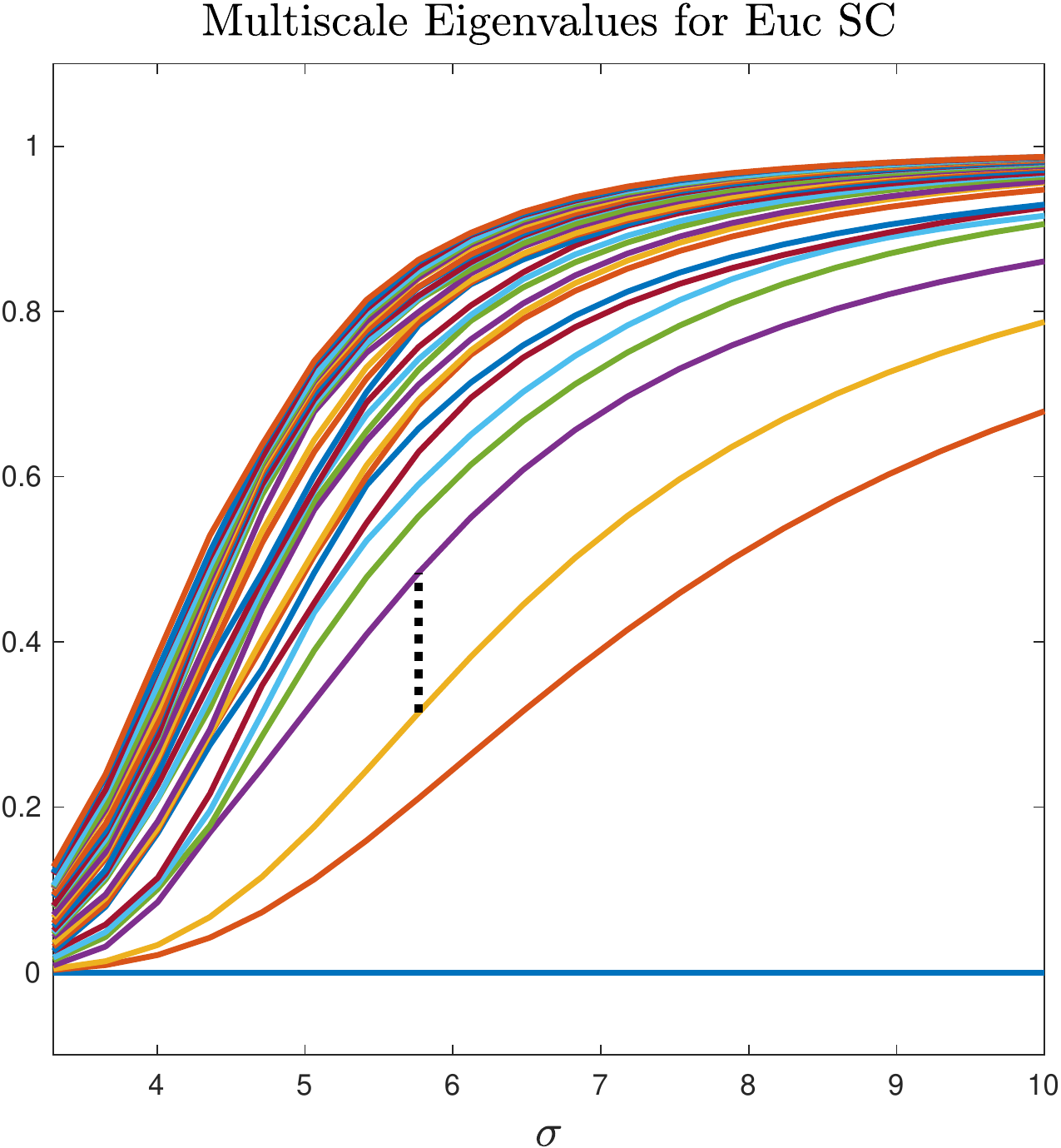}
		\includegraphics[width=\textwidth]{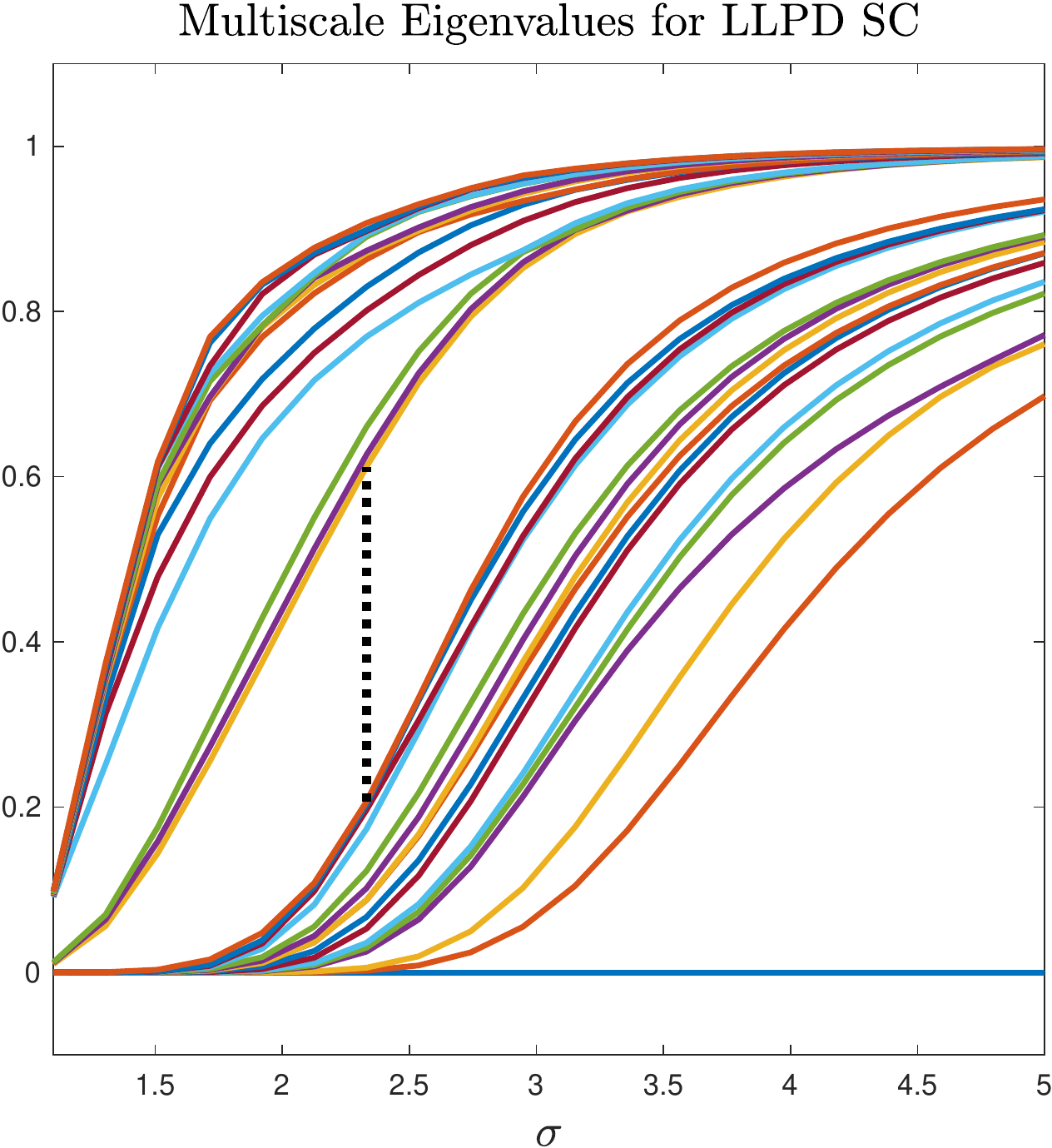}
		\subcaption{\label{fig:COIL16}COIL 16}
	\end{subfigure}
	\begin{subfigure}[t]{.19\textwidth}
		\captionsetup{width=.95\linewidth}
		\includegraphics[width=\textwidth]{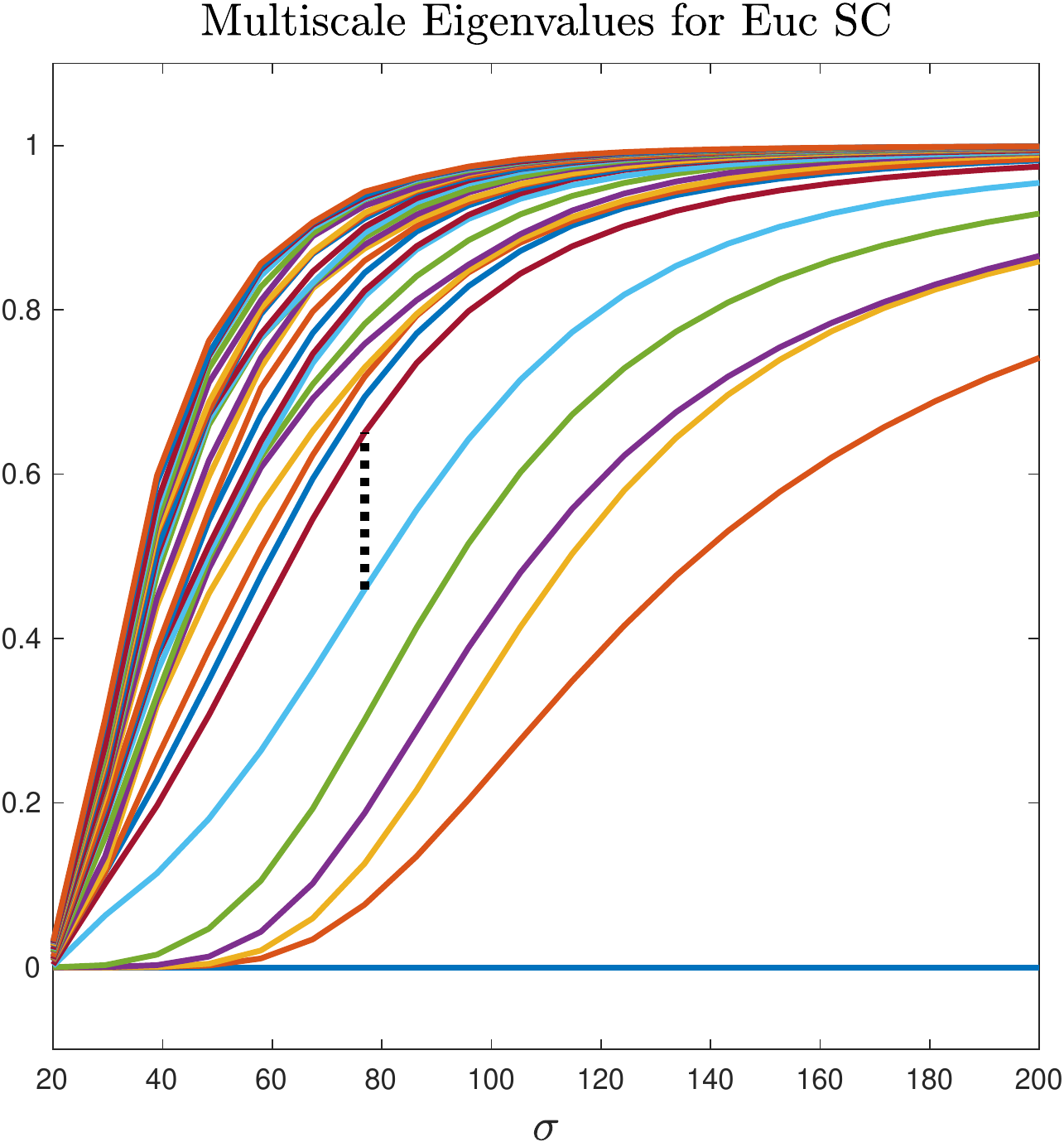}
		\includegraphics[width=\textwidth]{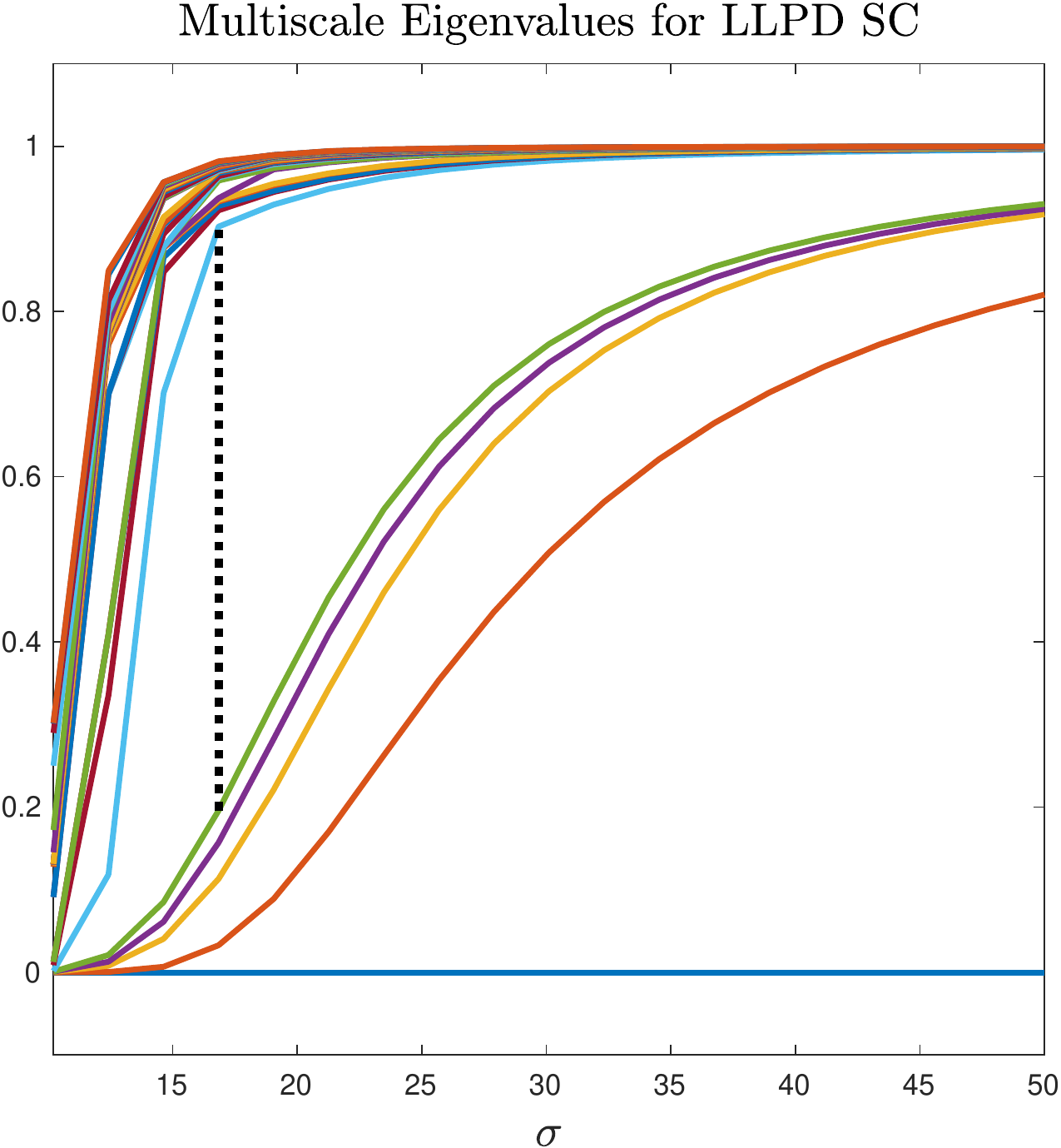}
		\subcaption[width=.8\textwidth]{\label{fig:PenDigits}Pen Digits}
	\end{subfigure}
	\begin{subfigure}[t]{.19\textwidth}
		\captionsetup{width=.95\linewidth}
		\includegraphics[width=\textwidth]{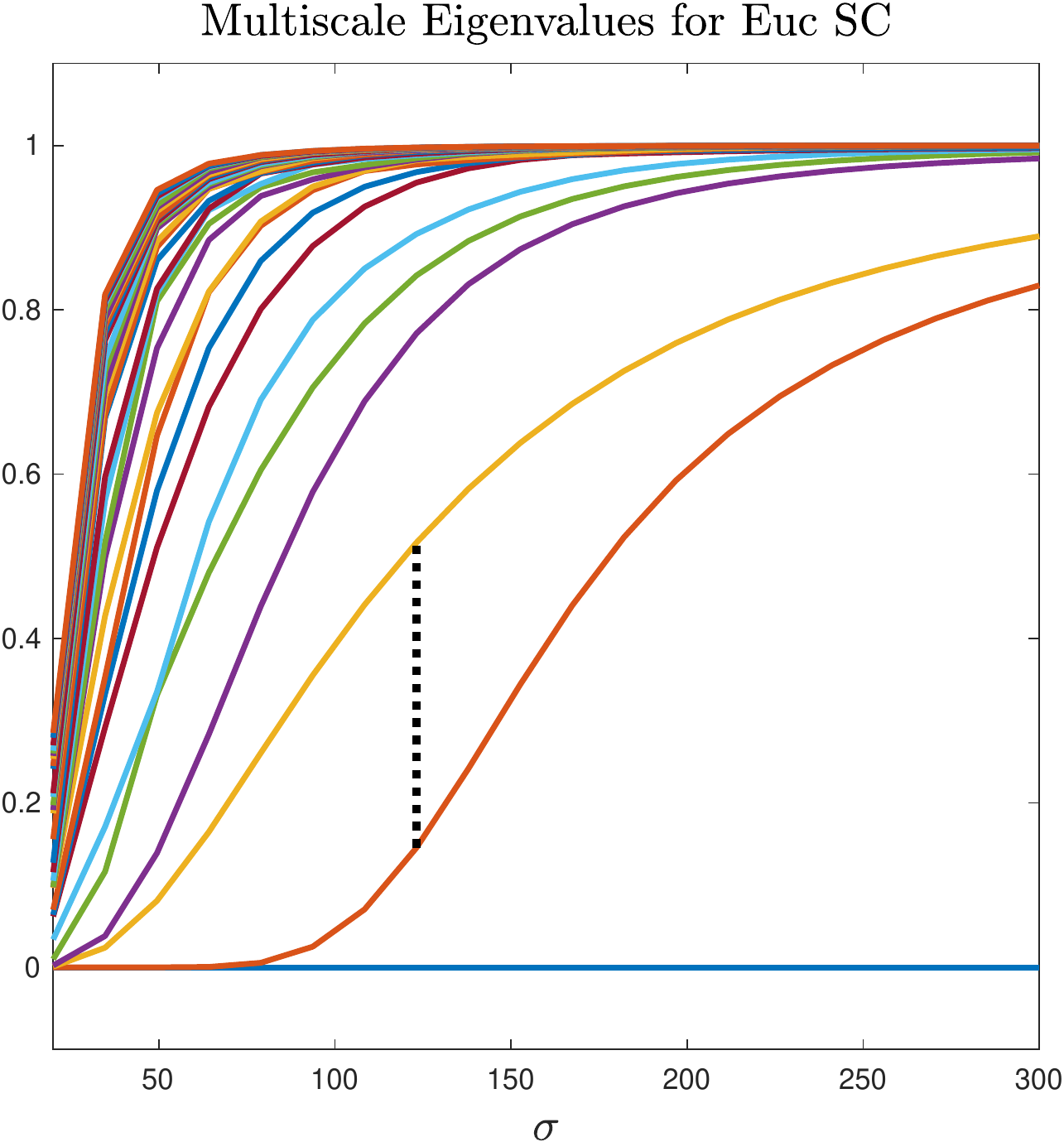}
		\includegraphics[width=\textwidth]{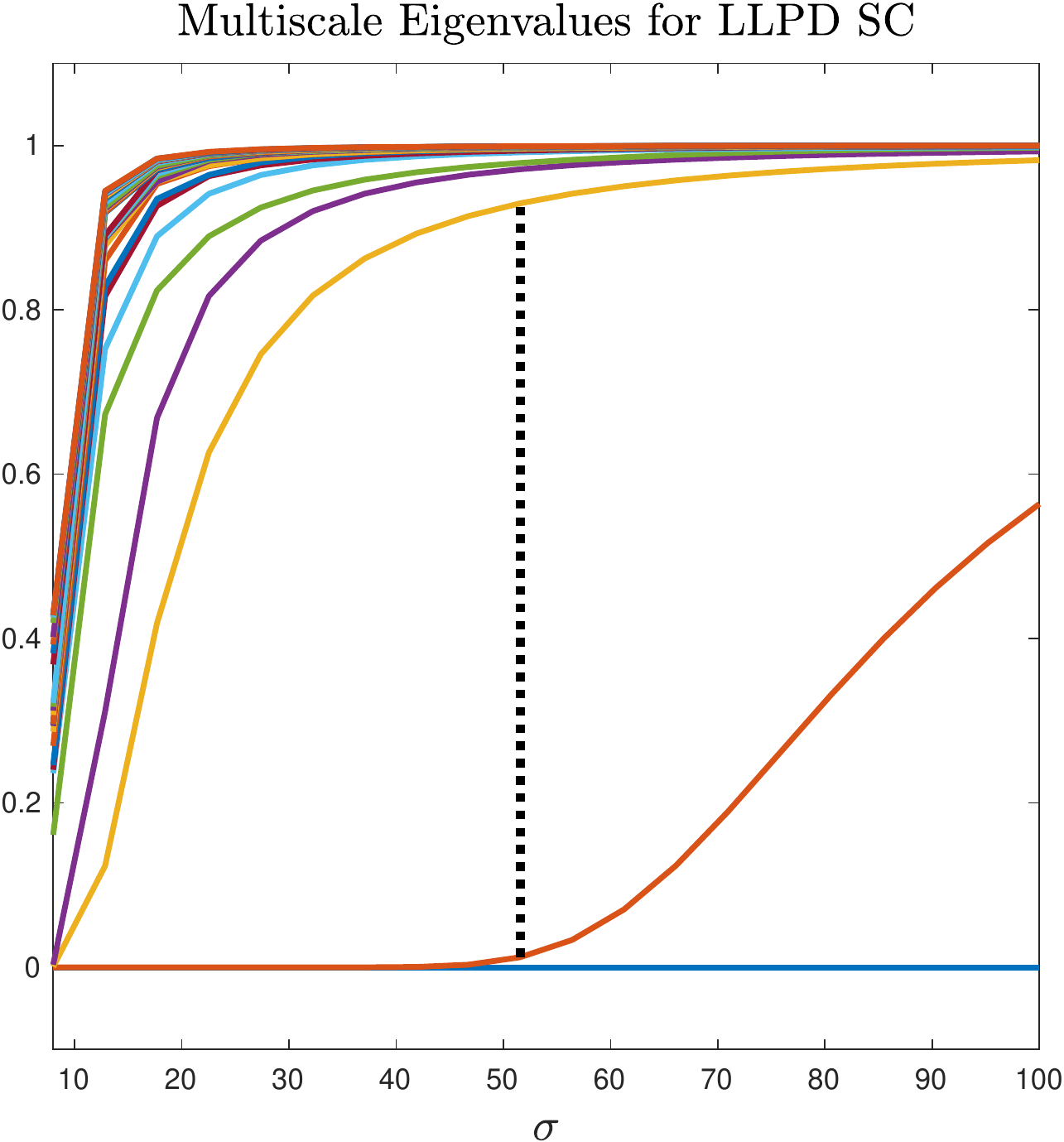}
		\subcaption{\label{fig:Landsat}Landsat}
	\end{subfigure}
	\caption{\label{fig:MultiscaleEigsRealData}	
		Multiscale eigenvalues of $\Lsym$ for real data sets using Euclidean distance (top, (b)-(e), does not appear for (a)) and LLPD (bottom (a)-(e)).}
\end{figure}

Labeling accuracy results as well as the $\hat{K}$ values returned by our algorithm are given in Table \ref{tab:Results}. LLPD spectral clustering correctly estimates $\numclust$ for all data sets except the full COIL data set and Landsat. Euclidean spectral clustering fails to correctly detect $\numclust$ on all real data sets.  Figure \ref{fig:MultiscaleEigsRealData} shows both the Euclidean and LLPD eigenvalues for Skins, DrivFace, COIL 16, Pen Digits, and Landsat. Euclidean spectral clustering results for Skins are omitted because Euclidean spectral clustering with a dense Laplacian is computationally intractable with such a large sample size.  

At least $90\%$ of data points were retained during the denoising procedure with the exception of Skins ($88.0\%$ retained) and Landsat ($67.2\%$ retained).
After denoising, LLPD spectral clustering achieved an overall accuracy exceeding $98.6\%$ on all real data sets except COIL 20 ($90.5\%$). Euclidean spectral clustering performed well on DrivFaces (OA $94.1\%$) and Pen Digits (OA $98.1\%$), but poorly on the remaining data sets, where the overall accuracy ranged from $68.9\%-76.8\%$. $K$-means also performed well on DrivFaces (OA $87.5\%$) and Pen Digits (OA $97.6\%$) but poorly on the remaining data sets, where the overall accuracy ranged from $54.7\% - 78.5\%$.
 
\begin{table}
\begin{center}
\begin{adjustbox}{max width=1.00\textwidth}

\begin{tabular}{| c | c | c | c | c |}
\hline
Data Set & Accuracy Statistic & $\numclust$-means & Euclidean SC & LLPD SC \\ \hline
\multirow{4}{*}{\parbox{9cm}{\centering \textbf{Four Lines} \\($n=116000, \tilde{n}=20000, N=97361, D=2, d=1, \numclust=4, \zeta_{N}=12.0239, \theta=.01, \hat{\sigma}=0.2057, \delta=.9001$)}} & OA & .4951 & .6838 & \textbf{1.000} \\
 \cline{2-5}
 & AA &  .4944 & .6995 & \textbf{1.000} \\ 
 \cline{2-5}
 & $\kappa$ & .3275 & .5821 & \textbf{1.000} \\ 
  \cline{2-5}
  & $\hat{\numclust}$ & - & 6 & \textbf{4} \\ 
  \hline
  
  \multirow{4}{*}{\parbox{9cm}{\centering \textbf{Nine Gaussians} \\($n=500, \tilde{n}=0, N=428, D=2, d=2, \numclust=9, \zeta_{N}=10.7, \theta=.13, \hat{\sigma}=0.1000, \delta=.1094$)}} & OA & \textbf{.9930}  & \textbf{.9930} &  \textbf{.9930} \\
 \cline{2-5}
 & AA & \textbf{.9920}  & \textbf{.9920} & \textbf{.9920} \\ 
 \cline{2-5}
 & $\kappa$ & \textbf{.9921} & \textbf{.9921} & \textbf{.9921} \\ 
  \cline{2-5}
  & $\hat{\numclust}$ & - &  \textbf{9} &  \textbf{9} \\ 
  \hline
  
    \multirow{4}{*}{\parbox{9cm}{\centering \textbf{Concentric Spheres} \\($n=3813, \tilde{n}=2000, N=1813, D=1000, d=2, \numclust=3, \zeta_{N}=7.2520, \theta=2, \hat{\sigma}=0.1463, \delta=.5$)}} & OA & .3464 & .3519 & \textbf{.9989} \\
 \cline{2-5}
 & AA & .3463 & .3438 & \textbf{.9988} \\ 
 \cline{2-5}
 & $\kappa$ & .0094 & .0155 & \textbf{.9981}\\ 
  \cline{2-5}
  & $\hat{\numclust}$ & - &  4 & \textbf{3} \\ 
  \hline

\multirow{4}{*}{\parbox{9cm}{\centering \textbf{Parallel Planes}\\ ($n=205000,  \tilde{n}=200000, N=5000,  D=30, d=10, \numclust=5, \zeta_{N}=5, \theta=.45, \hat{\sigma}=0.0942, \delta=.3553$)}} & OA & .5594 & .3964  & \textbf{.9990} \\
 \cline{2-5}
 & AA & .5594 & .3964 & \textbf{.9990} \\ 
 \cline{2-5}
 & $\kappa$ & .4493 & .2455 & \textbf{.9987} \\ 
  \cline{2-5}
  & $\hat{\numclust}$ & - & 2 & \textbf{5} \\ 
  \hline

 \multirow{4}{*}{\parbox{9cm}{\centering \textbf{Skins} \\($n=245057, N=215694, D=3, \numclust=2, \zeta_{N}=4.5343, \theta=2, \hat{\sigma}=50, \hat{\delta}=0$)}} & OA & .5473 & - & \textbf{.9962}\\
 \cline{2-5}
 & AA &  .4051 & - & \textbf{.9970} \\ 
 \cline{2-5}
 & $\kappa$ &  -.1683 & - & \textbf{.9890} \\ 
 \cline{2-5}
 & $\hat{\numclust}$ & - & - & \textbf{2} \\ 
 \hline

\multirow{4}{*}{\parbox{9cm}{\centering \textbf{DrivFaces}\\ ($n=612, N=574, D=6400, \numclust=4, \zeta_{N}=6.4494, \theta=10,  \hat{\sigma}=4.9474, \hat{\delta}=9.5976$)}} & OA &  .8746 & .9408  & \textbf{1.000} \\
 \cline{2-5}
 & AA & .8882  & .9476 & \textbf{1.000} \\ 
 \cline{2-5}
 & $\kappa$ & .9198  & .9198 & \textbf{1.000} \\ 
  \cline{2-5}
  & $\hat{\numclust}$ & - & 2 & \textbf{4} \\ 
  \hline
  
  \multirow{4}{*}{\parbox{9cm}{\centering \textbf{COIL 20}\\ ($n=1440, N=1351, D=1024, \numclust=20, \zeta_{N}=27.5714, \theta=4.5, \hat{\sigma}=1.9211, \hat{\delta}=3.3706$)}} & OA & .6555 & .6890 & \textbf{.9055}  \\
 \cline{2-5}
 & AA &.6290  & .6726 &  \textbf{.8833}\\ 
 \cline{2-5}
 & $\kappa$ & .6368 & .6724 &  \textbf{.9004}\\ 
  \cline{2-5}
  & $\hat{\numclust}$ & - & 3 & 17 \\ 
  \hline
  
    \multirow{4}{*}{\parbox{9cm}{\centering \textbf{COIL 16} \\ ($n=1152, N=1088, D=1024, \numclust=16, \zeta_{N}=22.1837, \theta=3.9, \hat{\sigma}=2.3316, \hat{\delta}=5.4350$) }} & OA & .7500 & .7330 & \textbf{1.000} \\
 \cline{2-5}
 & AA & .7311 & .6782 & \textbf{1.000} \\ 
 \cline{2-5}
 & $\kappa$ & .7330 & .6864 & \textbf{1.000} \\ 
  \cline{2-5}
  & $\hat{\numclust}$ & - & 3 & \textbf{16} \\ 
  \hline
  
    \multirow{4}{*}{\parbox{9cm}{\centering \textbf{Pen Digits} \\($n=3779, N=3750, D=16, \numclust=5, \zeta_{N}=5.2228 , \theta=60, \hat{\sigma}=16.8421, \hat{\delta}=11.3137$)}} & OA &  .9760 & .9813 &  \textbf{.9949} \\
 \cline{2-5}
 & AA &  .9764 & .9816   &  \textbf{.9949} \\ 
 \cline{2-5}
 & $\kappa$ & .9700 & .9767 & \textbf{.9937} \\ 
  \cline{2-5}
  & $\hat{\numclust}$ & - & 6 & \textbf{5} \\ 
  \hline
  
    \multirow{4}{*}{\parbox{9cm}{\centering \textbf{Landsat}\\ ($n=1136, N=763, D=36, \numclust=4, \zeta_{N}=8.4778, \theta=32, \hat{\sigma}=51.5789,  \hat{\delta}=28.2489$)}} & OA & .7851 & .7680 & \textbf{.9869}\\
 \cline{2-5}
 & AA & .8619 & .8532 & \textbf{.9722}  \\ 
 \cline{2-5}
 & $\kappa$ & .6953 &  .6722 & \textbf{.9802} \\ 
  \cline{2-5}
  & $\hat{\numclust}$ & - & 2 & 2 \\ 
  \hline
  
\end{tabular}
\end{adjustbox}

\end{center}
\caption{ \label{tab:Results}  In all examples, LLPD spectral clustering performs at least as well as $K$-means and Euclidean spectral clustering, and it typically outperforms both.  Best results for each method and performance metric are bolded.  For each dataset, we include parameters that determine the theoretical results.  For both real and synthetic datasets, $n$ (the total number of data points), $N$ (the number of data points after denoising), $D$ (the ambient dimension of the data), $\numclust$ (the number of clusters in the data), $\zeta_{N}$ (cluster balance parameter on the denoised data), $\theta$ (LLPD denoising threshold), and $\hat{\sigma}$ (learned scaling parameter in LLPD weight matrix) are given.  For the synthetic data, $\tilde{n}$ (number of noise points), $d$ (intrinsic dimension of the data), and $\delta$ (minimal Euclidean distance between clusters) are given, since these are known or can be computed exactly.  For the real data, $\hat{\delta}$ (the minimal Euclidean distanced between clusters, after denoising) is provided.  We remark that for the Skins dataset, a very small number of points (which are integer triples in $\mathbb{R}^{3}$) appear in both classes, so that $\hat{\delta}=0$.  Naturally these points are not classified correctly, which leads to a slightly imperfect accuracy for LLPD spectral clustering.}
\end{table}


\section{Conclusions and Future Directions}\label{sec:Conclusions}

This article developed finite sample estimates on the behavior of the LLPD metric, derived theoretical guarantees for spectral clustering with the LLPD, and introduced fast approximate algorithms for computing the LLPD and LLPD spectral clustering.  The theoretical guarantees on the eigengap provide mathematical rigor for the heuristic claim that the eigengap determines the number of clusters, and theoretical guarantees on labeling accuracy improve on the state of the art in the LDLN data model.  Moreover, the proposed approximation scheme enables efficient LLPD spectral clustering on large, high-dimensional datasets.  Our theoretical results are verified numerically, and it is shown that LLPD spectral clustering determines the number of clusters and labels points with high accuracy in many cases where Euclidean spectral clustering fails.  In a sense, the method proposed in this article combines two different clustering techniques: density techniques like DBSCAN and single linkage clustering, and spectral clustering.  The combination allows for improved robustness and performance guarantees compared to either set of techniques alone.

It is of interest to generalize and improve the results in this article.  Our theoretical results involved two components.  First, we proved estimates on distances between points under the LLPD metric, under the assumption that data fits the LDLN model.  Second, we proved that the weight matrix corresponding to these distances enjoys a structure which guarantees that the eigengap in the normalized graph Laplacian is informative.  The first part of this program is generalizable to other distance metrics and data drawn from different distributions.  Indeed, one can interpret the LLPD as a minimum over the $\ell^{\infty}$ norm of paths between points.  Norms other than the $\ell^{\infty}$ norm may correspond to interesting metrics for data drawn from some class of distributions, for example, the geodesic distance with respect to some metric on a manifold.  Moreover, introducing a comparison of tangent-planes into the spectral clustering distance metric has been shown to be effective in the Euclidean setting \citep{Arias2017}, and allows one to distinguish between intersecting clusters in many cases.  Introducing tangent plane comparisons into the LLPD construction would perhaps allow the results in this article to generalize to data drawn from intersecting distributions.

An additional problem not addressed in the present article is the consistency of LLPD spectral clustering.  It is of interest to consider the behavior as $n\rightarrow \infty$ and determine if LLPD spectral clustering converges in the large sample limit to a continuum partial differential equation.  This line of work has been fruitfully developed in recent years for spectral clustering with Euclidean distances \citep{Trillos2016_1, Trillos2016_2, Trillos2016_3}.


\acks{MM and JMM are grateful for partial support by NSF-IIS-1708553, NSF-DMS-1724979, NSF-CHE-1708353 and AFOSR FA9550-17-1-0280.}


\appendix

\section{Proofs from Section \ref{sec:FiniteSampleAnalysis} }
\label{app:PDProofs}

\subsection*{Proof of Lemma \ref{Lemma:BoundingConstants}}
\label{subsec:PDProofsBoundingConstants}

Let $y\in S$ satisfy $\|x-y\|_{2}\le \tau$.   Suppose $\tau\ge {\epsilon}/{4}$.  For the upper bound, we have: $$\volD(B(S,\tau)\cap B_{\epsilon}(x))\le \volD(B_{\epsilon}(x)) = \volD(B_{1})\epsilon^{D}\le \volD(B_{1})\epsilon^{d}4^{D-d}\left(\tau\wedge \epsilon\right)^{D-d}.$$For the lower bound, set $z=(1-\alpha)x+\alpha y, \ \alpha=\frac{\epsilon}{4\tau}.$  Then $\|z-x\|_{2}\le {\epsilon}/{4}$ and $\|z-y\|_{2}\le \tau -{\epsilon}/{4}$, so $B_{{\epsilon}/{4}}(z)\subset B(S,\tau)\cap B_{\epsilon}(x)$, 
and $$4^{-D}\volD(B_{1})\epsilon^{d}(\epsilon\wedge\tau)^{D-d}\le 4^{-D}\volD(B_{1})\epsilon^{D}=\volD(B_{{\epsilon}/{4}}(z))\le\volD(B(S,\tau)\cap B_{\epsilon}(x)).$$  This shows that (\ref{eqn:GeometricBound}) holds in the case $\tau\ge {\epsilon}/{4}$.

Now suppose $\tau < {\epsilon}/{4}.$  We consider two cases, with the second to be reduced to the first one.

\emph{Case 1:  $x=y\in S$.} Let $\{y_{i}\}_{i=1}^{n}$ be a $\tau$-packing of $S\cap B_{\epsilon-\tau}(y)$, i.e.: $S\cap B_{\epsilon-\tau}(y)\subset \bigcup_{i=1}^{n}B_{\tau}(y_{i})$, and $\|y_{i}-y_{j}\|_{2}>\tau, i\neq j$.
We show this implies that $B(S,\tau)\cap B_{\frac{\epsilon}{2}}(y)\subset\bigcup_{i=1}^{n}B_{2\tau}(y_{i})$.  Indeed, let $x_{0}\in B(S,\tau)\cap B_{\frac{\epsilon}{2}}(y)$.   
Then there is some $x^{*}\in S$ such that $\|x_0-x^{*}\|_{2} \leq \tau$, and so $\|x^{*} - y\|_{2} \leq \|x^{*}-x_0\|_{2}+\|x_0-y\|_{2} \leq \tau + {\epsilon}/{2}<\epsilon-\tau$ (since $\tau<{\epsilon}/{4}$), and hence $x^{*} \in S\cap B_{\epsilon-\tau}(y)$.
Thus there exists $y_{i}^{*}$ in the $\tau$-packing of $S\cap B_{\epsilon-\tau}(y)$ such that $x^{*}\in B_{\tau}(y_{i}^{*})$, so that $\|x_{0}-y_{i}^{*}\|_{2}\le \|x_{0}-x^{*}\|_{2}+\|x^{*}-y_{i}^{*}\|_{2}\le 2\tau$, and $x_{0}\in B_{2\tau}(y_{i}^{*})$.  Hence, 
\begin{equation}
\label{equ:ubvol_geolemma}
\volD(B(S,\tau)\cap B_{\frac{\epsilon}{2}}(y))\le \sum_{i=1}^{n}\volD(B_{2\tau}(y_{i}))=n\volD(B_{1})2^{D}\tau^{D}.
\end{equation}
Similarly, it is straight-forward to verify that $\bigcup_{i=1}^{n}B_{\frac{\tau}{2}}(y_{i})\subset B(S,\tau)\cap B_{\epsilon}(y),$
and since the $\{B_{\frac{\tau}{2}}(y_{i})\}_{i=1}^{n}$ are pairwise disjoint, it follows that 
\begin{equation}
\label{equ:lbvol_geolemma}
n\volD(B_{1})2^{-D}\tau^{D}=\volD\left(\cup_{i=1}^{n} B_{\frac{\tau}{2}}(y_{i})\right)\le \volD(B(S,\tau)\cap B_{\epsilon}(y)).
\end{equation}  
We now estimate $n$.  Indeed, $S\cap B_{\frac{\epsilon}{2}}(y)\subset S\cap B_{\epsilon-\tau}(y)\subset \bigcup_{i=1}^{n}S\cap B_{\tau}(y_{i}),$ so that by assumption $S\in \mathcal{S}_{d}(\kappa,\epsilon_{0})$ and $\epsilon\in (0,\frac{2\epsilon_{0}}{5})\subset (0,\epsilon_{0})$, $$2^{-d}\epsilon^{d}\kappa^{-1}\vold(B_{1})\le \vold(S\cap B_{\frac{\epsilon}{2}}(y))\le \sum_{i=1}^{n}\vold(S\cap B_{\tau}(y_{i}))\le \kappa\tau^{d}n\vold(B_{1}).$$ It follows that
\begin{equation}
\label{equ:lbn_geolemma}
2^{-d}\left({\epsilon}/{\tau}\right)^{d}\kappa^{-2}\le n.
\end{equation} 
Similarly, $\bigcup_{i=1}^{n}S\cap B_{\frac{\tau}{2}}(y_{i})\subset S\cap B_{\epsilon}(y)$ yields 
$$n\kappa^{-1}\tau^{d}2^{-d}\vold(B_{1})\le\sum_{i=1}^{n}\vold(S\cap B_{\frac{\tau}{2}}(y_{i}))\le \vold(S\cap B_{\epsilon}(y))\le \kappa \epsilon^{d}\vold(B_{1}),$$
so that
\begin{equation}
\label{equ:ubn_geolemma}
n\le 2^{d}\kappa^{2}\left({\epsilon}/{\tau}\right)^{d}.
\end{equation}
By combining (\ref{equ:lbvol_geolemma}) and (\ref{equ:lbn_geolemma}), we obtain
\begin{align}
\label{equ:LowerBound}
\volD(B_{1})2^{-(d+D)}\kappa^{-2}\tau^{D}\left({\epsilon}/{\tau}\right)^{d}\le\volD(B(S,\tau)\cap B_{\epsilon}(y))
\end{align}
and by combining (\ref{equ:ubvol_geolemma}) and (\ref{equ:ubn_geolemma}), we obtain 
\begin{align}\label{eqn:UpperBound}\volD(B(S,\tau)\cap B_{\frac{\epsilon}{2}}(y))\le \volD(B_{1})2^{d+D}\kappa^{2}\tau^{D}\left({\epsilon}/{\tau}\right)^{d},\end{align}  
which are valid for any $\epsilon<\epsilon_0, \tau < {\epsilon}/{4}$.
Replacing ${\epsilon}/{2}$ and $\tau$ with $\epsilon$ and $2\tau$, respectively, in (\ref{eqn:UpperBound}), and combining with (\ref{equ:LowerBound}), we obtain, for $\epsilon < {\epsilon_0}/{2}, \tau < {\epsilon}/{4}$, $$\volD(B_{1})2^{-(d+D)}\kappa^{-2}\tau^{D}\left({\epsilon}/{\tau}\right)^{d}\le\volD(B(S,\tau)\cap B_{\epsilon}(y))\le \volD(B_{1})2^{d+2D}\kappa^{2}\tau^{D}\left({\epsilon}/{\tau}\right)^{d}\,.$$

\emph{Case 2: $x\notin S$.}  Notice that $\|x-y\|_{2}\le\tau\le{\epsilon}/{4}$, so $B_{\frac{3\epsilon}{4}}(y)\subset B_{\epsilon}(x)\subset B_{\frac{5\epsilon}{4}}(y)$. Thus:
$\volD(B(S,\tau) \cap B_{\epsilon}(x)) \leq \volD(B(S,\tau) \cap B_{\frac{5\epsilon}{4}}(y)) \leq \volD(B_{1})2^{2D+2d}\kappa^{2}\tau^{D}\left({\epsilon}/{\tau}\right)^{d},$
so as long as $\epsilon < \frac{2\epsilon_0}{5}$ we have
\begin{align*}
\volD(B(S,\tau) \cap B_{\epsilon}(x)) 
&\geq \volD(B(S,3\tau/4) \cap B_{{3\epsilon}/{4}}(y)) \geq \volD(B_{1})2^{-(2D+d)}\kappa^{-2}\tau^{D}\left({\epsilon}/{\tau}\right)^{d}.
\end{align*} 

We thus obtain the statement in Lemma \ref{Lemma:BoundingConstants}.


 \subsection*{Proof of Theorem \ref{thm:WithinCluster}}
	
	Cover $B(S,\tau)$ with an ${\epsilon}/{4}$-packing $\{y_{i}\}_{i=1}^{N}$,
 such that  $B(S,\tau)\subset \bigcup_{i=1}^{N} B_{{\epsilon}/{4}}(y_{i})$, and $\|y_{i}-y_{j}\|_{2}>{\epsilon}/{4}, \forall i\neq j$.
		$\{B_{{\epsilon}/{8}}(y_{i})\}_{i=1}^{N}$ are thus pairwise disjoint, so that $\sum_{i=1}^{N}\volD(B_{{\epsilon}/{8}}(y_{i})\cap B(S,\tau))\le \volD(B(S,\tau)).$  By Lemma \ref{Lemma:BoundingConstants}, we may bound $C_{1}\left({\epsilon}/{8}\right)^{d}\left({\epsilon}/{8}\wedge \tau\right)^{D-d}\le \volD(B_{{\epsilon}/{8}}(y_{i})\cap B(S,\tau)),$ where $C_{1}=\kappa^{-2}2^{-(2D+d)}\volD(B_{1}).$ 
	Hence, $NC_{1}\left({\epsilon}/{8}\right)^{d}\left({\epsilon}/{8}\wedge \tau\right)^{D-d}\le \volD(B(S,\tau)),$ so that 	
	\begin{align*}N\le C\volD(B(S,\tau))\left({\epsilon}/{8}\right)^{-d}\left({\epsilon}/{8}\wedge \tau\right)^{-(D-d)},\ C=\kappa^{2}2^{2D+d}(\volD(B_{1}))^{-1}.\end{align*}So, $C\volD(B(S,\tau))\left ({\epsilon}/{8}\right)^{-d}\left({\epsilon}/{8}\wedge \tau\right)^{-(D-d)}$ balls of radius ${\epsilon}/{4}$ are needed to cover $B(S,\tau)$.  We now determine how many samples $n$ must be taken so that each ball contains at least one sample with probability exceeding $1-t$.  If this occurs, then each pair of points is connected by a path with all edges of length at most $\epsilon$.  Notice that the distribution of the number of points $\omega_{i}$ in the set $B_{{\epsilon}/{4}}(y_{i})\cap B(S,\tau)$ is $\omega_{i}\sim \Bin(n,p_{i})$, where
	\begin{align*}p_{i}=\frac{\volD(B(S,\tau)\cap B_{{\epsilon}/{4}}(y_{i}))}{\volD(B(S,\tau))}\ge  \frac{C^{-1}\left({\epsilon}/{4}\right)^{d}\left({\epsilon}/{4}\wedge \tau\right)^{D-d}}{\volD(B(S,\tau))} := p.\end{align*}	
	Since $\Prob( \exists i : \omega_i = 0)\leq N (1-p)^n \leq N e^{-pn} \leq t$ as long as $n \geq 1/p\log N/t$, it suffices for $n$ to satisfy $n\ge  \frac{C \volD(B(S,\tau))}{ \left({\epsilon}/{4}\right)^{d}(\tau \wedge {\epsilon}/{4})^{D-d}}\log\frac{C \volD(B(S,\tau))}{ \left({\epsilon}/{8}\right)^{d}(\tau \wedge {\epsilon}/{8})^{D-d}t}$.
	
\subsection*{Proof of Corollary \ref{cor:WithinCluster_tau_small}}

	For a fixed $l$, choose a $\tau$ packing of $S_l$, i.e. let $y_1, \ldots, y_m \in S_l$ such that $S_l \subset\cup_i B_{\tau}(y_i)$ and $\|y_i-y_j\|_{2}>\tau$ for $i\ne j$.
	Then $B(S_l, \tau) \subset B_{2\tau}(y_i)$. 
	Now we control the size of $m$. Since the $B_{\frac{\tau}{2}}(y_i)$ are disjoint and $S_l \in S(\kappa, \epsilon_0)$,
	$\vold(S_l) \geq \sum_{i=1}^m \vold(S_l \cap B_{\frac{\tau}{2}}(y_i))
	\geq m \kappa^{-1}\vold(B_1)\left(\frac{\tau}{2}\right)^d$,
	so that $m \leq \kappa \frac{\vold(S_l)}{\vold(B_1)}\left(\frac{\tau}{2}\right)^{-d}$. 
	Furthermore, since $\frac{2}{5}\epsilon_0 > 2\tau$ we have by Lemma \ref{Lemma:BoundingConstants}:
	\begin{align*}
	\frac{\volD(B(S_l,\tau))}{\volD(B_1)} &\leq \sum_{i=1}^m \frac{\volD(B(S_l,\tau)\cap B_{2\tau}(y_i))}{\volD(B_1)}
	\leq m\kappa^{2}2^{(2D+2d)}\left(2\tau\right)^d \tau^{D-d}
	\leq \kappa^{3}2^{(2D+4d)}\frac{\vold(S_l)}{\vold(B_1)} \tau^{D-d}.
	\end{align*}
	Combining the above with (\ref{equ:tau_small_sampling_cond}) implies
	$$n_l\ge  \frac{C \volD(B(S_l,\tau))}{ \left({\epsilon}/{4}\right)^{d}\tau^{D-d}\volD(B_{1})}\log\frac{C\volD(B(S_l,\tau))}{\left({\epsilon}/{8}\right)^{d}\tau^{D-d}\volD(B_{1})t}$$
	for $C=\kappa^{2}2^{2D+d}$. Thus by Theorem \ref{thm:WithinCluster}, $\Prob(\max_{x\ne y \in X_l} \rho_{\ell \ell}(x,y)< \epsilon)\ge 1-\frac{t}{\numclust}.$ Repeating the above argument for each $S_l$ and letting $E_l$ denote the event that $\max_{x\ne y \in X_l} \rho_{\ell \ell}(x,y) \geq \epsilon$, we obtain
	$\Prob( \epsilonincluster \geq \epsilon ) = \Prob( \max_{1\le l\le \numclust}\max_{x \ne y \in X_{l}} \rho_{\ell\ell}(x, y) \geq \epsilon )=\Prob(\cup_l E_l) \leq \sum_l \Prob(E_l) \leq \numclust(\frac{t}{\numclust})=t.$
	 
\subsection*{Proof of Theorem \ref{thm:CombinedResult}}

	Re-writing the inequality assumed in the theorem, we are guaranteed that
	\[ C\left( \max_{l =1, \ldots, \numclust} \frac{4^d \kappa^52^{4D+5d} \vold(S_l)}{n_l \vold(B_1)}\log\left(2^d n_l\frac{2\numclust}{t}\right)\right)^{\frac{1}{d}} < \left( \frac{ \left(\frac{t}{2}\right)^{\frac{1}{\kNoise}}}{((\kNoise+1)\tilde{n})^{\frac{\kNoise+1}{\kNoise}}} \frac{\volD(\tilde{X})}{\volD(B_1)}\right)^{\frac{1}{D}}.\]
	Let $C\epsilon^*_1$ denote the left hand side of the above inequality and $\epsilon^*_2$ the right hand side. Then for all $1\leq l\leq\numclust$,
	$
	n_l \geq \frac{\kappa^52^{4D+5d} \vold(S_l)}{\left({\epsilon^*_{1}}/{4}\right)^d\vold(B_1)}\left(\log\left(2^d n_l\frac{2\numclust}{t}\right)\right)
	$,
	and since $\log\left(2^d n_l\frac{2\numclust}{t}\right) \geq 1$, clearly $n_l \geq \frac{\kappa^52^{4D+5d} \vold(S_l)}{\left({\epsilon^*_{1}}/{4}\right)^d\vold(B_1)}$, and we obtain
	$n_l \geq \left( \frac{\kappa^52^{4D+5d} \vold(S_l)}{\left({\epsilon^*_{1}}/{4}\right)^d\vold(B_1)}\log \left(\frac{\kappa^52^{4D+5d}\vold(S_l)}{\left({\epsilon^*_{1}}/{8}\right)^d\vold(B_1)}\frac{2\numclust}{t}\right)\right)$. 
	Since $\tau < \frac{\epsilon^*_{1}}{8} \wedge \frac{\epsilon_0}{5}$ by assumption, Corollary \ref{cor:WithinCluster_tau_small} yields $\Prob\left( \epsilonincluster < \epsilon^*_{1} \right) \geq 1 - \frac{t}{2}.$
	Also by Theorem \ref{thm:knnLLPD}, $\Prob( \epsilonnoiseknn > \epsilon^*_2) \geq 1 - \frac{t}{2}.$  Since we are assuming $C\epsilon^*_1<\epsilon^*_2$,
$
	\Prob ( C\epsilonincluster <  \epsilonnoiseknn)
	 \ge \Prob((\epsilonincluster < \epsilon^*_1) \cap(\epsilonnoiseknn > \epsilon^*_2))
	 \geq 1-t
	$.

\section{Proof of Theorem \ref{thm:SpectralGap}}
\label{app:SpectralGapProof}

Let $n_l =|A_l|$ and $m_l=|C_l|$ denote the cardinality of the sets in Assumption \ref{assumption:ultrametric}, and let $\eta_1 := 1 - f_{\sigma}(\epsilonincluster), \eta_{\thetathres} := 1 - f_{\sigma}(\thetathres), \eta_2 := f_{\sigma}(\epsilonnoiseknn).$

\subsection{Bounding Entries of Weight Matrix $W$ and Degrees}

The following Lemma guarantees that the weight matrix will have a convenient structure.  

\begin{lem}
\label{lem:ultrametric}
For $1 \leq l \leq \numclust$, let $A_l, C_l, \tilde{A}_l$ be as in Assumption \ref{assumption:ultrametric}.  Then:
\begin{enumerate}
\item For each fixed $x_{i}^{l}\in C_l$, $x_{i}^{l}$ is equidistant from all points in $A_l$; more precisely:
\[ \rho(x^l_i,x^l_j) =\min_{x^l \in A_l} \rho(x^l_i, x^l)=:\rho^l_i, \qquad \forall  x^l_i \in C_l, x^l_j \in A_l, 1 \leq l \leq \numclust.\]
\item The distance between any point in $\tilde{A}_l$ and $\tilde{A}_s$ is constant for $l \ne s$, that is:
\[ \rho(x^l_i,x^s_j) =\min_{x^l \in \tilde{A}_l, x^s \in \tilde{A}_s } \rho(x^l, x^s) =: \rho^{l,s} \qquad \forall x^l_i \in \tilde{A}_l, x^s_j \in \tilde{A}_s, 1 \leq l\ne s \leq \numclust. \] 
\end{enumerate} 
\end{lem}
\begin{proof}
To prove (1), let $x^l_i \in C_l$ and $x^l_j \in A_l$. Since $\rho^l_{i}$ is the minimum distance between $x^l_i$ and a point in $A_l$, clearly $\rho(x^l_i,x^l_j)\geq \rho^l_i$. Let $x^l_*$ denote the point in $A_l$ such that $\rho^l_{i} = \rho(x^l_{i}, x^l_*)$. Then $\rho(x^l_i, x^l_j)  \leq \rho(x^l_i, x^l_*) \vee \rho(x^l_*, x^l_j) \leq \rho(x^l_i, x^l_*) \vee \epsilonincluster =\rho^l_i \vee \epsilonincluster =\rho^l_i.$  Thus $\rho(x^l_i,x^l_j)= \rho^l_i$. 

To prove (2), let $x^l_i \in \tilde{A}_l$ and $x^s_j\in \tilde{A}_s$ for $l\ne s$. Clearly, $\rho(x^l_i,x^s_j) \geq \rho^{l,s}$. Now let $x^l_* \in \tilde{A}_l, x^s_* \in \tilde{A}_s$ be the points that achieve the minimum, i.e. $\rho^{l,s} = \rho(x^l_*, x^s_*)$. Note that $\rho(x^l_i,x^l_*) \leq \thetathres$ and similarly for $\rho(x^s_j,x^s_*)$ (if $x^l_i, x^l_*$ are both in $C_l$, pick any point $z \in A_l$ to obtain $\rho(x^l_i,x^l_*)\leq \rho(x^l_i,z) \vee \rho(z,x^l_*) \leq \thetathres$). Thus:
\begin{align*}
\rho(x^l_i,x^s_j) &\leq \rho(x^l_i, x^l_*) \vee \rho(x^l_*,x^s_*) \vee \rho(x^s_*,x^s_j) 
	 \leq \thetathres \vee \rho^{l,s} \vee \thetathres 
	=\rho^{l,s}, 
\end{align*}
since $\rho^{l,s}\geq \epsilonnoiseknn>\thetathres$. Thus $\rho(x^l_i,x^s_j) = \rho^{l,s}$.
\end{proof}

We now proceed with the proof of Theorem \ref{thm:SpectralGap}.
By Lemma \ref{lem:ultrametric}, the off-diagonal blocks of $W$  are constant, and so letting $w^{l,s}=W_{\tilde{A_l},\tilde{A_s}}$ denote this constant, $W$ has form
\begin{equation*}
W = 
\begin{bmatrix}
W_{\tilde{A_1},\tilde{A_1}} & w^{1,2} & \ldots & w^{1,\numclust} \\
w^{2,1} & W_{\tilde{A_2},\tilde{A_2}} & \ldots & w^{2,\numclust} \\
\vdots & & & \vdots \\
w^{\numclust,1} & w^{\numclust,2} & \ldots & W_{\tilde{A_\numclust},\tilde{A_\numclust}} \\
\end{bmatrix},
\end{equation*}

and $w^{l,s}\leq f_{\sigma}(\epsilonnoiseknn)$ for $1 \leq l \ne s \leq \numclust$ by (\ref{equ:assump3}).

We now consider an arbitrary diagonal block $W_{\tilde{A_l},\tilde{A_l}}$. For convenience let $x^l_i$, $1\leq i \leq n_l+m_l$ denote the points in $\tilde{A}_l$, ordered so that $x^l_i \in A_{l}$ for $i=1,\ldots, n_l$ and $x^l_i\in C_{l}$ for $i=n_l+1, \ldots, n_l+m_l$. For every $x^l_{i+n_l} \in C_l$, let $\rho^l_{i+n_l}$ denote the minimal distance to $A_l$, i.e. $\rho^l_{i+n_l} = \min_{x^l \in A_l} \rho(x^l_{i+n_l}, x^l), w^l_i = f_{\sigma}(\rho^l_{i+n_l})$ for all $1 \leq l \leq \numclust, 1\leq i\leq m_l$.
Then by Lemma \ref{lem:ultrametric}, any point in $C_{l}$ is equidistant from all points in $A_{l}$, so that $(W_{\tilde{A}_l,\tilde{A}_l})_{ij}  = w^l_{i-n_l} \text{ for all } x^l_{i} \in C_{l}, x^l_j \in A_{l},$
and by (\ref{equ:assump2}), $f_{\sigma}(\epsilonincluster) > w^l_{i-n_l} \geq f_{\sigma}(\thetathres)$ for $n_l+1 \leq i \leq n_l+m_l$.  Note if $x^l_i, x^l_j \in C_{l}$, then pick any $x^l_* \in A_l$, and one has $\rho(x^l_i, x^l_j) \leq \rho(x^l_i, x^l_*) \vee \rho(x^l_*,x^l_j) \leq \thetathres$ by (\ref{equ:assump2}). 

Thus the diagonal blocks of $W$ have the following form:
\[ W_{\tilde{A_l},\tilde{A_l}}  = 
\begin{bmatrix}
[f_{\sigma}(\epsilonincluster),1] & [f_{\sigma}(\epsilonincluster),1] & \ldots & [f_{\sigma}(\epsilonincluster),1] & w^l_1 & w^l_2 & \ldots & w^l_{m_l} \\
[f_{\sigma}(\epsilonincluster),1] & [f_{\sigma}(\epsilonincluster),1] & \ldots & [f_{\sigma}(\epsilonincluster),1] & w^l_1 & w^l_2 & \ldots & w^l_{m_l} \\
\vdots & & & \vdots & \vdots & \vdots & & \vdots \\
[f_{\sigma}(\epsilonincluster),1] & [f_{\sigma}(\epsilonincluster),1] & \cdots & [f_{\sigma}(\epsilonincluster),1] & w^l_1 & w^l_2 & & w^l_{m_l}  \\
w^l_1 & w^l_1 & \cdots & w^l_1 & [f_{\sigma}(\thetathres),1] & [f_{\sigma}(\thetathres),1] & \cdots & [f_{\sigma}(\thetathres),1] \\
w^l_2 & w^l_2 & \cdots & w^l_2 & [f_{\sigma}(\thetathres),1] & [f_{\sigma}(\thetathres),1] & \cdots & [f_{\sigma}(\thetathres),1] \\
\vdots & & & \vdots & \vdots & & & \vdots \\
w^l_{m_l} & w^l_{m_l} & \cdots & w^l_{m_l} &[f_{\sigma}(\thetathres),1] & [f_{\sigma}(\thetathres),1] & \cdots & [f_{\sigma}(\thetathres),1]
\end{bmatrix}
\]
The entries labeled $[f_{\sigma}(\epsilonincluster),1]$ or $[f_{\sigma}(\thetathres),1]$ indicate entries falling in the interval.  So we have the following bounds on the entries of $W$:
\begin{align*}
f_{\sigma}(\epsilonincluster) &\leq (W_{\tilde{A}_l,\tilde{A}_l})_{ij} \leq 1 &\text{ for } x^l_i, x^l_j \in A_l,\\
f_{\sigma}(\thetathres) &\leq (W_{\tilde{A}_l,\tilde{A}_l})_{ij} < f_{\sigma}(\epsilonincluster) &\text{ for } x^l_i \in A_l, x^l_j \in C_l, \\
f_{\sigma}(\thetathres) &\leq (W_{\tilde{A}_l,\tilde{A}_l})_{ij} \leq 1 &\text{ for } x^l_i, x^l_j \in C_l, \\
0 &\leq (W_{\tilde{A}_l,\tilde{A}_s})_{ij} \leq f_{\sigma}(\epsilonnoiseknn) &\text{ for } x^l_i \in \tilde{A}_l, x^s_j \in \tilde{A}_s, l\ne s.
\end{align*}

Let $\text{deg}^l_i$ denote the degree of $x^l_i$, and let $w^l = \sum_{j=1}^{m_l}w^l_j$, $o^l=\sum_{s\ne l}(n_s+m_s)w^{l,s}$. Note that regarding the degrees:  
\begin{align*}
n_lf_{\sigma}(\epsilonincluster)+w^l+o^l &\leq \text{deg}^l_i \leq n_l+w^l+o^l \quad\text{ for } x^l_i \in A_l,\\
(n_l+m_l)f_{\sigma}(\thetathres) + o^l &\leq \text{deg}^l_i \leq n_l+m_l+o^l \quad\text{ for } x^l_i \in C_l.
\end{align*}
where $m_lf_{\sigma}(\thetathres) \leq w^l \leq m_lf_{\sigma}(\epsilonincluster) \leq m_l.$

\subsection{Bounding Entries of Normalized Weight Matrix $\Deg^{-\frac{1}{2}}W\Deg^{-\frac{1}{2}}$}
\label{subsec:BoundingNormW}

We thus obtain the following entrywise bounds for the diagonal block $\Deg_{\tilde{A}_l}^{-\frac{1}{2}}W_{\tilde{A}_l,\tilde{A}_l}\Deg_{\tilde{A}_l}^{-\frac{1}{2}}$:
\begin{align*} 
\frac{f_{\sigma}(\epsilonincluster) }{n_l+w^l+o^l} &\leq (\Deg_{\tilde{A}_l}^{-\frac{1}{2}}W_{\tilde{A}_l,\tilde{A}_l}\Deg_{\tilde{A}_l}^{-\frac{1}{2}})_{ij} \leq \frac{1}{n_lf_{\sigma}(\epsilonincluster)+w^l+o^l} \quad &\text{ for } x^l_i, x^l_j \in A_l \\
\frac{f_{\sigma}(\thetathres)}{n_l+m_l+o^l} &\leq  (\Deg_{\tilde{A}_l}^{-\frac{1}{2}}W_{\tilde{A}_l,\tilde{A}_l}\Deg_{\tilde{A}_l}^{-\frac{1}{2}})_{ij}  \leq \frac{1}{(n_l+m_l)f_{\sigma}(\thetathres) + o^l} \quad &\text{ for } x^l_i, x^l_j \in C_l
\end{align*}
For $x^l_i \in A_l, x^l_j \in C_l$, we have:
\begin{footnotesize}
\begin{align*} \frac{f_{\sigma}(\thetathres)}{\sqrt{n_l+w^l+o^l}\sqrt{n_l+m_l+o^l}} &\leq (\Deg_{\tilde{A}_l}^{-\frac{1}{2}}W_{\tilde{A}_l,\tilde{A}_l}\Deg_{\tilde{A}_l}^{-\frac{1}{2}})_{ij} 
	< \frac{f_{\sigma}(\epsilonincluster)}{\sqrt{n_lf_{\sigma}(\epsilonincluster)+w^l+o^l}\sqrt{(n_l+m_l)f_{\sigma}(\thetathres) + o^l}}. 
\end{align*}
\end{footnotesize}

Now consider the off-diagonal block $\Deg^{-\frac{1}{2}}_{\tilde{A}_l}W_{\tilde{A}_l,\tilde{A}_s}\Deg^{-\frac{1}{2}}_{\tilde{A}_s}$ for some $l \ne s$. 
Since $\text{deg}^l_i \geq f_\sigma(\thetathres)\min_l(n_l+m_l) = f_\sigma(\thetathres) \zeta_{N}^{-1}N$ for all data points, we have:
\begin{align*}
\left| \Deg^{-\frac{1}{2}}_{\tilde{A}_l}W_{\tilde{A}_l,\tilde{A}_s}\Deg^{-\frac{1}{2}}_{\tilde{A}_s} \right| &\leq \frac{\zeta_{N} f_\sigma(\epsilonnoiseknn)}{f_\sigma(\thetathres) N}
\end{align*}

\subsection{Perturbation to Obtain a Block Diagonal and Block Constant Matrix}
\label{subsec:pert_desc}

Consider the normalized weight matrix $\Deg^{-\frac{1}{2}}W \Deg^{-\frac{1}{2}}$. This matrix consists of diagonal blocks of the form $\Deg_{\tilde{A}_l}^{-\frac{1}{2}}W_{\tilde{A}_l,\tilde{A}_l}\Deg_{\tilde{A}_l}^{-\frac{1}{2}}$ and off diagonal blocks of the form $\Deg_{\tilde{A}_l}^{-\frac{1}{2}}W_{\tilde{A}_l,\tilde{A}_s}\Deg_{\tilde{A}_s}^{-\frac{1}{2}}$, some $l\ne s$. We will consider the spectral perturbations associated with (1) setting off-diagonal blocks to 0 and (2) making the diagonal blocks essentially block constant.  More precisely, we consider the spectral perturbations associated with the matrix perturbations $\bf{P}_{1}, P_{2}$ given as:
\begin{align*}
\Deg^{-\frac{1}{2}}W \Deg^{-\frac{1}{2}} = 
&\begin{bmatrix}
\Deg_{\tilde{A_1}}^{-\frac{1}{2}}W_{\tilde{A_1},\tilde{A_1}}\Deg_{\tilde{A_1}}^{-\frac{1}{2}} & \Deg_{\tilde{A_1}}^{-\frac{1}{2}}W_{\tilde{A_1},\tilde{A_2}}\Deg_{\tilde{A_2}}^{-\frac{1}{2}} & \ldots & \Deg_{\tilde{A_1}}^{-\frac{1}{2}}W_{\tilde{A_1},\tilde{A_\numclust}}\Deg_{\tilde{A_\numclust}}^{-\frac{1}{2}} \\
\Deg_{\tilde{A_2}}^{-\frac{1}{2}}W_{\tilde{A_2},\tilde{A_1}}\Deg_{\tilde{A_1}}^{-\frac{1}{2}} & \Deg_{\tilde{A_2}}^{-\frac{1}{2}}W_{\tilde{A_2},\tilde{A_2}}\Deg_{\tilde{A_2}}^{-\frac{1}{2}} & \ldots & \Deg_{\tilde{A_2}}^{-\frac{1}{2}}W_{\tilde{A_2},\tilde{A_\numclust}}\Deg_{\tilde{A_\numclust}}^{-\frac{1}{2}} \\
\vdots & & & \vdots \\
\Deg_{\tilde{A_\numclust}}^{-\frac{1}{2}}W_{\tilde{A_\numclust},\tilde{A_1}}\Deg_{\tilde{A_1}}^{-\frac{1}{2}} & \Deg_{\tilde{A_\numclust}}^{-\frac{1}{2}}W_{\tilde{A_\numclust},\tilde{A_2}}\Deg_{\tilde{A_2}}^{-\frac{1}{2}} & \ldots & \Deg_{\tilde{A_\numclust}}^{-\frac{1}{2}}W_{\tilde{A_\numclust},\tilde{A_\numclust}}\Deg_{\tilde{A_\numclust}}^{-\frac{1}{2}} \\
\end{bmatrix} \\
\overset{\bf{P}_1}{\longrightarrow} 
&\begin{bmatrix}
\Deg_{\tilde{A_1}}^{-\frac{1}{2}}W_{\tilde{A_1},\tilde{A_1}}\Deg_{\tilde{A_1}}^{-\frac{1}{2}} & 0 & \ldots & 0 \\
0 & \Deg_{\tilde{A_2}}^{-\frac{1}{2}}W_{\tilde{A_2},\tilde{A_2}}\Deg_{\tilde{A_2}}^{-\frac{1}{2}} & \ldots & 0\\
\vdots & & & \vdots \\
0 & 0 & \ldots & \Deg_{\tilde{A_\numclust}}^{-\frac{1}{2}}W_{\tilde{A_\numclust},\tilde{A_\numclust}}\Deg_{\tilde{A_\numclust}}^{-\frac{1}{2}} \\
\end{bmatrix} := C\\
\overset{\bf{P}_2}{\longrightarrow} 
&\begin{bmatrix}
B_1 & 0 & \ldots & 0 \\
0 & B_2 & \ldots & 0\\
\vdots & & & \vdots \\
0 & 0 & \ldots & B_\numclust \\
\end{bmatrix} :=B, 
\end{align*}
where $B_l$ is defined by
$  B_l=\frac{f_{\sigma}(\epsilonincluster)}{n_l+w^l+o^l } \textbf{1}_{n_l+m_l}. $
Throughout the proof $\textbf{1}_{N}$ denotes the $N \times N$ matrix of all 1's, $I_N$ the $N\times N$ identity matrix, and $\|\cdot\|_{2}$ the spectral norm.

\subsection{Bounding $\bf{P_1}$ (Diagonalization)}
\label{subsec:diagonalization}

We first consider the spectral perturbation due to $\bf{P}_{1}$. Using the bounds from \ref{subsec:BoundingNormW} for an off-diagonal block, the perturbation in the eigenvalues due to $\bf{P_1}$ is bounded by:
\begin{align*}
\|\Deg^{-\frac{1}{2}}W \Deg^{-\frac{1}{2}} - C\|_2 &\leq N \|\Deg^{-\frac{1}{2}}W \Deg^{-\frac{1}{2}} - C\|_{\text{max}} \leq \frac{\zeta_{N} f_\sigma(\epsilonnoiseknn)}{f_\sigma(\thetathres)} = \zeta_{N}\eta_2 + O(\zeta_{N}\eta_2\eta_\thetathres) := P_1
\end{align*}	

\subsection{Bounding $\bf{P_2}$ (Constant Blocks)}
\label{subsec:P2}

We now consider the spectral perturbation due to $\bf{P_2}$. Because $\bf{P_2}$ acts on the blocks of a block diagonal matrix, it is sufficient to bound the perturbation of each block. For the $l^{\text{th}}$ block, let
\[ \Deg_{\tilde{A_l}}^{-\frac{1}{2}}W_{\tilde{A_l},\tilde{A_l}}\Deg_{\tilde{A_l}}^{-\frac{1}{2}}  - 
B_l
:= \begin{bmatrix}   Q_l & R_l \\ R^T_l & S_l \end{bmatrix}  \]
where $Q$ is $n_l \times n_l$, $R$ is $n_l \times m_l$, and $S$ is $m_l \times m_l$, and $R^T_l$ denotes the transpose of $R_l$. We will control the magnitude of each entry in $Q_l,R_l,S_l$ using the bounds computed in  \ref{subsec:BoundingNormW}.

\textit{Bounding $Q_l$:} For $x^l_i,x^l_j \in A_l$, we have
\[ (B_l)_{ij}=\frac{f_{\sigma}(\epsilonincluster)}{n_l+w^l+o^l}\leq \left(\Deg_{\tilde{A_l}}^{-\frac{1}{2}}W_{\tilde{A_l},\tilde{A_l}}\Deg_{\tilde{A_l}}^{-\frac{1}{2}} \right)_{ij} \leq \frac{1}{n_lf_{\sigma}(\epsilonincluster)+w^l+o^l}, \]
so that
$
(Q_l)_{ij} 
	\leq \frac{1}{n_lf_{\sigma}(\epsilonincluster)+w^l+o^l} - \frac{f_{\sigma}(\epsilonincluster)}{n_l+w^l+o^l}  
	\leq \frac{ \frac{1}{f_{\sigma}(\epsilonincluster)} }{n_l+w^l+o^l} - \frac{f_{\sigma}(\epsilonincluster)}{n_l+w^l+o^l} 
	= \frac{2\eta_1+O(\eta_1^2)}{n_l+w^l+o^l}
$.
Since $(Q_l)_{ij} \geq 0$, the above is in fact a bound for $|(Q_l)_{ij}|$, and we obtain:
\begin{align*}
|(Q_l)_{ij}| &\leq \frac{2\eta_1+O(\eta_1^2)}{(1-\eta_\thetathres)(n_l+m_l)} = \frac{2\eta_1+O(\eta_1^2+\eta_\thetathres^2)}{(n_l+m_l)} 
\end{align*}

\textit{Bounding $R_l$:} For $x^l_i \in A_{l}$ and $x^l_{j}\in C_{l}$, note that
\begin{align*} 
\left(\Deg_{\tilde{A_l}}^{-\frac{1}{2}}W_{\tilde{A_l},\tilde{A_l}}\Deg_{\tilde{A_l}}^{-\frac{1}{2}} \right)_{ij} -  (B_l)_{ij} 
&\leq \frac{ f_{\sigma}(\epsilonincluster)}{\sqrt{n_lf_{\sigma}(\epsilonincluster)+w^l+o^l}\sqrt{(n_l+m_l)f_{\sigma}(\thetathres)+o^l}} 
- \frac{f_{\sigma}(\epsilonincluster)}{n_l+w^l+o^l}\\
&=\frac{ \frac{\sqrt{f_{\sigma}(\epsilonincluster)}}{\sqrt{f_{\sigma}(\thetathres)}} - f_{\sigma}(\epsilonincluster) }{\sqrt{n_l+w^l+o^l}\sqrt{n_l+m_l+o^l}}.
\end{align*}
Similarly:
\begin{align*} 
\left(\Deg_{\tilde{A_l}}^{-\frac{1}{2}}W_{\tilde{A_l},\tilde{A_l}}\Deg_{\tilde{A_l}}^{-\frac{1}{2}} \right)_{ij} -  (B_l)_{ij} 
&\geq \frac{ f_{\sigma}(\thetathres)}{\sqrt{n_l+w^l+o^l}\sqrt{n_l+m_l+o^l}} - \frac{ f_{\sigma}(\epsilonincluster)}{\sqrt{n_l+w^l+o^l}\sqrt{n_l+m_lf_{\sigma}(\thetathres)+o^l}} \\
&= \frac{ f_{\sigma}(\thetathres) - \frac{f_{\sigma}(\epsilonincluster)}{\sqrt{f_{\sigma}(\thetathres)}} }{\sqrt{n_l+w^l+o^l}\sqrt{n_l+m_l+o^l}}.
\end{align*}
so that $ |R_{ij}| \leq \frac{ \left(\frac{\sqrt{f_{\sigma}(\epsilonincluster)}}{\sqrt{f_{\sigma}(\thetathres)}} - f_{\sigma}(\epsilonincluster)\right) \vee \left(\frac{f_{\sigma}(\epsilonincluster)}{\sqrt{f_{\sigma}(\thetathres)}} - f_{\sigma}(\thetathres)\right) }{\sqrt{n_l+w^l+o^l}\sqrt{n_l+m_l+o^l}}.$  Thus we obtain:

\begin{align*}
|R_{ij}| &\leq  \left[\left(\frac{\sqrt{f_{\sigma}(\epsilonincluster)}}{f_{\sigma}(\thetathres)} -\frac{ f_{\sigma}(\epsilonincluster)}{\sqrt{f_\sigma(\thetathres)}}\right) \vee \left(\frac{f_{\sigma}(\epsilonincluster)}{f_{\sigma}(\thetathres)} - \sqrt{f_{\sigma}(\thetathres)}\right)\right] (n_l+m_l)^{-1} \\
&\leq \left[\left(\frac{\sqrt{1-\eta_1}}{1-\eta_\thetathres} - \frac{1-\eta_1}{\sqrt{1-\eta_\thetathres}}\right) \vee \left( \frac{1-\eta_1}{1-\eta_\thetathres} - \sqrt{1-\eta_\thetathres}\right)\right] (n_l+m_l)^{-1}  \\
&\leq \left[\left(\frac{\eta_1}{2}+\frac{\eta_\thetathres}{2}+ O(\eta_1^2+\eta_\thetathres^2)\right) \vee \left(\frac{3\eta_{\thetathres}}{2}-\eta_1+O(\eta_1^2+\eta_{\thetathres}^2)\right)\right](n_l+m_l)^{-1}  \\
&\leq \left[\frac{3\eta_\thetathres}{2} + O(\eta_1^2+\eta_\thetathres^2)\right](n_l+m_l)^{-1} 
\end{align*}

\textit{Bounding $S_l$:} For $x^l_{i},x^l_{j} \in C_{l}$, note that
\begin{align*} 
\left(\Deg_{\tilde{A_l}}^{-\frac{1}{2}}W_{\tilde{A_l},\tilde{A_l}}\Deg_{\tilde{A_l}}^{-\frac{1}{2}} \right)_{ij} - (B^l)_{ij} \leq  \frac{1}{(n_l+m_l)f_{\sigma}(\thetathres) + o^l} - \frac{f_{\sigma}(\epsilonincluster)}{n_l+w^l+o^l }
\leq \frac{ f_{\sigma}(\thetathres)^{-1} - f_{\sigma}(\epsilonincluster)}{n_l+m_l+o^l}.
\end{align*}
Similarly:
\begin{align*} 
\left(\Deg_{\tilde{A_l}}^{-\frac{1}{2}}W_{\tilde{A_l},\tilde{A_l}}\Deg_{\tilde{A_l}}^{-\frac{1}{2}} \right)_{ij} - (B^l)_{ij} &\geq  \frac{f_{\sigma}(\thetathres)}{n_lf_{\sigma}(\epsilonincluster)+m_l+o^l} -  \frac{f_{\sigma}(\epsilonincluster)}{n_l+w^l+o^l } \\
&\geq \frac{f_{\sigma}(\thetathres)}{n_l+m_l+o^l} -  \frac{f_{\sigma}(\epsilonincluster)}{n_l+m_lf_{\sigma}(\thetathres)+o^l }
=\frac{f_{\sigma}(\thetathres)-\frac{f_{\sigma}(\epsilonincluster)}{f_{\sigma}(\thetathres)}}{n_l+m_l+o^l}.
\end{align*}
Thus we have: $|S_{ij}| \leq \frac{ \left(\frac{1}{f_{\sigma}(\thetathres)} - f_{\sigma}(\epsilonincluster)\right) \vee \left(\frac{f_{\sigma}(\epsilonincluster)}{f_{\sigma}(\thetathres)}-f_{\sigma}(\thetathres)\right)}{n_l+m_l+o^l},$ so that 

\begin{align*}
|S_{ij}| &\leq  \left[\left(\frac{1}{f_{\sigma}(\thetathres)} - f_{\sigma}(\epsilonincluster)\right) \vee \left(\frac{f_{\sigma}(\epsilonincluster)}{f_{\sigma}(\thetathres)}-f_{\sigma}(\thetathres)\right)\right](n_l+m_l)^{-1} \\
&= \left[\left( \eta_1+\eta_{\thetathres}+O(\eta_{\thetathres}^2) \right)  \vee \left( 2\eta_{\thetathres} -\eta_1 +O(\eta_1^2+\eta_{\thetathres}^2) \right)\right](n_l+m_l)^{-1}  \\
&\leq \left[2\eta_\thetathres + O(\eta_1^2+\eta_{\thetathres}^2)\right](n_l+m_l)^{-1} 
\end{align*}

Thus the norm of the spectral perturbation of  $\Deg_{\tilde{A}_l}^{-\frac{1}{2}}W_{\tilde{A}_l,\tilde{A}_l}\Deg_{\tilde{A}_l}^{-\frac{1}{2}} \overset{\bf{P_2}}{\longrightarrow} B_l$
is bounded by
\begin{align*}
\|\Deg_{\tilde{A}_l}^{-\frac{1}{2}}W_{\tilde{A}_l,\tilde{A}_l}\Deg_{\tilde{A}_l}^{-\frac{1}{2}} - B_l\|_2 
&\leq (n_l+m_l)\|\Deg_{\tilde{A}_l}^{-\frac{1}{2}}W_{\tilde{A}_l,\tilde{A}_l}\Deg_{\tilde{A}_l}^{-\frac{1}{2}} - B_l\|_{\text{max}} \\
&\leq (n_l+m_l)\left( \|Q\|_{\text{max}} \vee \|R\|_{\text{max}} \vee \|S\|_{\text{max}} \right) \\
&\leq 2\eta_1+2\eta_\thetathres + O(\eta_1^2+\eta_\thetathres^2) := P^l_2.
\end{align*}

Defining $P_2 := \max_l P^l_2,$ the perturbation of all eigenvalues due to $\bf{P_2}$ is bounded by $P_2$. 


\subsection{Bounding the Eigenvalues of $\Lsym$}
\label{subsec:BoundingLSYMEigs}
Since the eigenvalues of $ \textbf{1}_{n_l+m_l} $ are
\[
\lambda_i(\textbf{1}_{n_l+m_l})  =
\begin{cases}
 0 & i=1, \ldots, n_l+m_l-1 \\
 n_l+m_l & i=n_l+m_l
\end{cases}\,,
\]
we have
\[
\lambda_i(B_l)  =
\begin{cases}
 0 & i=1, \ldots, n_l+m_l-1 \\
 \frac{(n_l+m_l)f_{\sigma}(\epsilonincluster)}{n_l+w^l+o^l}& i=n_l+m_l.
\end{cases}
\]	
Note that since the blocks $B^l$ are orthogonal, the eigenvalues of $B$ are simply the union of the eigenvalues of the blocks, and the eigenvalues of $I-B$ are obtained by subtracting the eigenvalues of $B$ from 1.  
Thus:
\[
\lambda^l_i(I-B) =
\begin{cases}1-\frac{ (n_l+m_l)f_{\sigma}(\epsilonincluster)}{n_l+w^l+o^l} & i=1, \ 1\leq l \leq \numclust, \\
 1 & i=2, \ldots, n_l+m_l, \ 1\leq l \leq \numclust.
\end{cases}
\]
Since $| \lambda_i(\Lsym) - \lambda_i(I-B) | \leq \| B - \Deg^{-\frac{1}{2}}W\Deg^{-\frac{1}{2}} \|_2 
 \leq P_1 + P_2,$
and $\Lsym$ is positive semi-definite, by the Hoffman-Wielandt Theorem \citep{stewart1990matrix}, the eigenvalues of $L_{\text{SYM}}$ are:
\begin{align}
\lambda^l_i(L_{\text{SYM}}) &=\left(1-\frac{ (n_l+m_l)f_{\sigma}(\epsilonincluster)}{n_l+w^l+o^l} \pm (P_1+P_2)\right)\vee 0, \qquad i=1, \ 1\leq l \leq \numclust  \label{equ:SmallEigs},\\
\lambda^l_i(L_{\text{SYM}}) &=1 \pm (P_1+P_2), \qquad i=2, \ldots, n_l+m_l, \ 1\leq l \leq \numclust \label{equ:LargeEigs}\nonumber.
\end{align}
where:
\begin{align*}
P_1 &=\zeta_{N} \eta_2 + O(\zeta_{N}^2\eta_2^2 + \eta_\thetathres^2) \quad,\quad
P_2 = 2\eta_1+2\eta_\thetathres + O(\eta_1^2+\eta_\thetathres^2) ,\\
P_1+P_2 &= 2\eta_1 +2\eta_\thetathres +\zeta_{N} \eta_2 + O(\eta_1^2+\eta_\thetathres^2+\zeta_{N}^2\eta_2^2).
\end{align*}

\subsection{The Largest Spectral Gap of $\Lsym$}

For the remainder of the proof let $\{\lambda_i\}_{i=1}^N$ denote the eigenvalues of $\Lsym$ sorted in increasing order, and $\Delta_i = \lambda_{i+1}-\lambda_i$ for $1 \leq i \leq N-1$. Note the condition we will derive to guarantee $\Delta_{\numclust}$ is the largest eigengap also ensures that the $\numclust$ smallest eigenvalues are given by (\ref{equ:SmallEigs}).


Recalling the definition of $\zeta_{N}$ from Theorem \ref{thm:SpectralGap}, for $1\leq i\leq \numclust$, we have 
\begin{align*}
0 \leq \lambda_i &\leq  1-\frac{ (n_l+m_l)f_{\sigma}(\epsilonincluster)}{n_l+w^l+o^l} + (P_1+P_2) \\
&\leq 1-\frac{ (n_l+m_l)f_{\sigma}(\epsilonincluster)}{(n_l+m_l)+\sum_{s\ne l} (n_s+m_s)f_{\sigma}(\epsilonnoiseknn) } + (P_1+P_2) \\
&\leq 1 - \frac{(1-\eta_1)}{1+\zeta_{N}\eta_2}+(P_1+P_2) \\
&= 1 - (1-\eta_1)(1-\zeta_{N}\eta_2+O(\zeta_{N}^2\eta_2^2))+(P_1+P_2) \\
&= \eta_1+\zeta_{N}\eta_2+(P_1+P_2) +O(\eta_1^2+\zeta_{N}^2\eta_2^2)).
\end{align*}
Thus for $i < \numclust$, the gap is bounded by: $\Delta_i 
\leq \lambda_{i+1}  \leq \eta_1+\zeta_{N}\eta_2+(P_1+P_2) +O(\eta_1^2+\zeta_{N}^2\eta_2^2)).$
For $i > \numclust$, we have $\Delta_i 
\leq 1 + (P_1+P_2) - \left(1 - (P_1+P_2)\right)  \leq  2(P_1+P_2).$  Finally for $i=\numclust$:
\begin{align*}
\Delta_\numclust 
&\geq 1 - (P_1+P_2) - ( \eta_1+\zeta_{N}\eta_2+P_1+P_2) +O(\eta_1^2+\zeta_{N}^2\eta_2^2)) \\
&\geq 1 -\eta_1 -\zeta_{N}\eta_2-2(P_1+P_2)+O(\eta_1^2+\zeta_{N}^2\eta_2^2)).
\end{align*}
Thus $\Delta_{\numclust}$ is the largest gap if $\frac{1}{2} \geq \eta_1+\zeta_{N}\eta_2+2(P_1+P_2) + O(\eta_1^2+\zeta_{N}^2\eta_2^2) = 5\eta_1+4\eta_\thetathres +6\zeta_{N}\eta_2+ O(\eta_1^2+\eta_\thetathres^2+\zeta_{N}^2\eta_2^2)$,  which is the condition of Theorem \ref{thm:SpectralGap}. 
 

\subsection{Bounding the Spectral Embedding and Labeling Accuracy}

We apply Theorem 2 from \cite{Fan2018_Eigenvector} to bound the eigenvector perturbation. We let $\Phi=(\phi_1 \ldots \phi_\numclust)$ denote the $N$ by $\numclust$ matrix whose columns are the top $\numclust$ eigenvectors of $B$ (defined in \ref{subsec:pert_desc}), ordered so that $\phi_l$ corresponds to the block $B_l$. We let $\tilde{\Phi}$ be the equivalent quantity for $\Deg^{-\frac{1}{2}}WD^{-\frac{1}{2}}$. Defining the coherence of $\Phi$ as $\text{coh}(\Phi) = (N/\numclust) \max_i \sum_{j=1}^\numclust \Phi^2_{ij}$, we note that
\begin{align*}
\text{coh}(\Phi) &\leq \frac{N}{\numclust}\left(\|\phi_1\|^2_\infty + \cdots +\|\phi_\numclust\|^2_\infty\right) \leq \frac{N}{\numclust} \cdot \frac{\numclust\zeta_{N}}{N} = \zeta_{N},
\end{align*}
 i.e. $\Phi$ has low coherence since each eigenvector is constant on a cluster. Thus by Theorem 2 from \cite{Fan2018_Eigenvector}, there exists a rotation $R$ such that
\begin{align*}
\|\tilde{\Phi}R - \Phi\|_{\text{max}} &= O\left( \frac{\numclust^{\frac{5}{2}}\zeta_{N}^2 \|\Deg^{-\frac{1}{2}}WD^{-\frac{1}{2}}-B\|_{\infty}}{\lambda_\numclust(B) \sqrt{N}}\right)
\end{align*}
We recall from Section \ref{subsec:BoundingLSYMEigs} that
\[\lambda_\numclust(B) 
\geq \min_{1\leq l \leq \numclust}\frac{(n_l+m_l)(1-\eta_1)}{n_l+m_l+N\eta_2} = \frac{1-\eta_1}{1+\zeta_{N}\eta_2} = 1-\eta_1-\zeta_{N}\eta_2+O(\eta_1^2+\zeta_{N}^2\eta_2^2)\]
Letting $C_l$ denote the diagonal blocks of $C$ and using the bounds computed in Sections \ref{subsec:diagonalization} and \ref{subsec:P2}, we have:
\begin{align*}
\|\Deg^{-\frac{1}{2}}WD^{-\frac{1}{2}}-B\|_{\infty} &\leq \|\Deg^{-\frac{1}{2}}WD^{-\frac{1}{2}}-C\|_{\infty} + \max_l \|C_l-B_l\|_{\infty} \\
&\leq N\|\Deg^{-\frac{1}{2}}WD^{-\frac{1}{2}}-C\|_{\max} + \max_l\ (n_l+m_l)\|C_l-B_l\|_{\max} \\
&\leq 2\eta_1+2\eta_\thetathres + \zeta_{N}\eta_2+O(\eta_1^2+\eta_\thetathres^2+\zeta_{N}^2\eta_2^2)
\end{align*}
We conclude that
\begin{align*}
\|\tilde{\Phi}R - \Phi\|_{\text{max}} &\leq \frac{c\numclust^{\frac{5}{2}}\zeta_{N}^2}{\sqrt{N}}\left[\eta_1+\eta_\thetathres + \zeta_{N}\eta_2+O(\eta_1^2+\eta_\thetathres^2+\zeta_{N}^2\eta_2^2)\right] := P_3.
\end{align*}
for some absolute constant $c$. Letting $\{r_i\}_{i=1}^N$ denote the rows of $\Phi$ and $\{\tilde{r}_i\}_{i=1}^N$ denote the rows of $\tilde{\Phi}R$, we have $\|r_i - \tilde{r}_i\|_2 \leq \sqrt{\numclust}\|\tilde{\Phi}R - \Phi\|_{\text{max}}  \leq \sqrt{\numclust}P_3$ for all $i$. Letting $\pi(i) \in \{1,\ldots,\numclust\}$ denote the index of the set $\tilde{A}_l$ which contains the point corresponding to the $i^{th}$ row, we have $r_i = [0 \ldots (n_{\pi(i)}+m_{\pi(i)})^{-1/2} \ldots 0]$ for all $i$, where the non-zero element occurs in the $\pi(i)^{\text{th}}$ column. Thus the spectral embedding maps all points in $\tilde{A}_l$ inside a sphere in $\mathbb{R}^\numclust$ centered at $z_l = [0 \ldots (n_l+m_l)^{-1/2}\ldots 0]$ with radius $\sqrt{\numclust}P_3$. When $l\ne s$, we have $\|z_l - z_s\|_2 \geq \sqrt{\frac{2}{N}}$. Thus $\sqrt{\frac{2}{N}} > 10\sqrt{\numclust}P_3$
is sufficient to ensure that these spheres are well separated, i.e. the embedding is a perfect representation (see Definition \ref{def:perfectrep}) of the clusters $\tilde{A}_l$ with $r=2\sqrt{\numclust}P_3$.  Simplifying this condition, we thus obtain perfect label accuracy by clustering by distances on the spectral embedding whenever
\begin{align*}
\frac{1}{\numclust^3\zeta_{N}^2} \gtrsim \eta_1+\eta_\thetathres +\zeta_{N}\eta_2+ O(\eta_1^2+\eta_\thetathres^2+\zeta_{N}^2\eta_2^2).
\end{align*}

\bibliography{LLPD_JMLR_revised.bib}

\end{document}